\documentclass{article}

\usepackage{tikz}
\usepackage{bbm}
\usepackage{eqnarray}
\usepackage{color}
\usepackage{bm}
\usepackage{amssymb}
\usepackage{epsfig}
\usepackage{epsf}
\usepackage{float}
\usepackage{subfigure}
\usepackage{amsfonts,amsmath,amsthm,fancyhdr,multirow,hyperref}
\usepackage{booktabs,longtable,authblk}
\usepackage{mathrsfs,hhline,enumerate}

\usepackage[top=2.5cm, bottom=2.5cm, left=3cm, right=3cm]{geometry}
\newtheorem{theorem}{Theorem}[section]
\newtheorem{assumption}{Assumption}
\newtheorem{corollary}[theorem]{Corollary}
\newtheorem{definition}[theorem]{Definition}

\newtheorem{lemma}[theorem]{Lemma}
\newtheorem{proposition}[theorem]{Proposition}
\newtheorem{remark}{Remark}

\renewcommand{\H}{\mathcal{H}}

\newcommand{\T}{\mathcal{T}}
\newcommand{\B}{\mathcal{B}}

\newcommand{\E}{\mathcal{E}}

\newcommand{\X}{\mathcal{X}}
\newcommand{\Y}{\mathcal{Y}}

\newcommand{\md}{\mathrm{d}}
\newcommand{\Z}{\mathcal{Z}}
\newcommand{\W}{\mathcal{W}}
\renewcommand{\l}{\left}
\renewcommand{\r}{\right}

\newcommand{\be}{\mathbb{E}}
\newcommand{\bn}{\mathbb{N}}
\numberwithin{equation}{section}

\title{A Kernel-Based Stochastic Approximation Framework for Nonlinear Operator Learning$^\dag$\footnotetext{$\dag$~The work of Lei Shi is supported by the National Natural Science Foundation of China under Grant No.~12171093 and 12571099. Email addresses: jqyang24@m.fudan.edu.cn (J.-Q. Yang), leishi@fudan.edu.cn (L. Shi). The corresponding author is Lei Shi. }}
\author{Jia-Qi Yang}
\author{Lei Shi}
\affil{School of Mathematical Sciences and Shanghai Key Laboratory for
	Contemporary Applied Mathematics, Fudan University, Shanghai 200433, China.}

\date{}

\begin{document}
	\maketitle
	
	\begin{abstract}
We develop a stochastic approximation framework for learning nonlinear operators between infinite-dimensional spaces utilizing general Mercer operator-valued kernels. Our framework  encompasses two key classes: (i) operator-valued
kernels whose associated integral operators are compact and hence
admit discrete spectral decompositions, and (ii) separable kernels of the form $K(x,x')=k(x,x')T$,  where $k$ is a scalar-valued kernel and $T$ is a positive operator on the output space. This broad setting induces expressive vector-valued reproducing kernel Hilbert spaces (RKHSs) that generalize the classical $K=kI$ paradigm, thereby enabling rich structural modeling with rigorous theoretical guarantees. To address target operators lying outside the RKHS, we introduce vector-valued interpolation spaces to precisely quantify misspecification error. 
Within this framework, we establish non-asymptotic
convergence rates for prediction, estimation, and
misspecification errors in the online and
finite-horizon settings. 
Importantly, the framework also accommodates a range of operator learning settings, from Fredholm integral operators to encoder--decoder architectures. Numerical experiments on the two-dimensional Navier--Stokes equations illustrate the proposed approach.
	\end{abstract}
	

    \section{Introduction}

    Suppose that $\X$ is a Polish space \footnote{We do not assume local compactness of the input space $\X$ in this work. Local compactness can be used to show that the density of the RKHS $\H_K$ in $L^2(\X,\rho_{\X};\Y)$ for any probability measure $\rho_{\X}$ is equivalent to its density in $\mathcal{C}_0(\X,\Y)$, and that $L^2$ can be replaced by $L^p$ for any $1 \leq p < \infty$. These results are related to $\mathcal{C}_0$ operator-valued kernels; see \cite[Theorem 1]{carmeli2010vector}.}
, such as a Euclidean or a Sobolev space 
$W^{k,p}$ with $1 \leq p < \infty$ (or their open or closed subsets), and $\Y$ is a separable  Hilbert space with norm $\|\cdot\|_{\Y}$ and inner product $\langle\cdot,\cdot\rangle_{\Y}$.  Let $\rho$ be a probability distribution in $\X\times\Y$, and denote by $\rho_{\X}$ its marginal distribution on $\X$ with $\mathrm{supp}(\rho_{\X})=\X$. Assume that $\be_{(x,y)\sim\rho}\l[\|y\|_{\Y}^2\r]<\infty$. We write
    $L^2(\X,\rho_{\X};\Y)$ for the Lebesgue-Bochner space \cite[Chapter 1]{hytonen2016analysis} consisting of (equivalence classes of) strongly measurable operators $h:\X\to\Y$ such that the Bochner norm
    \[
    \l\|h\r\|_{\rho_{\X}}:=\l(\int_{\X}\l\|h(x)\r\|_{\Y}^2\mathrm{d}\rho_{\X}(x)\r)^{1/2}
    \] is finite. For any $h\in L^2(\X,\rho_{\X};\Y)$, the expected risk of $h$ is defined as 
    \[
    \mathcal{E}(h):=\be_{(x,y)\sim\rho}\l[\l\|h(x)-y\r\|_{\Y}^2\r]=\int_{\X\times\Y}\l\|h(x)-y\r\|_{\Y}^2\md\rho(x,y).
    \]
    The regression operator $h^\dagger$ is defined $\rho_{\X}$-almost everywhere by
    \begin{equation}\label{regression}
        h^\dagger(x):=\be_{y\sim\rho(y|x)}[y]=\int_{\Y}y\md\rho(y|x),  \qquad \rho_{\X}\text{-almost every }x\in\X,
    \end{equation}
    and uniquely minimizes $\E(h)$ over $L^2(\X,\rho_{\X};\Y)$, up to $\rho_\X$-null sets. 

    Given an i.i.d. sample $\mathbf{z}=\{z_t=\l(x_t,y_t\r)\}_{t=1}^T$ drawn from $\rho$ on $\X\times\Y$, our goal is to approximate the regression operator $h^\dagger$ based on $\mathbf{z}$ in order to minimize the prediction error. To this end, we consider an operator-valued kernel $K:\X\times\X\to \B(\Y)$, where $\B(\Y)$ denotes the space of  bounded linear operators on $\Y$. A mapping $K$ is called a kernel if it satisfies
    \begin{itemize}
        \item Hermitian symmetry: for all $x,x'\in\X$, we have $K(x,x')=\l(K(x',x)\r)^*$, where $(\cdot)^*$
        denotes the adjoint operator;
        \item Positive semi-definiteness: for any $n\in\mathbb{N}$, any $\{x_i\}_{i=1}^n\subset\X$, and any $\{y_i\}_{i=1}^n\subset\Y$,
        \[
        \sum_{i,j=1}^n\langle K(x_i,x_j)y_j,y_i\rangle_{\Y}\geq0.
        \]
    \end{itemize}
    Such a kernel $K$ induces a reproducing kernel Hilbert space (RKHS)  of $\Y$-valued operators on $\X$ \cite{burbea1984banach,micchelli2005learning,carmeli2010vector}, defined as the closure of the linear span
    \[
    \mathcal{H}_K:=\overline{\mathrm{span}}\left\{K(\cdot,x)y\mid x\in \X,y\in\mathcal{Y}\right\},
    \]
    equipped with an inner product $\langle\cdot,\cdot\rangle_K$ satisfying the reproducing property, i.e., 
    \[
    \langle K(\cdot,x)y,K(\cdot,x')y'\rangle_K=\langle K(x',x)y,y'\rangle_{\Y} \quad\text{and} \quad\langle h,K(\cdot,x)y\rangle_K=\langle h(x),y\rangle_{\Y},
    \]
    for any $h\in\H_K$, $x,x'\in\X$, and $y,y'\in\Y$.
    Throughout the paper, we assume the following condition on the kernel $K$:
    \begin{equation} \label{kernel}
        K\text{ is Mercer and }
        \sup\l\{\|K(x,x)\|:x\in\X\r\}\leq\kappa^2\   \footnotemark.
    \end{equation}
    Here an operator-valued kernel \(K:\X\times\X \to \B(\Y)\) is called a
Mercer kernel if \(\H_K\) is a subspace
of the space of continuous functions from \(\X\) to \(\Y\), denoted
\(\mathcal{C}(\X,\Y)\). 
    This general setting encompasses at least the following two important cases:
    \begin{itemize}
        \item Case 1: $K(x,x)$ is a compact linear operator on $\Y$ for all $x\in\X$.        
        \item Case 2: $K(x,x')=k(x,x')T$, where $k$ is a scalar-valued kernel and $T$ is a bounded  self-adjoint (possibly non-compact) positive operator. 
    \end{itemize}
\footnotetext{The uniform boundedness assumption on $K$, i.e., $\sup_{x\in\X} \l\|K(x,x)\r\| \leq\kappa^2$, is assumed rather than the weaker square-integrability condition $\int_{\X}\int_{\X}\l\|K(x,x')\r\|^2\mathrm{d}\rho_{\X}(x)\mathrm{d}\rho_{\X}(x')\leq\kappa^4$, since our analysis requires the pointwise estimate $\|h(x)\|_{\Y} \leq \kappa \|h\|_K$ for all $x\in\X$ and $h\in\H_K$.}

The first case includes operator-valued kernels generated by scalar-valued kernels through integral operator constructions; see Subsection \ref{Learning PDEs}. Under this setting, the associated integral operator is compact \cite[Proposition 3]{carmeli2010vector}. 
By contrast, in the second case, the corresponding integral operator can be non-compact. The special case $T=I$ has been studied in the context of regularized least squares and spectral algorithms \cite{mollenhauer2022learning,li2024towards,meunier2024optimal}, as well as in our recent work on regularized stochastic gradient descent \cite{yang2025learning}. These analyses rely on an isometric isomorphism between the RKHS and the space of Hilbert–Schmidt operators, with source conditions imposed on the latter. As a concrete example of the second case, when  $k$ is a Matérn kernel and $T=I$, the associated RKHS coincides with a vector-valued Sobolev space with equivalent norms; see Remark \ref{remark 5} for details. Other examples of operator-valued kernels include the neural tangent kernel  defined for two-layer neural operators \cite{nguyen2024optimal}, as well as constructions studied in \cite{carmeli2010vector, kadri2016operator}.
In contrast to the above approaches, our work imposes conditions directly on the integral operator, leading to a unified framework for theoretical analysis.

For an estimator $h$, we define the estimation error as $\l\|h-h^\dagger\r\|_K^2$, which quantifies approximation in the RKHS $\mathcal{H}_K$, in turn, controls convergence in the space of continuous functions $\mathcal{C}(\X,\Y)$ or, more generally, in Sobolev spaces (see Remark \ref{remark 5}). Together with the prediction error $\mathcal{E}(h)-\mathcal{E}(h^\dagger)$, this quantity provides a key metric for evaluating the performance of the stochastic approximation scheme. When the target operator $h^\dagger$ does not necessarily reside in  $\mathcal{H}_K$, the estimation error may no longer provide a meaningful measure of approximation quality. This situation, commonly known as model misspecification \cite{rao1971some,bach2008consistency}, has been recently investigated in several works on kernel methods \cite{fischer2020sobolev,li2024towards}, which establish convergence rates under various conditions. Even if $\mathcal{H}_K$ is dense in $L^2(\X,\rho_\X;\Y)$, the assumption that $h^\dagger$ lies precisely in $\mathcal{H}_K$ is often too restrictive in practice. For example, if $K=kI$ with a Matérn kernel $k$, then $\mathcal{H}_K$ is a vector-valued Sobolev space, and requiring $h^\dagger\in \mathcal{H}_K$ would imply that its derivatives up to a certain order are square-integrable. In \cite{li2024towards}, the notion of interpolation spaces for such operator-valued diagonal kernels is introduced; it is shown there that these interpolation spaces correspond to a lower-order vector-valued fractional Sobolev space. These interpolation spaces and their associated norms have been referred to as Sobolev spaces and Sobolev norms, respectively, in several works on scalar-valued kernels \cite{steinwart2012mercer,fischer2020sobolev,lu2022sobolev}.  To extend these ideas to a broader class of operator-valued kernels, we combine the $K$-functional from the real interpolation method with the spectral theorem to define an appropriate interpolation space (see Definition \ref{interpolation space} and Theorem \ref{thm:interpolation}). Within this space, the discrepancy between $h^\dagger$ and its approximation is referred to as the misspecification error, which can be rigorously quantified even when $h^\dagger\notin \mathcal{H}_K$. Our framework generalizes existing results for diagonal kernels of the form  $K=kI$ to general operator-valued kernels and provides convergence guarantees for stochastic approximation schemes in this more general setting.
   
In this paper, we consider estimating the target operator $h^\dagger$ by a stochastic gradient descent approach. 
For any $h\in\mathcal H_K$, the Fréchet derivative \cite{dunford1988linear} of $\mathcal E$ on $\mathcal H_K$ at $h$ is represented by
\[
2\mathbb E_{(x,y)\sim\rho}[K(\cdot,x)(h(x)-y)].
\]  
Replacing the population expectation with its instantaneous empirical counterpart based on a single observation $z_t$ yields the following stochastic approximation iteration:
    \begin{equation} \label{algorithm}
        \begin{cases}
            h_1 := \mathbf{0}, \\
            h_{t+1} := h_t-\eta_t K(\cdot,x_t)(h_t(x_t)-y_t),
        \end{cases}
    \end{equation}
where $\eta_t>0$ is the step size at the $t$-th iteration. Here, $\mathbf{0}$ denotes the zero element in $\H_K$, and the same notation will be used for the zero element in other Hilbert spaces throughout the paper. We study two types of step size selection strategies. The first is the online setting, where the data arrives sequentially and the total number of samples (or iterations) $T$ is unknown and possibly infinite, as is typical in applications with streaming data  \cite{smale2006online,guo2023capacity,boudart2024structured}. In this case, a polynomially decaying step size is employed, given by
$\eta_t=\eta_1 t^{-\theta}$, where $\eta_1$ is a constant independent of $t$ and  $0<\theta<1$. The second is the finite-horizon setting, where the sample size $T<\infty$ is fixed and known in advance. In this case, although the algorithm still processes one sample at a time, the knowledge of $T$  allows for a step size of the form $\eta_t=\eta T^{-\theta'}$, where $\eta$ is a constant independent of $T$ and $0<\theta'<1$. This setting reflects scenarios where a fixed-size dataset is available and the algorithm makes a single pass over it. These two step size selection strategies serve as an implicit regularization, enhancing the robustness and generalization ability of the algorithm \cite{smale2006online}.

This framework aligns naturally with the broader paradigm of operator learning, which seeks to approximate mappings between infinite-dimensional function spaces using data. A significant motivation comes from solving partial differential equations (PDEs), where the objective is to efficiently learn mappings from boundary or initial conditions to solutions, a task ubiquitous in scientific and engineering applications \cite{kovachki2024operator,subedi2025operator}. In recent years, neural operator architectures, such as DeepONet \cite{lu2021learning}, FNO \cite{li2021fourier}, and PCA-Net \cite{bhattacharya2021model}, have demonstrated strong empirical performance across various scientific domains. These parametric models employ finite-dimensional neural networks to represent nonlinear operators. While architectures such as FNO offer advantages like discretization invariance, their expressivity is limited by a fixed network size and does not scale adaptively with increasing data volume. Kernel-based methods offer a nonparametric alternative whose capacity increases with the data and whose theoretical guarantees, particularly for prediction and estimation error, are well established in the scalar-output setting $\Y=\mathbb{R}$. Extensions to infinite-dimensional outputs via kernels of the form $K = kI$ have been recently analyzed in  \cite{mollenhauer2022learning,li2024towards,meunier2024optimal,shi2024learning,yang2025learning}.  
By contrast, general operator-valued kernels $K$, which can encode nonseparable output couplings beyond diagonal forms, have received comparatively less systematic treatment in operator learning.
Moreover, optimization-theoretic analysis in the neural operator literature remains limited; a notable exception is \cite{nguyen2024optimal}, which introduces a neural tangent kernel framework. By contrast, kernel-based operator learning admits rigorous optimization-theoretic analysis with provable convergence. Despite their theoretical strengths, kernel-based operator learning methods have only recently gained some attention. Notable examples include kernel ridge regression for learning Green’s functions \cite{stepaniants2023learning}, nonlinear PDE operators \cite{batlle2024kernel}, a three-step operator-learning scheme \cite{long2024kernel}, and the kernel equation learning framework for solving and discovering PDEs \cite{jalalian2025data}. Numerical results in \cite{batlle2024kernel,long2024kernel,jalalian2025data} further demonstrate that kernel-based approaches can achieve performance competitive with neural operator methods. Beyond PDE-related applications, kernel-based operator learning also appears in functional regression \cite{kadri2010nonlinear,crambes2013asymptotics,kadri2016operator,brogat2022vector}, structured output prediction \cite{ciliberto2016consistent,ciliberto2020general,brogat2022vector,boudart2024structured}, instrumental-variable kernel regression \cite{singh2019kernel}, regression with proximal variables \cite{mastouri2021proximal}, conditional mean embeddings \cite{64c20018a13e479998e9d77dedd3c87c,park2020measure}, and data-driven modeling of dynamical systems \cite{song2009hilbert,kostic2022learning}, among many others.

In this work, we develop a stochastic approximation framework for nonlinear operator learning with general operator-valued kernels. 
The framework updates the estimator using one sample
at each iteration and is naturally suited to
infinite-dimensional input and output spaces,
making it particularly relevant for learning PDE operators. 
It offers a flexible nonparametric alternative to existing parametric approaches in operator learning.  We establish a non-asymptotic convergence analysis of both prediction and estimation errors under two step size strategies. In addition, by exploiting vector-valued interpolation spaces, we derive misspecification error rates which, to the best of our knowledge, have not previously been established for general operator-valued kernels. These results provide theoretical guarantees for the training behavior of the proposed operator learning algorithm, addressing a key gap in the literature where optimization-theoretic analyses remain limited. Under mild assumptions, we also obtain sharper convergence rates. The proposed method applies to a wide range of problems, including vector-valued functional regression, learning PDE operators, and inverse problems for nonlinear PDEs. 
Under the assumptions specified in
Subsection~\ref{Learning PDEs}, it applies to learning
square-integrable Green's functions and, more generally,
Hilbert--Schmidt Fredholm integral operators, 
as well as to operator learning between infinite-dimensional spaces from linear measurement data (see Subsection~\ref{Example: Learning via Encoder–Decoder Frameworks}). Finally, numerical experiments on the two-dimensional Navier--Stokes equations illustrate the proposed approach.

The main contributions of our work are summarized below:
\begin{itemize}
    \item We extend the interpolation spaces studied in
\cite[Section 3]{li2024towards} for $K=kI$
to general operator-valued Mercer kernels
using the spectral calculus of $L_K$.
These spaces allow us to analyze the misspecification
error, including cases where $h^\dagger\notin\H_K$.
    \item Under general assumptions, we establish prediction, estimation, and misspecification rates for learning with general operator-valued kernels, including cases with non-compact integral operators $L_K$. With slightly stronger assumptions, we obtain sharper rates. To the best of our knowledge, these results are new.

    \item Our convergence analysis covers finite- and
infinite-dimensional input and output spaces.
The bounds quantify how the regularity of the target
operator and the spectral properties of the kernel
integral operator affect convergence. 
\item We apply the method to Fredholm integral operators
and encoder--decoder architectures.
For Green's function learning, we work with Green kernels
in $L^2$ and construct an operator-valued Mercer kernel
on the bounded input set.
This complements the formulation using a scalar RKHS in
\cite{stepaniants2023learning}.

\end{itemize}

The remainder of this paper is organized as follows. In Section \ref{section 2}, we present the assumptions and main theoretical results. Section \ref{section 3} illustrates the application of the proposed algorithm and provides supporting numerical experiments. For clarity, all technical proofs are deferred to Section \ref{section 4} and the Appendix.

    \section{Main Theoretical Results} \label{section 2}
    
    In this section, we introduce the notation and mathematical preliminaries needed for the subsequent analysis. We then state the assumptions and present the main theoretical results.

    \subsection{Notation and Mathematical Preliminaries}
    Denote the space of all bounded linear operators in $\Y$ by $\B(\Y)$. Denote the set $\{1,2,\cdots,T\}$ by $\bn_T$. Given Hilbert spaces $\mathcal{H}_1$ and $\mathcal{H}_2$, and elements $f \in \mathcal{H}_1$, $g, h \in \mathcal{H}_2$, we define the rank-one operator $f \otimes g: \mathcal{H}_2 \to \mathcal{H}_1$ by $(f \otimes g)(h) := \langle g, h \rangle_{\mathcal{H}_2} f$. For any bounded linear operator $A:\H_1\to\H_2$, we denote its adjoint by $A^*$, defined by $\langle Af,g\rangle_{\H_2}=\langle f,A^*g\rangle_{\H_1}$. 
    Write $z^t=(z_1,\ldots,z_t)$ and let $\be_{z^t}$ denote
expectation over $z^t$. An expectation subscript specifies
the variables integrated out, with all other sample variables
held fixed. We use $\be_{z^0}$ to denote the identity.
Each iterate $h_t$ is a measurable function of $z^{t-1}$
and is independent of $z_t$.

    We write $\|\cdot\|_{K}$ and $\langle \cdot, \cdot \rangle_K$ for the norm and inner product of the RKHS $\H_K$ induced by $K$, respectively. Define the integral operator $L_K$ on $L^2(\X, \rho_{\X}; \Y)$ associated with $K$ by
    \begin{equation*}
        (L_Kh)(x)=\int_{\X}K(x,t)h(t)\md\rho_{\X}(t).
    \end{equation*}
    Then $L_K$ is well-defined, self-adjoint, and positive
    with $\|L_K\| \leq \kappa^2$. Recall that $\operatorname{supp}(\rho_{\X}) = \X$. 
    Define the canonical map
\[
\iota_K:\H_K\to L^2(\X,\rho_{\X};\Y),
\qquad
h\mapsto [h],
\]
where $[h]$ denotes the equivalence class of $h$ under
$\rho_{\X}$ almost everywhere equality. By the reproducing property
and the uniform boundedness condition~\eqref{kernel},
\[
\|\iota_K h\|_{\rho_{\X}}^2
=
\int_{\X}\|h(x)\|_{\Y}^2\,\mathrm{d}\rho_{\X}(x)
\leq
\kappa^2\|h\|_K^2,
\qquad h\in\H_K.
\]
Hence $\iota_K$ is bounded with $\|\iota_K\|\leq\kappa$. Moreover, since $K$ is Mercer, every $h\in\H_K$ is continuous.
Together with $\operatorname{supp}(\rho_{\X})=\X$, this implies that
$\iota_K$ is injective. 

We henceforth identify $\H_K$ with its image $\iota_K(\H_K)$
whenever no confusion can arise, using $\|\cdot\|_K$
for the corresponding RKHS norm.
The adjoint of $\iota_K$ is given by
\[
\iota_K^*f
=\int_{\X}K(\cdot,x)f(x)\,\md\rho_{\X}(x),
\qquad f\in L^2(\X,\rho_{\X};\Y).
\]
We have $L_K=\iota_K\iota_K^*$ and define the covariance operator $C_K:=\iota_K^*\iota_K:\H_K\to\H_K.$ This operator is bounded, self-adjoint, and nonnegative, with $\|C_K\|=\|L_K\|\leq\kappa^2.$ Moreover, since $\iota_K C_K=L_K\iota_K$, continuous functional calculus yields $\iota_K C_K^\alpha=L_K^\alpha\iota_K$ for any $\alpha\geq0$.

If $\mathscr{F}$ is the Borel $\sigma$-algebra and $\mathscr{F}\ni E\mapsto\mathcal{P}(E)\in\B\l(L^2(\X, \rho_{\X}; \Y)\r)$ is a projection-valued measure, for $f_1,f_2\in L^2(\X, \rho_{\X}; \Y)$, we write $\left\langle\mathrm{d}\mathcal{P}(\lambda)f_1,f_2\right\rangle_{\rho_\X}$ as the bounded  measure defined by $E\mapsto\left\langle \mathcal{P}(E)f_1,f_2\right\rangle_{\rho_\X}$. Then, $L_K$ admits the spectral decomposition:
    \begin{equation}\label{spectral1}
        L_K=\int_{\sigma_K}\lambda\mathrm{d}\mathcal{P}(\lambda),
    \end{equation}
    where $\sigma_K$ is the spectrum of $L_K$, a compact subset in $[0,\infty)$, and $E\mapsto \mathcal{P}(E)$ is the corresponding spectral measure. 
    By \cite[Proposition 6.1]{carmeli2006vector} and the subsequent
discussion,
\[
\iota_K(\H_K)
=
\left\{
f\in\ker L_K^\perp:
\int_{\sigma_K\setminus\{0\}}
\lambda^{-1}
\left\langle
\mathrm{d}\mathcal{P}(\lambda)f,f
\right\rangle_{\rho_{\X}}
<\infty
\right\}.
\]
Moreover, for any $h,g\in\H_K$,
\[
\langle h,g\rangle_K
=
\int_{\sigma_K\setminus\{0\}}
\lambda^{-1}
\left\langle
\mathrm{d}\mathcal{P}(\lambda)\iota_Kh,\iota_Kg
\right\rangle_{\rho_{\X}}.
\]
Equivalently, $L_K^{1/2}:
\ker L_K^\perp
\longrightarrow
\iota_K(\H_K)$ is an isometric isomorphism.

    Next, for any $h \in \H_K$, we define the evaluation operator at $x \in \X$ by
    \[
    ev_x(h)=h(x),
    \]
    whose adjoint satisfies $\operatorname{ev}_x^*(y) = K(\cdot, x) y$ for any $y \in \Y$. Since $\|\operatorname{ev}_x\| = \|\operatorname{ev}_x^*\| \leq \kappa$, it follows that $\|h(x)\|_{\Y} \leq \kappa \|h\|_K$ for all $h \in \H_K$ and $x \in \X$. Furthermore, the evaluation operators satisfy
    \[
    ev_{x}ev_{x'}^*=K(x,x') \quad\text{and}\quad ev_{x}^*ev_{x'}(h)=K(\cdot,x)h(x').
    \]
    for any $x,x'\in \X$ and $h\in\H_K$. 
    Taking expectation over $x\sim\rho_{\X}$ gives $\be_{x\sim\rho_{\X}}\l[ev_x^*ev_x\r]=C_K,$ where the expectation is understood after applying the operators
to each fixed $h\in\H_K$. Since $\iota_K$ is injective, its polar decomposition
$\iota_K=UC_K^{1/2}$ yields a unitary map
$U:\H_K\to\ker L_K^\perp$ satisfying $L_K\big|_{\ker L_K^\perp}=UC_KU^*.$ Thus, for every continuous function
$q:[0,\kappa^2]\to\mathbb R$ and every $s>0$,
\[
\|q(C_K)\|\leq\|q(L_K)\|,
\qquad
\operatorname{Tr}(C_K^s)=\operatorname{Tr}(L_K^s),
\]
where the traces may take the value $+\infty$.
For further background on operator-valued RKHSs, see
\cite{micchelli2005learning,carmeli2006vector,carmeli2010vector}.

    \subsection{Vector-valued Interpolation Space} \label{Section:Interpolation space}

In this subsection, we introduce interpolation spaces for vector-valued functions, motivated by a key issue in the analysis of stochastic approximation schemes. Because the updates are driven by the gradient of the prediction error $\mathcal{E}(h)$ computed in the RKHS $\mathcal{H}_K$, the resulting solution remains confined to $\mathcal{H}_K$. However, the target operator $h^\dagger$ may lie outside $\H_K$, representing a misspecified case in which the hypothesis space excludes the true target operator. To address this issue, we introduce, alongside the prediction error (which measures predictive performance), a misspecification error that quantifies the distance to $h^\dagger$ in an enlarged space. This motivates defining the interpolation space $[\H_K]^\beta$ with $\beta\geq 0$ (Definition \ref{interpolation space}), which provides a natural ambient space for measuring misspecification errors. 
For $0<\beta<1$, Theorem \ref{thm:interpolation}
shows that $[\H_K]^\beta$ coincides, up to equivalence
of norms, with the real interpolation space
$\l[L^2(\X,\rho_{\X};\Y),[\H_K]^1\r]_{\beta,2}$
defined via the $K$-functional. All proofs for this subsection are deferred to Appendix \ref{Appendix1}.

For any $h \in \H_K$, denote $\iota_K h$ by $[h]$. We now extend the notion of interpolation spaces to general operator-valued kernels.

\begin{definition}[Vector-valued interpolation space]
\label{interpolation space}
For any $\beta\geq0$, the \emph{vector-valued interpolation space}
$[\H_K]^\beta$ is defined by
\[
[\H_K]^\beta
:=
\left\{
L_K^{\beta/2}f:
f\in \ker L_K^\perp
\right\}
\subset L^2(\X,\rho_{\X};\Y),
\]
endowed with the norm
\[
\left\|L_K^{\beta/2}f\right\|_{[\H_K]^\beta}
:=
\|f\|_{\rho_{\X}},
\qquad
f\in\ker L_K^\perp.
\]
\end{definition}

This norm is well defined since
$L_K^{\beta/2}$ is injective on $\ker L_K^\perp$. For $\beta>0$, we denote by
\[
L_K^{-\beta/2}:
[\H_K]^\beta\to\ker L_K^\perp
\]
the inverse of the restriction of $L_K^{\beta/2}$ to
$\ker L_K^\perp$. Thus,
\[
\|h\|_{[\H_K]^\beta}
=
\|L_K^{-\beta/2}h\|_{\rho_{\X}},
\qquad
h\in[\H_K]^\beta.
\]
By construction, $[\H_K]^0=\ker L_K^\perp=\overline{\operatorname{ran}L_K}$, endowed with the $L^2$ norm, and $[\H_K]^1=\iota_K(\H_K)$, endowed with the RKHS norm. Moreover, the operator $L_K^{\beta/2}$ induces an isometric isomorphism from $\ker L_K^\perp$ onto $\l[\H_K\r]^{\beta}$. For any $0\leq\beta_1<\beta_2$, the embedding
\[
[\H_K]^{\beta_2}\hookrightarrow[\H_K]^{\beta_1}
\]
is continuous. More precisely,
\[
\|h\|_{[\H_K]^{\beta_1}}
\leq
\|L_K\|^{(\beta_2-\beta_1)/2}
\|h\|_{[\H_K]^{\beta_2}}.
\]
The embedding is compact whenever $L_K$ is compact. More
specifically, it is Hilbert--Schmidt if $\operatorname{Tr}
\left(
L_K^{\beta_2-\beta_1}
\right)
<\infty.$

We next introduce the real interpolation spaces via
the $K$-functional and establish the corresponding
characterization of $[\H_K]^\beta$ for $0<\beta<1$. 
    \begin{definition}[$K$-functional \cite{triebel1995interpolation}] \label{interpolation space 3}

Let $\mathcal{G}_1$ and $\mathcal{G}_2$ be two Banach spaces that are continuously embedded in a common topological vector space $\mathcal{G}$. Then, for any $f\in\mathcal{G}_1+\mathcal{G}_2$ and $t>0$, the $K$-functional is defined by
\[
K\l(f,t,\mathcal{G}_1,\mathcal{G}_2\r):=\inf_{f=f_1+f_2}\Bigl\{\l\|f_1\r\|_{\mathcal{G}_1}+t\l\|f_2\r\|_{\mathcal{G}_2}: f_1\in\mathcal{G}_1,\;f_2\in\mathcal{G}_2\Bigr\}.
\]
For $0<\beta<1$, the corresponding interpolation norm is defined by 
\[
\|f\|_{\beta,2}:=\l(\int_0^\infty\!\l(t^{-\beta}K\l(f,t,\mathcal{G}_1,\mathcal{G}_2\r)\r)^2 t^{-1}\,\mathrm{d}t\r)^{1/2}.
\]
The associated interpolation space is then given by
\[
\l[\mathcal{G}_1,\mathcal{G}_2\r]_{\beta,2}:=\Bigl\{f\in \mathcal{G}_1+\mathcal{G}_2:\|f\|_{\beta,2}<\infty\Bigr\}.
\]
    \end{definition}

In our context, 
we are particularly interested in the case $\mathcal G_1=L^2(\X,\rho_{\X};\Y)$ and $\mathcal G_2=[\H_K]^1.$ 
We now show that the interpolation spaces defined in Definition \ref{interpolation space} and Definition \ref{interpolation space 3} coincide and that the corresponding norms are equivalent.

\begin{theorem}
\label{thm:interpolation}
For any $0<\beta<1$, we have
\[
[\H_K]^\beta
=
\operatorname{ran}L_K^{\beta/2}
=
\left[
L^2(\X,\rho_{\X};\Y),[\H_K]^1
\right]_{\beta,2}
=
\left[
[\H_K]^0,[\H_K]^1
\right]_{\beta,2}.
\]
The corresponding norms are equivalent.

More precisely, let
\[
d_\beta
:=
\int_0^\infty
\frac{s^{1-2\beta}}{1+s^2}\,\mathrm{d}s
\in(0,\infty).
\]
Then, for every $f\in\ker L_K^\perp$,
\[
d_\beta^{1/2}\|f\|_{\rho_{\X}}
\leq
\left\|L_K^{\beta/2}f\right\|_{\beta,2}
\leq
(2d_\beta)^{1/2}\|f\|_{\rho_{\X}},
\]
where $\|\cdot\|_{\beta,2}$ denotes the real interpolation norm
associated with the couple $\left(
L^2(\X,\rho_{\X};\Y),[\H_K]^1
\right).$
\end{theorem}

In Appendix \ref{Appendix1}, we present the proof of this result. The proof relies on the spectral theorem for bounded self-adjoint operators on Hilbert spaces, which permits representing $L_K$ as a multiplication operator on an $L^2$ space over a $\sigma$-finite measure space via a unitary transformation. This representation then allows us to employ standard techniques from interpolation theory to complete the proof.
 
    \begin{remark}
The interpolation spaces in
\cite[Section 3]{li2024towards} are defined for
kernels of the form $K(x,x')=k(x,x')I$.
Here we define $[\H_K]^\beta$ through the spectral
calculus of $L_K$ on $L^2(\X,\rho_{\X};\Y)$
for general operator-valued Mercer kernels
satisfying \eqref{kernel}.
This construction also applies when $L_K$ is non-compact.
\end{remark}

\subsection{Prediction, Estimation, and Misspecification Errors} \label{subsection: rates}

    In this subsection, we present the theoretical guarantees for the proposed algorithm under a sequence of increasingly stronger, yet natural, assumptions. We first establish upper bounds for the prediction and estimation errors under Assumptions \ref{a0} and \ref{a3}. Importantly, this first result holds for general operator-valued kernels satisfying \eqref{kernel}, where the associated integral operator $L_K$ may be non-compact. To the best of our knowledge, such a general framework has not been analyzed previously; in particular, it covers the settings of \cite{li2024towards,yang2025learning}. We then provide convergence rates for the misspecification error, which characterize the operator approximation capability when the target operator does not lie in the RKHS. Finally, we introduce a slightly stronger assumption (Assumption \ref{a1}), along with an additional trace condition (Assumption \ref{a2}), under which we derive sharper convergence rates. 
    This setting includes, in particular, the construction of Green kernels in Subsection~\ref{Learning PDEs}. By the
factorization $L_K=\mathsf{S}^*L_k\mathsf{S}$, the Schatten
condition $\operatorname{Tr}(L_k^s)<\infty$ for any
$0<s\leq1$ implies Assumption~\ref{a2} for $L_K$; see
Proposition~\ref{Green}.

    While stronger conditions—such as  moment assumptions (e.g., \cite{guo2023capacity, shi2024learning, yang2025learning})—can yield faster convergence rates, they depart from our objective of maintaining wide applicability. Moreover, our method naturally extends to the covariate shift setting (e.g., \cite{sugiyama2012machine, ma2023optimally}). With an additional boundedness assumption on the output, recent techniques  \cite{yang2025learning} can be employed to derive high-probability bounds that guarantee almost sure convergence. However, such refinements fall beyond the scope of this paper.

    \begin{assumption} \label{a0}
        The variance of the noise satisfies $\be_{(x,y)\sim\rho}\l[\l\|y-h^\dagger(x)\r\|_{\Y}^2\r]\leq\sigma^2$.
    \end{assumption}
This is a mild assumption, requiring only that the noise variable $y-h^\dagger(x)$ is square-integrable.
    
    \begin{assumption} \label{a3}
        There exists $r>0$ such that $h^\dagger=L_K^rg^\dagger$, where $g^\dagger\in L^2(\X, \rho_{\X};\Y)$.
    \end{assumption}
    This is a classical assumption used to characterize the smoothness of the target operator $h^\dagger$. Specifically, it means that $h^\dagger$ belongs to the image of the operator power $L_K^r$
  acting on the space $L^2(\X, \rho_{\X};\Y)$
  \[
\operatorname{ran}L_K^r
=
\left\{
h\in\ker L_K^\perp:
\int_{\sigma_K\setminus\{0\}}
\lambda^{-2r}
\langle \mathrm{d}\mathcal{P}(\lambda)h,h\rangle_{\rho_{\X}}
<\infty
\right\}.
\]

    Clearly, larger values of $r$ correspond to stronger smoothness assumptions on $h^\dagger$. This, in turn, typically leads to  an improved convergence of the learning algorithm. In particular, when $r\geq\frac{1}{2}$, it follows that $h^\dagger\in\iota_K(\H_K)$. Moreover, without loss of generality, we choose
$g^\dagger\in\ker L_K^\perp$.

\begin{remark}\label{remark 3}
Suppose $K(x,x')=k(x,x')I$.
For $\beta>0$, let $[\mathcal H]_X^\beta$ denote the
scalar-valued interpolation space induced by $k$,
defined as in Definition \ref{interpolation space},
and let
$\iota_\beta:[\mathcal H]_X^\beta
\to L^2(\X,\rho_{\X};\mathbb R)$
be its canonical embedding.
Denote by
\[
\Psi:
S_2\bigl(L^2(\X,\rho_{\X};\mathbb R),\Y\bigr)
\longrightarrow L^2(\X,\rho_{\X};\Y)
\]
the isometric isomorphism determined by
$\Psi(y\otimes f)=f(\cdot)y$.

In this notation, the regularity condition considered in
\cite{li2024towards,meunier2024optimal} can be expressed as
\[
h^\dagger=\Psi(C^*\iota_\beta^*),
\qquad
C^*\in S_2([\mathcal H]_X^\beta,\Y),
\qquad
\|C^*\|_{\mathrm{HS}}\leq B.
\]
This condition is equivalent to
\[
h^\dagger\in[\H_K]^\beta
=\operatorname{ran}L_K^{\beta/2},
\qquad
\|h^\dagger\|_{[\H_K]^\beta}\leq B,
\]
with the exact norm identity $\|C^*\|_{\mathrm{HS}}
=
\|h^\dagger\|_{[\H_K]^\beta}
=
\|L_K^{-\beta/2}h^\dagger\|_{\rho_{\X}}.$ Thus this condition is equivalent to Assumption \ref{a3}
with $r=\beta/2$, together with
$\|g^\dagger\|_{\rho_{\X}}\leq B$.
\end{remark}

\begin{remark}\label{remark 4}
When $K(x,x')=k(x,x')I$, the map
$H\mapsto[x\mapsto H\phi(x)]$, where
$\phi(x):=k(\cdot,x)$, is an isometric isomorphism
from $S_2(\H_k,\Y)$ onto $\H_K$.
See \cite[Proposition 2.1]{yang2025learning}.
Consequently, if $h^\dagger\in\H_K$,
there exists a unique $H^\dagger\in S_2(\H_k,\Y)$
such that $h^\dagger(x)=H^\dagger\phi(x)$.
Under Assumption \ref{a3}, we have
$h^\dagger\in\H_K$ whenever $r\geq1/2$.

Let
$C:=\be_{x\sim\rho_{\X}}[\phi(x)\otimes\phi(x)]$
be the covariance operator on $\H_k$.
For $q>0$, the condition considered in
\cite{shi2024learning,yang2025learning} is
\[
H^\dagger=S^\dagger C^q,
\qquad S^\dagger\in S_2(\H_k,\Y).
\]
This condition is equivalent to
$h^\dagger\in\operatorname{ran}L_K^{q+1/2}$,
and hence corresponds to Assumption \ref{a3}
with $r=q+1/2$.
\end{remark}
The proofs of Remarks \ref{remark 3} and \ref{remark 4}
are given in Appendix \ref{Appendix 2}.

To state the results on convergence rates, we define
\begin{equation} \label{gamma}
\begin{aligned}
\gamma_1
&:=\frac{\theta}
{4\kappa^2(1+2\kappa^2)(\delta+1)},\\
\gamma_1'
&:=\frac{\theta'}
{4\kappa^2(1+2\kappa^2)(1+2\theta')}.
\end{aligned}
\end{equation}
where $\delta$ is the constant defined in
Proposition \ref{cite1} corresponding to $v=1$ and the chosen $\theta$. Recall that \(z^T=(z_1,\ldots,z_T)\), where \(z_t=(x_t,y_t)\) and
\(z_1,\ldots,z_T\) are i.i.d. samples drawn from \(\rho\), and that
\(\mathbb E_{z^T}\) denotes expectation with respect to the product measure \(\rho^{\otimes T}\).

We first state convergence rates for the prediction and estimation errors under Assumptions~\ref{a0} and~\ref{a3}.

\begin{theorem} \label{Thm1}
Let $T\geq1$. Suppose Assumption \ref{a0} holds with $\sigma^2>0$ and Assumption \ref{a3} holds with $r>0$ and $g^\dagger\in L^2(\X, \rho_{\X};\Y)$. Then the following results hold:
\begin{itemize}
    \item[(1)] If we choose the step sizes $\left\{\eta_t=\eta_1t^{-\theta}\right\}_{t\geq1}$ with $0<\eta_1<\min\l\{\|L_K\|^{-1},1-\theta,\gamma_1\r\}$ and  $0<\theta<1$, then when $r>0$, the prediction error satisfies
    \begin{equation*}
\be_{z^T}\l[\E(h_{T+1})-\E(h^\dagger)\r]
\leq c_1\eta_1^{-2r}
\begin{cases}
(T+1)^{-\theta}f_{0}(T),
& \text{if }0<\theta\leq
\frac{\min\{2r,1\}}{1+\min\{2r,1\}},\\
(T+1)^{-\min\{2r,1\}(1-\theta)},
& \text{if }
\frac{\min\{2r,1\}}{1+\min\{2r,1\}}
<\theta<1,
\end{cases}
\end{equation*}
where
\[
f_{0}(T):=
\begin{cases}
\log(T+1),&\text{if }\theta=\frac12 \text{ and } r\geq\frac{1}{2},\\
1,&\text{otherwise}.
\end{cases}
\]
    \item[(2)] If we choose the step sizes $\{\eta_t=\eta_1\}_{t\in\bn_T}$ with $\eta_1 = \eta T^{-\theta'}$, $0<\eta<\min\l\{\|L_K\|^{-1},1,\gamma_1'\r\}$, and $0<\theta'<1$, then when $r>0$, the prediction error satisfies
    \begin{equation*}
    \begin{aligned}
        \be_{z^T}\l[\E(h_{T+1})-\E(h^\dagger)\r]
&\leq
c_1'\l[
(1+\eta T^{1-\theta'})^{-2r}
+\eta T^{-\theta'}
\r]
\\&\lesssim \eta^{-2r}
\begin{cases}
    (T+1)^{-\theta'},  & \text{ if }0<\theta'\leq\frac{2r}{1+2r}, \\
    (T+1)^{-2r(1-\theta')},  & \text{ if }\frac{2r}{1+2r}<\theta'<1,
\end{cases}
    \end{aligned}
\end{equation*}
    and when $r>\frac{1}{2}$ and $\frac{1}{2}<\theta'<1$, the estimation error satisfies
    \begin{equation*}
            \be_{z^T}\l\|h_{T+1}-h^\dagger\r\|_K^2\leq c_1'\eta^{-(2r-1)}
            \begin{cases}
                (T+1)^{1-2\theta'}, & \text{ if }\frac{1}{2}<\theta'\leq\frac{2r}{2r+1}, \\
                (T+1)^{-(2r-1)(1-\theta')} & \text{ if }\frac{2r}{2r+1}<\theta'<1.
            \end{cases}
    \end{equation*}
\end{itemize}
Here the constants $c_1$ and $c_1'$ are independent of $T$, $\eta_1$, and $\eta$, while $\gamma_1$ and $\gamma_1'$ are defined in \eqref{gamma}.
\end{theorem}

In the above theorem, we derive error bounds for stochastic approximation with operator-valued kernels under two step-size strategies: the decaying step size and the constant step size. The error estimates for both the prediction error and estimation error are derived under mild assumptions. 
Unlike prior work \cite{brogat2022vector,shi2024learning,yang2025learning} that focuses on specific kernels or linear models, our analysis establishes general error bounds under fewer restrictions, demonstrating the effectiveness of the stochastic approximation framework for nonlinear operator learning. 
In particular, Theorem \ref{Thm1} applies under
Assumptions \ref{a0} and \ref{a3} and
kernel condition \eqref{kernel}, without requiring
compactness of the associated integral operator.

We now provide the convergence rates of the misspecification error.

\begin{theorem} \label{thm3}
Let $T\geq1$ and $0<\beta<1$. Suppose Assumption \ref{a0} holds with $\sigma^2>0$, Assumption \ref{a3} holds with $r>\frac{\beta}{2}$ and $g^\dagger\in L^2(\X, \rho_{\X};\Y)$. Then the following results hold:
\begin{itemize}
    \item[(1)] If we choose the step sizes $\left\{\eta_t=\eta_1t^{-\theta}\right\}_{t\geq1}$ with $0<\eta_1<\min\l\{\|L_K\|^{-1},1-\theta,\gamma_1\r\}$ and  $\frac{\beta}{1+\beta}<\theta<1$, then the misspecification error satisfies
    \begin{equation*}
    \begin{aligned}
        \be_{z^T}\l[\l\|h_{T+1}-h^\dagger\r\|_{\beta,2}^2\r]\leq&c_2\eta_1^{-(2r-\beta)}
            \begin{cases}
                (T+1)^{\beta-\theta(1+\beta)}f_1(T), &\text{ if } \frac{\beta}{1+\beta}<\theta\leq\frac{\min\l\{2r, 1\r\}}{1+\min\l\{2r, 1\r\}}, \\
                (T+1)^{-\min\l\{2r-\beta,1-\beta\r\}(1-\theta)}, &\text{ if } \frac{\min\l\{2r, 1\r\}}{1+\min\l\{2r, 1\r\}}<\theta<1,
            \end{cases}
    \end{aligned}
    \end{equation*}
    where 
    \begin{equation*}
        f_1(T)=
        \begin{cases}
            \log(T+1), &\text{ if } \theta=\frac{1}{2}, \\
            1, &\text{ otherwise}.
        \end{cases}
    \end{equation*}
    \item[(2)] If we choose the step sizes $\{\eta_t=\eta_1\}_{t\in\bn_T}$ with $\eta_1 = \eta T^{-\theta'}$, $0<\eta<\min\l\{\|L_K\|^{-1},1,\gamma_1'\r\}$, and $\frac{\beta}{1+\beta}<\theta'<1$, then the misspecification error satisfies
    \begin{equation*}
    \begin{aligned}
        \be_{z^T}\l[\l\|h_{T+1}-h^\dagger\r\|_{\beta,2}^2\r]
        \leq c_2'\eta^{-(2r-\beta)}
            \begin{cases}
                T^{\beta-\theta'(1+\beta)}, &\text{ if } \frac{\beta}{1+\beta}<\theta'\leq\frac{2r}{2r+1}, \\
                T^{-(2r-\beta)(1-\theta')}, &\text{ if } \frac{2r}{2r+1}<\theta'<1.
            \end{cases}
    \end{aligned}
    \end{equation*}
\end{itemize}
Here the constants $c_2$ and $c_2'$ are independent of $T$, $\eta_1$, and $\eta$, while $\gamma_1$ and $\gamma_1'$ are defined in \eqref{gamma}.
\end{theorem}

The cases \(\beta=0\) and \(\beta=1\) correspond to the two endpoint scales, namely the prediction and RKHS estimation scales. These endpoints are treated separately in Theorem \ref{Thm1}. 
By strengthening Assumption \ref{a0} to Assumption \ref{a1} and imposing additional spectral conditions on the integral operator 
$L_K$, we obtain sharper error bounds.

\begin{assumption} \label{a1}
For almost all $x\in\X$, $\be_{y\sim\rho(y|x)}\l[\|y-h^\dagger(x)\|^2_{\Y}\r]\leq\sigma^2$.
\end{assumption}    
This assumption is slightly stronger than Assumption \ref{a0}. 
It requires a uniform almost-sure upper bound
on the conditional second moment of the noise.

\begin{assumption}\label{a2}
        There exists $0< s\leq1$ such that $\mathrm{Tr}(L_K^s)<\infty$.
    \end{assumption}

This is a Schatten-type capacity condition, which in particular
implies that $L_K$ is compact. Assumption~\ref{a2} holds with
$s=1$ if $K(x,x)$ is trace class for $\rho_{\X}$-almost every
$x\in\X$ and 
\[
\int_{\X}\operatorname{Tr}(K(x,x))\,\mathrm{d}\rho_{\X}(x)
<\infty;
\]
see \cite[Corollary 4.6]{carmeli2006vector}. If $L_K$ has finite
rank, then Assumption~\ref{a2} holds for every $0<s\leq1$.
A typical example is $K(x,x')=k(x,x')T,$
where $k$ is a scalar-valued kernel satisfying $\int_{\X}k(x,x)\,\mathrm{d}\rho_{\X}(x)<\infty$ and $T$ is a nonnegative trace-class operator on $\Y$. In this
case, Assumption~\ref{a2} holds with $s=1$. In particular, if
$\dim\Y<\infty$, then the uniform boundedness condition
\eqref{kernel} gives
\[
\int_{\X}\operatorname{Tr}(K(x,x))\,\mathrm{d}\rho_{\X}(x)
\leq
(\dim\Y)\kappa^2,
\]
and hence Assumption~\ref{a2} holds with $s=1$.

Let $\{\sigma_n\}_{n\geq1}$ denote the nonzero eigenvalues of
$L_K$, repeated according to multiplicity and arranged in
nonincreasing order. Assumption~\ref{a2} implies $\sigma_n
\leq
\operatorname{Tr}(L_K^s)^{1/s}n^{-1/s}.$ 
Moreover, for every $\lambda>0$, the effective dimension
\[
\begin{aligned}
\mathcal{N}_{L_K}(\lambda)
&:=
\operatorname{Tr}\big((L_K+\lambda I)^{-1}L_K\big)\\
&=
\sum_{n\geq1}\frac{\sigma_n}{\sigma_n+\lambda}\\
&\leq
\lambda^{-s}\sum_{n\geq1}\sigma_n^s
=
\lambda^{-s}\operatorname{Tr}(L_K^s).
\end{aligned}
\]
Here we used $\frac{u}{u+\lambda}
\leq
\left(\frac{u}{\lambda}\right)^s$ for any $u\geq0$ and $0<s\leq1$. For $0<s<1$, the order conditions
\[
\sigma_n=O(n^{-1/s})
\quad\text{and}\quad
\mathcal{N}_{L_K}(\lambda)=O(\lambda^{-s})
\quad\text{as }\lambda\downarrow0
\]
are equivalent up to multiplicative constants. However, these
order conditions are generally weaker than the summability
condition $\operatorname{Tr}(L_K^s)<\infty$ at the same exponent.

We now present improved bounds on the prediction error and estimation error.

\begin{theorem} \label{Thm2}
Let $T\geq 1$. Suppose Assumption \ref{a3} holds with $r\geq\frac{1}{2}$ and
$g^\dagger\in L^2(\X, \rho_{\X};\Y)$, Assumption \ref{a1} holds with
$\sigma^2>0$, and Assumption \ref{a2} holds with $0< s\leq1$.
Then the following results hold:
\begin{itemize}
\item[(1)]
If the step sizes are chosen as
$\{\eta_t=\eta_1t^{-\theta}\}_{t\geq1}$
with $0<\theta<1$ and
$0<\eta_1<\min\{\kappa^{-2},1-\theta\}$,
then the following prediction error bounds hold.

If $0<s<1$, then
\begin{equation*}
\be_{z^T}\l[\E(h_{T+1})-\E(h^\dagger)\r]
\leq c_3\eta_1^{-2r}
\begin{cases}
(T+1)^{-\theta},
&\text{if }0<\theta\leq
\frac{\min\{2r,2-s\}}{1+\min\{2r,2-s\}},\\
(T+1)^{-\min\{2r,2-s\}(1-\theta)},
&\text{if }
\frac{\min\{2r,2-s\}}{1+\min\{2r,2-s\}}<\theta<1.
\end{cases}
\end{equation*}
If $s=1$, then
\begin{equation*}
\be_{z^T}\l[\E(h_{T+1})-\E(h^\dagger)\r]
\leq c_3\eta_1^{-2r}
\begin{cases}
(T+1)^{-\theta},
&\text{if }0<\theta<\frac12,\\
(T+1)^{-1/2}\log(T+1),
&\text{if }\theta=\frac12,\\
(T+1)^{-(1-\theta)},
&\text{if }\frac12<\theta<1.
\end{cases}
\end{equation*}

Moreover, if $r>\frac12$, $0<s<1$, and $\frac{s}{1+s}<\theta<1$, then the estimation error satisfies
\begin{equation*}
        \be_{z^T}\l\|h_{T+1}-h^\dagger\r\|_K^2
        \leq c_3\eta_1^{-(2r-1)}
        \begin{cases}
            (T+1)^{s-(1+s)\theta}f_2(T),
            &\text{ if } \frac{s}{1+s}<\theta\leq
            \min\l\{\frac{2r+s-1}{2r+s},\frac{1}{2}\r\}, \\[1mm]
            (T+1)^{-\min\l\{2r-1,1-s\r\}(1-\theta)},
            &\text{ if }
            \min\l\{\frac{2r+s-1}{2r+s},\frac{1}{2}\r\}<\theta<1,
        \end{cases}
    \end{equation*}
where
\begin{equation*}
f_2(T)=
    \begin{cases}
        \log(T+1), &\text{ if } \theta=\frac{1}{2}, \\
        1, &\text{ otherwise}.
    \end{cases}
\end{equation*}

\item[(2)]
If the step sizes are chosen as
$\{\eta_t=\eta_1\}_{t\in\bn_T}$
with $\eta_1=\eta T^{-\theta'}$,
$0<\theta'<1$, and
$0<\eta<\min\{\kappa^{-2},1\}$,
then the prediction error satisfies
\begin{equation*}
\be_{z^T}\l[\E(h_{T+1})-\E(h^\dagger)\r]
\leq c_3'\eta^{-2r}
\begin{cases}
(T+1)^{-\theta'},
&\text{if }0<\theta'\leq\frac{2r}{2r+1},\\
(T+1)^{-2r(1-\theta')},
&\text{if }\frac{2r}{2r+1}<\theta'<1.
\end{cases}
\end{equation*}
Moreover, if $r>\frac12$ and $\frac{s}{1+s}<\theta'<1$, then the estimation error satisfies
\begin{equation*}
    \be_{z^T}\l\|h_{T+1}-h^\dagger\r\|_K^2
    \leq c_3'\eta^{-(2r-1)}
    \begin{cases}
        (T+1)^{s-(1+s)\theta'},
        &\text{ if } \frac{s}{1+s}<\theta'\leq
        \frac{2r+s-1}{2r+s}, \\[1mm]
        (T+1)^{-(2r-1)(1-\theta')},
        &\text{ if } \frac{2r+s-1}{2r+s}<\theta'<1.
    \end{cases}
\end{equation*}

\end{itemize}
Here the constants $c_3$ and $c_3'$ are independent of
$T$, $\eta_1$, and $\eta$.
\end{theorem}

The step-size exponents can be chosen to balance
the terms in the corresponding upper bounds.
Under the assumptions of Theorem \ref{Thm1},
prediction error converges for $0<\theta,\theta'<1$.
For $r>1/2$, the condition $\theta'>1/2$
is sufficient for finite-horizon estimation
error convergence.
Under the assumptions of Theorem \ref{thm3},
$\theta,\theta'>\beta/(1+\beta)$ are sufficient
for convergence in $\|\cdot\|_{\beta,2}$. Under the stronger assumptions of Theorem \ref{Thm2},
online estimation error converges when $r>1/2$,
$0<s<1$, and $\theta>s/(1+s)$.
Finite-horizon estimation error converges when
$r>1/2$, $0<s\leq1$, and $\theta'>s/(1+s)$.
The convergence exponents in the online upper bounds
stop improving with $r$ beyond $r=1/2$ in
Theorems \ref{Thm1} and \ref{thm3}, and beyond
$r=1-s/2$ in Theorem \ref{Thm2}.

Under Assumptions \ref{a3}, \ref{a1}, and \ref{a2}, with
$r\geq1/2$, the finite-horizon analysis also yields error bounds
throughout the interpolation scale.

\begin{corollary}\label{cor:finite-interpolation}
Let $T\geq1$ and $0<\beta<1$.
Suppose Assumption \ref{a3} holds with $r\geq1/2$ and
$g^\dagger\in L^2(\X,\rho_{\X};\Y)$,
Assumption \ref{a1} holds with $\sigma^2>0$, and
Assumption \ref{a2} holds with $0<s\leq1$.
Choose the step sizes $\{\eta_t=\eta_1\}_{t\in\bn_T}$
with $\eta_1=\eta T^{-\theta'}$, $0<\eta<\min\{\kappa^{-2},1\}$ and
$0<\theta'<1$.
Then there exists a constant $c_\beta>0$,
independent of $T$ and $\eta$, such that
\begin{equation}\label{eq:finite-interpolation-bound}
\begin{aligned}
&\be_{z^T}\l[
\l\|h_{T+1}-h^\dagger\r\|_{\beta,2}^2
\r]\\
&\quad\leq
c_\beta\left[
\eta^{-(2r-\beta)}
T^{-(2r-\beta)(1-\theta')}
+
\begin{cases}
\eta T^{-\theta'},
& s+\beta\leq1,\\[1mm]
\eta^{s+\beta}
T^{s+\beta-1-(s+\beta)\theta'},
& s+\beta>1.
\end{cases}
\right]
\end{aligned}
\end{equation}

In particular, choosing
\begin{equation}\label{eq:finite-interpolation-step}
\theta'=
\begin{cases}
\dfrac{2r-\beta}{2r-\beta+1},
& s+\beta\leq1,\\[2mm]
\dfrac{2r+s-1}{2r+s},
& s+\beta>1,
\end{cases}
\end{equation}
gives
\begin{equation}\label{eq:finite-interpolation-rate}
\begin{aligned}
&\be_{z^T}\l[
\l\|h_{T+1}-h^\dagger\r\|_{\beta,2}^2
\r]\leq
2c_\beta\eta^{-(2r-\beta)}
\begin{cases}
T^{-\frac{2r-\beta}{2r-\beta+1}},
& s+\beta\leq1,\\[2mm]
T^{-\frac{2r-\beta}{2r+s}},
& s+\beta>1.
\end{cases}
\end{aligned}
\end{equation}
\end{corollary}

This corollary gives finite-horizon error bounds in
intermediate norms under the well-specified condition
$r\geq1/2$. It complements Theorem \ref{thm3}, which also
covers the misspecified range $\beta/2<r<1/2$. The endpoint estimates are stated in Theorem \ref{Thm2}:
$\beta=0$ corresponds to the prediction error measured
in $\|\cdot\|_{\rho_{\X}}$, while $\beta=1$ corresponds
to the estimation error measured in $\|\cdot\|_K$.
The real interpolation norm $\|\cdot\|_{\beta,2}$
is used only for $0<\beta<1$. The constant $c_\beta$ may depend on
$r,s,\beta,\kappa,\sigma^2$,
$\|g^\dagger\|_{\rho_{\X}}$, and
$\operatorname{Tr}(L_K^s)$, but is independent of
$T$ and $\eta$.

We remark that when the kernel takes the form $K(x,x') = k(x,x') I$, Assumption~\ref{a2} can also be imposed on the scalar-valued integral operator $L_k$, which leads to improved convergence bounds; see \cite{li2024towards,meunier2024optimal,yang2025learning}.

\section{Discussion and Numerical Experiments} \label{section 3}
In this section, we present two representative examples of operator learning, corresponding respectively to the two cases following \eqref{kernel}: (i) learning Green's functions (and, more generally, Fredholm integral operators) with compact kernels, and (ii) learning through encoder--decoder frameworks with diagonal kernels. 
We subsequently provide an illustrative numerical experiment for the proposed algorithm on the Navier--Stokes equations.

\subsection{Learning Green's Functions}
\label{Learning PDEs}

Learning partial differential equations (PDEs) is an emerging field
at the intersection of machine learning and applied mathematics.
Traditional numerical methods, such as finite-difference and
finite-element schemes, solve individual PDE instances accurately,
but repeated numerical solves can be computationally expensive when
many related instances are considered. In contrast, operator learning
seeks to approximate the solution operator that maps input data, such
as forcing terms or boundary conditions, to the corresponding
solutions. This enables rapid prediction for new inputs without
repeatedly solving the PDE.

As a motivating example, consider the following time-independent PDE:
\begin{equation*}
\left\{
\begin{aligned}
\mathcal{L}u &= f && \text{in }D,\\
\mathcal{B}u &= 0 && \text{on }\partial D,
\end{aligned}
\right.
\end{equation*}
where $D\subset\mathbb{R}^d$ is a bounded domain, $\mathcal{L}$ is a
linear differential operator, and $\mathcal{B}$ specifies the boundary
conditions. Assuming well-posedness, this PDE induces a solution
operator $h^\dagger$ mapping $f$ to $u$. If there exists a Green's
function $G^\dagger\in L^2(D\times D)$ such that
\[
h^\dagger(f)(y)=\int_DG^\dagger(y,x)f(x)\,\mathrm{d}x
\quad\text{for almost every }y\in D,
\]
then $h^\dagger:L^2(D)\to L^2(D)$ is a Hilbert--Schmidt operator. More
generally, consider the Fredholm integral operator
\[
h^\dagger(f)(y)=\int_{D_{\X}}G^\dagger(y,x)f(x)\,\mathrm{d}x
\quad\text{for almost every }y\in D_{\Y},
\]
where $D_{\X}$ and $D_{\Y}$ are bounded domains in Euclidean spaces and
$G^\dagger\in L^2(D_{\Y}\times D_{\X})$ is unknown. The corresponding
operator $h^\dagger:L^2(D_{\X})\to L^2(D_{\Y})$ is Hilbert--Schmidt,
and hence compact and bounded. The PDE setting above corresponds to the
special case $D_{\X}=D_{\Y}=D$. We restrict attention here to
square-integrable Green kernels; distributional Green kernels are
outside the scope of this subsection.

For the statistical learning problem, let $\rho$ be a probability
measure on $L^2(D_{\X})\times L^2(D_{\Y})$, and set
$\X:=\operatorname{supp}\rho_{\X}\subset L^2(D_{\X})$ and
$\Y:=L^2(D_{\Y})$, where $\X$ is endowed with the topology inherited
from $L^2(D_{\X})$. As in the general regression setting, assume that
$\mathbb{E}_{(f,u)\sim\rho}\|u\|_{L^2(D_{\Y})}^2<\infty$ and
$\mathbb{E}[u\mid f]=h^\dagger(f)$ for $\rho_{\X}$-almost every $f$.
Thus, the observations may be noisy, while $h^\dagger$ remains the
regression operator. We further assume
\[
R:=\sup_{f\in\X}\|f\|_{L^2(D_{\X})}<\infty,
\qquad
\overline{\operatorname{span}(\X)}^{\,L^2(D_{\X})}=L^2(D_{\X}).
\]
The first condition will be used to verify the uniform boundedness of
the induced operator-valued kernel. The second is an identifiability
condition: two bounded linear operators that agree $\rho_{\X}$-almost
surely agree on $\X=\operatorname{supp}\rho_{\X}$ by continuity, and
hence agree on all of $L^2(D_{\X})$ by density. Their
Hilbert--Schmidt kernels therefore coincide almost everywhere.
Consequently, learning $h^\dagger$ from i.i.d. observations
$\{(f_i,u_i)\}_{i=1}^N\stackrel{\mathrm{i.i.d.}}{\sim}\rho$ may be
interpreted, at the population level, as estimating $G^\dagger$. 

To estimate $G^\dagger$ from data, we adopt the kernel-based framework
developed in this paper. Let
$k:(D_{\Y}\times D_{\X})\times(D_{\Y}\times D_{\X})\to\mathbb{R}$ be
a measurable, symmetric, positive-semidefinite kernel. We do not require
$k$ to be continuous, but assume
$k\in L^2((D_{\Y}\times D_{\X})\times(D_{\Y}\times D_{\X}))$. Let
$L_k:L^2(D_{\Y}\times D_{\X})\to L^2(D_{\Y}\times D_{\X})$ be the
associated integral operator,
\[
(L_kq)(y,x):=\int_{D_{\Y}}\int_{D_{\X}}
 k(y,x,\zeta,\xi)q(\zeta,\xi)\,\mathrm{d}\xi\,\mathrm{d}\zeta.
\]
Then $L_k$ is a nonnegative self-adjoint Hilbert--Schmidt operator, and
hence compact. Define the range space
\[
\mathcal{G}_k:=\operatorname{ran}L_k^{1/2},
\qquad
\langle F,G\rangle_{\mathcal{G}_k}:=
\big\langle L_k^{-1/2}F,L_k^{-1/2}G\big\rangle_{L^2(D_{\Y}\times D_{\X})},
\]
where $L_k^{-1/2}:\mathcal{G}_k\to\ker L_k^\perp$ denotes the inverse
of the restriction of $L_k^{1/2}$ to $\ker L_k^\perp$. With this norm,
$\mathcal{G}_k$ is a Hilbert space continuously embedded in
$L^2(D_{\Y}\times D_{\X})$, and
$\|G\|_{L^2(D_{\Y}\times D_{\X})}\leq
\|L_k\|^{1/2}\|G\|_{\mathcal{G}_k}$. Its elements are
$L^2$-equivalence classes, so no pointwise continuity is imposed.

\begin{remark}
Let $\mathcal H_k$ be the scalar RKHS induced by $k$ and
let $J:\mathcal H_k\to L^2(D_{\Y}\times D_{\X})$
be the canonical map.
If $J$ is bounded, then $L_k=JJ^*$ and
$\mathcal G_k=\operatorname{ran}L_k^{1/2}=J(\mathcal H_k)$,
with
\[
\|G\|_{\mathcal G_k}
=
\min\bigl\{
\|q\|_{\mathcal H_k}:
q\in\mathcal H_k,\ Jq=G
\bigr\},
\qquad G\in\mathcal G_k.
\]
If $J$ is injective, then
$\|Jq\|_{\mathcal G_k}=\|q\|_{\mathcal H_k}$
for every $q\in\mathcal H_k$.
In particular, for bounded continuous kernels on
$D_{\Y}\times D_{\X}$, such as Gaussian and Mat\'ern
kernels, $J$ identifies $\mathcal H_k$ isometrically
with $\mathcal G_k$.

The construction also includes discontinuous finite-rank
kernels, such as
$k(z,z')=\sum_{j=1}^m\lambda_j\psi_j(z)\psi_j(z')$,
where $\lambda_j>0$ and
$\psi_j\in L^2(D_{\Y}\times D_{\X})$ are measurable.
The functions $\psi_j$ may be chosen as piecewise
polynomial functions or wavelets.
\end{remark}

For $G\in\mathcal{G}_k$, define the integral operator and the error functional
\[
h_G(f):=\int_{D_{\X}}G(\cdot,x)f(x)\,\mathrm{d}x\in L^2(D_{\Y}),
\qquad
\mathcal{E}(G):=\mathbb{E}_{(f,u)\sim\rho}
\big[\|u-h_G(f)\|_{L^2(D_{\Y})}^2\big],
\]
respectively.
The target $G^\dagger$ is not required to belong to $\mathcal{G}_k$,
allowing the resulting regression problem to be misspecified. For
$v\in L^2(D_{\Y})$ and $f\in L^2(D_{\X})$, we identify the rank-one
Hilbert--Schmidt operator $v\otimes f$ with its kernel
\[
(v\otimes f)(\zeta,\xi):=v(\zeta)f(\xi).
\]
A direct calculation shows
that $\mathcal{E}$ is Fr\'echet differentiable on $\mathcal{G}_k$ and
\begin{equation*} 
\nabla_{\mathcal{G}_k}\mathcal{E}(G)
=2L_k\,\mathbb{E}_{(f,u)\sim\rho}
\big[(h_G(f)-u)\otimes f\big].
\end{equation*}
Absorbing the factor $2$ into the step size and initializing $G_1=0$,
we obtain
\begin{equation}\label{3.2}
\begin{aligned}
G_{t+1}
&=G_t-\eta_tL_k\big((h_{G_t}(f_t)-u_t)\otimes f_t\big)\\
&=G_t-\eta_t\int_{D_{\Y}}\int_{D_{\X}}
 k(\cdot,\cdot,\zeta,\xi)
 \big(h_{G_t}(f_t)(\zeta)-u_t(\zeta)\big)f_t(\xi)
 \,\mathrm{d}\xi\,\mathrm{d}\zeta,
\end{aligned}
\end{equation}
where the second equality is understood in
$L^2(D_{\Y}\times D_{\X})$. Since
$\operatorname{ran}L_k\subseteq\operatorname{ran}L_k^{1/2}
=\mathcal{G}_k$, the iteration is well defined and
$G_t\in\mathcal{G}_k$ for all $t\geq1$. Setting $h_t:=h_{G_t}$ yields
\begin{equation}\label{3.3}
\begin{aligned}
h_{t+1}(f)
={}&h_t(f)-\eta_t\int_{D_{\X}}\int_{D_{\Y}}\int_{D_{\X}}
 k(\cdot,x,\zeta,\xi)
 \big(h_t(f_t)(\zeta)-u_t(\zeta)\big)\\
&\hspace{4cm}\times f_t(\xi)f(x)
 \,\mathrm{d}\xi\,\mathrm{d}\zeta\,\mathrm{d}x,
\end{aligned}
\end{equation}
for every $f\in L^2(D_{\X})$, where the equality is understood in
$L^2(D_{\Y})$ and $h_1=0$.

For $f\in L^2(D_{\X})$, define
$B_f:L^2(D_{\Y})\to L^2(D_{\Y}\times D_{\X})$ by
$B_fg:=g\otimes f$. Then $B_f^*G=h_G(f)$ and
$\|B_f\|=\|f\|_{L^2(D_{\X})}$. The following proposition shows that
\eqref{3.3} is an instance of the general stochastic
approximation framework.

\begin{proposition}
\label{Green}
Let $k$, $L_k$, $\mathcal{G}_k$, and $\X$ satisfy the assumptions above.
The map $G\mapsto h_G|_{\X}$ is injective. Consequently,
\[
\mathcal{H}_K:=\{h_G|_{\X}:G\in\mathcal{G}_k\}
\subset C\big(\X;L^2(D_{\Y})\big),
\qquad
\langle h_F,h_G\rangle_K:=\langle F,G\rangle_{\mathcal{G}_k},
\]
is an $L^2(D_{\Y})$-valued RKHS isometrically isomorphic to
$\mathcal{G}_k$, with reproducing kernel
\[
K(f_1,f_2)=B_{f_1}^*L_kB_{f_2},\qquad f_1,f_2\in\X.
\]
Equivalently, for $g\in L^2(D_{\Y})$,
\[
[K(f_1,f_2)g](y)=\int_{D_{\X}}\int_{D_{\Y}}\int_{D_{\X}}
 k(y,x,\zeta,\xi)g(\zeta)f_1(x)f_2(\xi)
 \,\mathrm{d}\xi\,\mathrm{d}\zeta\,\mathrm{d}x
\]
for almost every $y\in D_{\Y}$, and
\[
K(\cdot,f)g=h_{L_k(g\otimes f)},
\qquad
\langle K(\cdot,f)g,h_G\rangle_K
=\langle h_G(f),g\rangle_{L^2(D_{\Y})}.
\]
Moreover, $K$ is jointly continuous in operator norm and hence Mercer,
each $K(f_1,f_2)$ is Hilbert--Schmidt and hence compact, and
\[
\|h_G(f)\|_{L^2(D_{\Y})}
\leq\|L_k\|^{1/2}\|G\|_{\mathcal{G}_k}\|f\|_{L^2(D_{\X})},
\qquad
\|K(f_1,f_2)\|
\leq\|L_k\|\|f_1\|_{L^2(D_{\X})}\|f_2\|_{L^2(D_{\X})}.
\]
In particular, $\sup_{f\in\X}\|K(f,f)\|\leq R^2\|L_k\|$. Finally,
let $\mathsf{S}:L^2(\X,\rho_{\X};L^2(D_{\Y}))\to
L^2(D_{\Y}\times D_{\X})$ be given by
\[
\mathsf{S}\Phi:=\int_{\X}\Phi(f)\otimes f\,\mathrm{d}\rho_{\X}(f).
\]
Then $\|\mathsf{S}\|\leq R$ and the integral operator associated with
$K$ satisfies $L_K=\mathsf{S}^*L_k\mathsf{S}$. Consequently, $L_K$ is
Hilbert--Schmidt, and hence compact.
\end{proposition}

The proof of Proposition~\ref{Green} is deferred to
Appendix~\ref{Appendix 3}. Its reproducing identity shows that the
restriction of iteration~\eqref{3.3} to $\X$ coincides with the
general iteration~\eqref{algorithm} instantiated with the kernel $K$. The factorization above also transfers Schatten conditions from
$L_k$ to $L_K$. More precisely, for every $0<s\leq1$,
\begin{equation*}
\operatorname{Tr}(L_K^s)
\leq
\|\mathsf{S}\|^{2s}\operatorname{Tr}(L_k^s)
\leq
R^{2s}\operatorname{Tr}(L_k^s),
\label{green-schatten-bound}
\end{equation*}
whenever $\operatorname{Tr}(L_k^s)<\infty$. Hence the capacity
condition in Assumption~\ref{a2} can be verified directly through the
scalar integral operator $L_k$.

\begin{corollary}
\label{Green corollary}
In the setting above, $K$ satisfies condition~\eqref{kernel} with
$\kappa^2=R^2\|L_k\|$. The assumption
$\be_{(f,u)\sim\rho}[\|u\|_{L^2(D_{\Y})}^2]<\infty$
implies Assumption~\ref{a0} for some finite $\sigma^2$. If Assumption~\ref{a3} holds for the
induced regression target and integral operator $L_K$, then the
conclusions of Theorems~\ref{Thm1} and~\ref{thm3} hold under their corresponding
parameter conditions. If Assumptions~\ref{a1} and~\ref{a2} also hold, then the
conclusions of Theorem~\ref{Thm2} hold under its corresponding parameter
conditions.
\end{corollary}

The factorization $L_K=\mathsf S^*L_k\mathsf S$
provides examples satisfying Assumption \ref{a2}.
If $L_k$ has finite rank, the assumption holds for every
$0<s\leq1$; if $L_k$ is trace class, it holds with $s=1$. 
If $G^\dagger\in\mathcal{G}_k$, then
$h^\dagger=h_{G^\dagger}|_{\X}\in\mathcal{H}_K$ and
\[
\|h_G-h^\dagger\|_K=\|G-G^\dagger\|_{\mathcal{G}_k},
\qquad
\|G-G^\dagger\|_{L^2(D_{\Y}\times D_{\X})}
\leq\|L_k\|^{1/2}\|h_G-h^\dagger\|_K.
\]
Thus, in the well-specified case, a bound of estimation error yields
an $L^2$ recovery bound for the Green kernel.  Obtaining such a recovery bound from the prediction error requires an additional stability condition. In numerical implementations, the integrals in
\eqref{3.2} and \eqref{3.3} are typically replaced by
suitable discrete approximations; see
\cite{stepaniants2023learning} for related implementations.
\subsection{Operator Learning via Encoder-–Decoder Frameworks} \label{Example: Learning via Encoder–Decoder Frameworks}

\begin{figure}
\centering

\tikzset{every picture/.style={line width=0.75pt}} 

\begin{tikzpicture}[x=0.75pt,y=0.75pt,yscale=-1,xscale=1]

\draw [color={rgb, 255:red, 208; green, 2; blue, 27 }  ,draw opacity=1 ]   (140,85) -- (253,85) ;
\draw [shift={(255,85)}, rotate = 180] [color={rgb, 255:red, 208; green, 2; blue, 27 }  ,draw opacity=1 ][line width=0.75]    (10.93,-3.29) .. controls (6.95,-1.4) and (3.31,-0.3) .. (0,0) .. controls (3.31,0.3) and (6.95,1.4) .. (10.93,3.29)   ;
\draw    (132.5,99) -- (133.47,161) ;
\draw [shift={(133.5,163)}, rotate = 269.1] [color={rgb, 255:red, 0; green, 0; blue, 0 }  ][line width=0.75]    (10.93,-3.29) .. controls (6.95,-1.4) and (3.31,-0.3) .. (0,0) .. controls (3.31,0.3) and (6.95,1.4) .. (10.93,3.29)   ;
\draw    (122.5,163) -- (121.53,101) ;
\draw [shift={(121.5,99)}, rotate = 89.1] [color={rgb, 255:red, 0; green, 0; blue, 0 }  ][line width=0.75]    (10.93,-3.29) .. controls (6.95,-1.4) and (3.31,-0.3) .. (0,0) .. controls (3.31,0.3) and (6.95,1.4) .. (10.93,3.29)   ;
\draw    (140,175) -- (253,175) ;
\draw [shift={(255,175)}, rotate = 175] [color={rgb, 255:red, 0; green, 0; blue, 0 }  ][line width=0.75]    (10.93,-3.29) .. controls (6.95,-1.4) and (3.31,-0.3) .. (0,0) .. controls (3.31,0.3) and (6.95,1.4) .. (10.93,3.29)   ;
\draw    (276.5,99) -- (277.47,161) ;
\draw [shift={(277.5,163)}, rotate = 269.1] [color={rgb, 255:red, 0; green, 0; blue, 0 }  ][line width=0.75]    (10.93,-3.29) .. controls (6.95,-1.4) and (3.31,-0.3) .. (0,0) .. controls (3.31,0.3) and (6.95,1.4) .. (10.93,3.29)   ;
\draw    (266.5,163) -- (266.5,101) ;
\draw [shift={(266.5,99)}, rotate = 90] [color={rgb, 255:red, 0; green, 0; blue, 0 }  ][line width=0.75]    (10.93,-3.29) .. controls (6.95,-1.4) and (3.31,-0.3) .. (0,0) .. controls (3.31,0.3) and (6.95,1.4) .. (10.93,3.29)   ;

\draw (120,76.4) node [anchor=north west][inner sep=0.75pt]    {$\mathcal{X}$};
\draw (264,77.4) node [anchor=north west][inner sep=0.75pt]    {$\mathcal{Y}$};
\draw (114,166.4) node [anchor=north west][inner sep=0.75pt]    {$\mathbb{R}^{m}$};
\draw (261,169.4) node [anchor=north west][inner sep=0.75pt]    {$\mathbb{R}^{n}$};
\draw (190,63.4) node [anchor=north west][inner sep=0.75pt]    {$\textcolor[rgb]{0.82,0.01,0.11}{h^{\dagger }}$};
\draw (137,125) node [anchor=north west][inner sep=0.75pt]  [font=\normalsize]  {$\phi$};
\draw (105,122.4) node [anchor=north west][inner sep=0.75pt]  [font=\normalsize]  {$\widehat{\phi}$};
\draw (281,125) node [anchor=north west][inner sep=0.75pt]  [font=\normalsize]  {$\varphi$};
\draw (250,122.4) node [anchor=north west][inner sep=0.75pt]  [font=\normalsize]  {$\widehat{\varphi}$};
\draw (190,156.4) node [anchor=north west][inner sep=0.75pt]  [color={rgb, 255:red, 208; green, 2; blue, 27 }  ,opacity=1 ]  {$g^{\dagger }$};

\end{tikzpicture}
\caption{Schematic illustration of the encoder--decoder operator learning framework in Subsection~\ref{Example: Learning via Encoder–Decoder Frameworks}.}
\label{fig1}

\end{figure}

Let $\X$ and $\Y$ be function spaces defined on domains $D$ and $D'$, respectively, and let $h^\dagger:\X\to \Y$ denote the target operator we aim to approximate. In practice, the functions $f\in\X$ and $u=h^\dagger(f)\in\Y$ are often not observed continuously but only through a finite number of measurements. This is common in applications where data are collected at discrete spatial or temporal locations \cite{lu2021learning,batlle2024kernel,liu2024deep,zhou2024approximation}.

To formalize this, we assume that the following point evaluations are well defined and induce bounded linear measurement operators $\phi:\X\to \mathbb{R}^m$ and $\varphi:\Y \to \mathbb{R}^n$, which map $f$ and $u$ to their evaluations at the prescribed points $\{\xi_i\}_{i=1}^m \subset D$ and $\{\xi'_j\}_{j=1}^n \subset D'$, respectively:
\[
\phi(f) := \big(f(\xi_1),\dots,f(\xi_m)\big), \quad 
\varphi(u) := \big(u(\xi'_1),\dots,u(\xi'_n)\big).
\]
Given a dataset $\{ (\phi(f_i), \varphi(u_i)) \}_{i=1}^N$, our goal is to approximate $h^\dagger$ through the regression problem induced by these discrete measurements.
To lift the discrete data back to continuous function spaces, we employ the  minimal-norm interpolation operators \cite{micchelli2005learning}
\[
\widehat{\phi}: \mathbb{R}^m \to \X, \quad
\widehat{\varphi}: \mathbb{R}^n \to \Y,
\]
associated with the reproducing kernels $K$ of $\X$ and $Q$ of $\Y$, respectively. 
Assuming that the Gram matrices $K(\Xi,\Xi)$ and $Q(\Xi',\Xi')$ are invertible, these operators, for all coefficient vectors $c$ and $c'$, satisfy
\[
\widehat{\phi}(c)(\xi) = K(\xi, \Xi) K(\Xi, \Xi)^{-1} c, \quad
\widehat{\varphi}(c')(\xi') = Q(\xi', \Xi') Q(\Xi', \Xi')^{-1} c',
\]
where $K(\Xi, \Xi)$ and $Q(\Xi', \Xi')$ are the kernel matrices with entries $K(\xi_i, \xi_j)$ and $Q(\xi'_i, \xi'_j)$, and $K(\xi, \Xi)$, $Q(\xi', \Xi')$ are row vectors of kernel evaluations. 
By the minimal-norm construction and the invertibility of the Gram matrices,
\[
\phi\widehat{\phi}=I_m,\qquad
\varphi\widehat{\varphi}=I_n.
\]
In particular, $P_n:=\widehat{\varphi}\varphi:\Y\to\Y$ is the orthogonal projection onto $\operatorname{ran}\widehat{\varphi}$. 
The regression function associated with the measured data is
\[
g^\dagger(a)
:=
\be\big[\varphi(u)\mid\phi(f)=a\big]
=
\be\big[\varphi(h^\dagger(f))\mid\phi(f)=a\big],
\]
defined almost everywhere with respect to the distribution
of $\phi(f)$.
A schematic illustration of the framework is shown in
Figure~\ref{fig1}.
We approximate $g^\dagger$ using the measured data
$\{(\phi(f_i),\varphi(u_i))\}_{i=1}^N$
and apply the decoder $\widehat{\varphi}$
to the predicted outputs.

Let $k:[0,\infty)\to\mathbb{R}$ be a continuous function
such that $(a,a')\mapsto k(\|a-a'\|_2)$
defines a positive definite kernel on $\mathbb{R}^m$. 
Using the i.i.d.\ dataset $\{ (\phi(f_i), \varphi(u_i)) \}_{i=1}^N$, we apply the stochastic approximation algorithm \eqref{algorithm} with the matrix-valued kernel $k(\|\cdot-\cdot\|_2)I_n$, where $I_n$ is the $n \times n$ identity matrix:
\begin{equation*} 
\begin{cases}
g_1 := \mathbf{0}, \\
g_{t+1} := g_t - \eta_t \, k\big(\|\cdot - \phi(f_t)\|_2\big) \big(g_t(\phi(f_t)) - \varphi(u_t)\big).
\end{cases}
\end{equation*}
Defining $h_t := \widehat{\varphi} \circ g_t \circ \phi$, we obtain the following iterative scheme in the original function space:
\begin{equation} \label{algorithm 1}
\begin{cases}
h_1 = \mathbf{0}, \\
h_{t+1} = h_t - \eta_t \, k\big(\|\phi(\cdot - f_t)\|_2\big) P_n (h_t(f_t) - u_t),
\end{cases}
\end{equation}
This iteration can be interpreted as a stochastic approximation with the operator-valued kernel $k(\|\phi(\cdot-\cdot)\|_2) P_n$, consistent with the general form in \eqref{algorithm}. Both the matrix-valued kernel and the induced kernel
on $\X$ are Mercer and uniformly bounded by $k(0)$
in operator norm.

Define the decoded regression function by
\[
\bar h(f)
:=
\widehat{\varphi}\big(g^\dagger(\phi(f))\big)
=
P_n\be\big[h^\dagger(f)\mid\phi(f)\big],
\]
where the equality holds $\rho_{\X}$-almost surely.
For each realization of the training data,
$h_t$ depends only on $\phi(f)$ and takes values
in $\operatorname{ran}P_n$.
The orthogonality of $P_n$ and conditional expectation
therefore gives
\[
\begin{aligned}
\|h_t-h^\dagger\|_{\rho_{\X}}^2
=
\|h_t-\bar h\|_{\rho_{\X}}^2
+
\|(I-P_n)h^\dagger\|_{\rho_{\X}}^2+
\|P_nh^\dagger-\bar h\|_{\rho_{\X}}^2.
\end{aligned}
\]
The three terms are the learning error,
the output projection error, and the loss of input
information, respectively.
The last two terms are independent of $t$
for fixed measurement maps.
Moreover,
\[
\|h_t-\bar h\|_{\rho_{\X}}^2
\leq
\|\widehat{\varphi}\|^2
\be_{f\sim\rho_{\X}}\big[
\|g_t(\phi(f))-g^\dagger(\phi(f))\|_2^2
\big].
\]
After taking expectations over the training data,
our prediction error bounds therefore control the first
term, provided that the assumptions of the corresponding
theorem hold for the encoded data distribution and
the matrix-valued kernel $k(\|\cdot-\cdot\|_2)I_n$. 
Here the regularity assumption applies to $g^\dagger$,
and the noise assumptions concern the residual
$\varphi(u)-g^\dagger(\phi(f))$.

Moreover, commonly used kernels such as the Gaussian, inverse multiquadric, and Matérn kernels yield positive definite radial kernels $K$ in any dimension $m$. Similar iterative schemes also arise for PCA-based linear measurements \cite[Appendix G]{yang2025learning}. 
The same construction applies to a continuous positive
definite dot-product kernel
$q(a,a')=\psi(\langle a,a'\rangle_2)$ on $\mathbb R^m$,
provided that $\sup_{f\in\X}\psi(\|\phi(f)\|_2^2)<\infty.$ The matrix-valued kernel $q(a,a')I_n$ then satisfies
the Mercer and uniform boundedness conditions
on the support of the encoded input distribution.

\begin{remark}\label{remark 5}
Let $\X=D_{\X}\subset\mathbb R^m$ be a bounded domain
with smooth boundary.
Let $k$ be a continuous translation-invariant kernel
on $\mathbb R^m$ whose Fourier transform satisfies the decay condition
\[
\widehat{k}(\xi)
\asymp (1+\|\xi\|_2^2)^{-\ell},
\qquad \ell>m/2,
\]
e.g., the Matérn kernel. For $K(x,x')=k(x,x')I$, the RKHS $\H_K$
can be identified with the vector-valued Sobolev space $W^{\ell,2}(D_{\X};\Y)$
with equivalent norms
\cite[Section 6]{li2024towards}.

Assume additionally that
$\mathrm{d}\rho_{\X}(x)=w(x)\,\mathrm{d}x$, where
$0<c_-\leq w(x)\leq c_+<\infty$
almost everywhere on $D_{\X}$.
For $0<\tau<\ell$, Theorem \ref{thm:interpolation} and Sobolev interpolation
give
\[
\begin{aligned}
\l[\H_K\r]^{\tau/\ell}
=
\big[
L^2(D_{\X},\rho_{\X};\Y),[\H_K]^1
\big]_{\tau/\ell,2}
=
W^{\tau,2}(D_{\X};\Y),
\end{aligned}
\]
with equivalent norms.
Thus our bounds in $\|\cdot\|_{\tau/\ell,2}$
yield Sobolev norm bounds under the assumptions
of the corresponding theorem.
\end{remark}

\subsection{Numerical Experiments}

In this subsection, we illustrate our nonlinear operator learning framework through a concrete example: the two-dimensional incompressible Navier–Stokes equations in the vorticity–stream function formulation. We assume spatial periodicity on the domain $D=[0,2\pi]^2$, and denote the vorticity by $u$ and the stream function by $\phi$:
\begin{equation*}
    \begin{cases}
        \frac{\partial u}{\partial t}+(c\cdot\nabla)u-\nu\Delta u=g, \\
        u=-\Delta\phi,\quad \int_D\phi=0,\\
        c=\left(\frac{\partial\phi}{\partial x_{2}},-\frac{\partial\phi}{\partial x_{1}}\right).
    \end{cases}
\end{equation*}
Given a fixed initial condition $u(0,\cdot)$ and viscosity $\nu=0.025$, our goal is to learn the mapping from the forcing function $g$ to the vorticity field at time $t=10$, i.e., 
$h^\dagger:g\mapsto u(10,\cdot)$. 

Assume that $g$ is drawn from the Gaussian process  $\mathcal{GP}(0,(-\Delta+3^2I)^{-4})$. The dataset \footnote{The dataset is available at \url{https://data.caltech.edu/records/fp3ds-kej20} (CaltechDATA).} used in this experiment is adopted from \cite{de2022cost}, which consists of 40,000 i.i.d. input–output pairs generated by solving the Navier–Stokes equations on a $64\times64$ spatial grid. We randomly split 
the dataset into training, validation, and test sets in a 0.7:0.15:0.15 ratio. During training, we perform PCA on both inputs and outputs, retaining the top 128 components for each. In the resulting reduced space, we use the stochastic approximation with a Matérn kernel multiplied by the identity operator, considering both fixed and decaying step sizes. The kernel hyperparameters and the learning rate schedule, including initial values and decay rates, are tuned based on performance on the validation set. 
After hyperparameter selection, we combine the training
and validation sets and retrain the model on 34,000 samples,
keeping the same 6,000 samples for testing.

\begin{figure}
\centering
\includegraphics[width=1.0\linewidth]{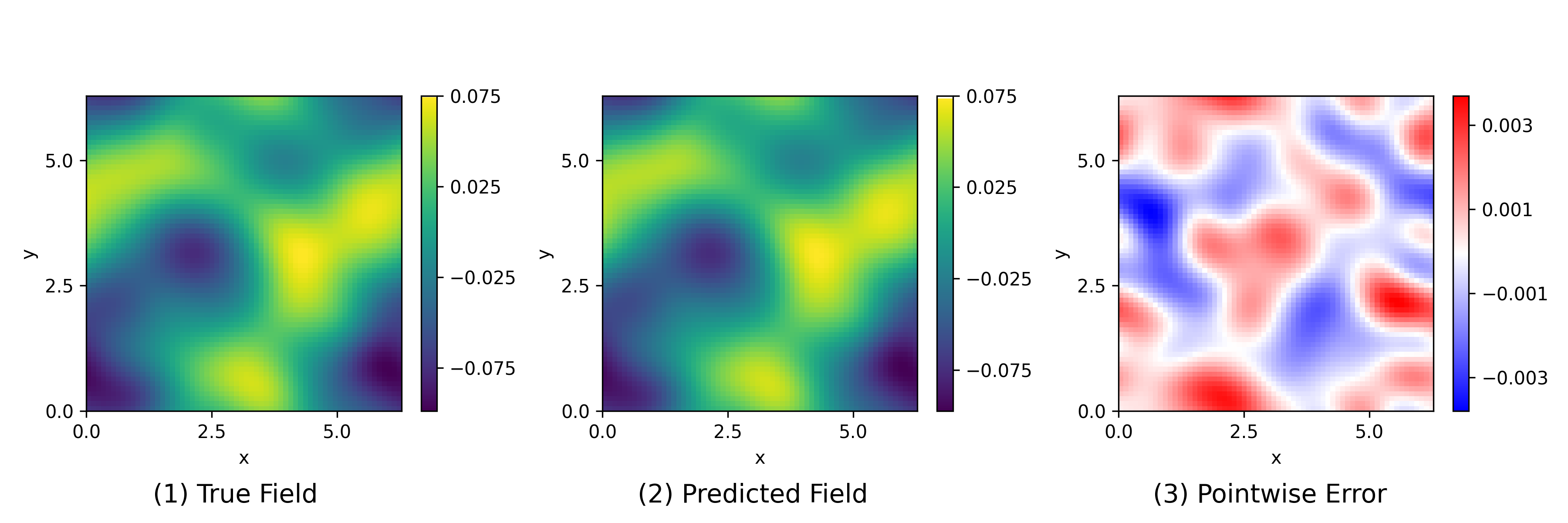}
\caption{A test sample for the Navier--Stokes problem:
the reference solution, the prediction, and the signed
pointwise error (prediction minus reference).}
\label{fig2}
\end{figure}

\begin{figure}
\centering
\includegraphics[width=1\linewidth]{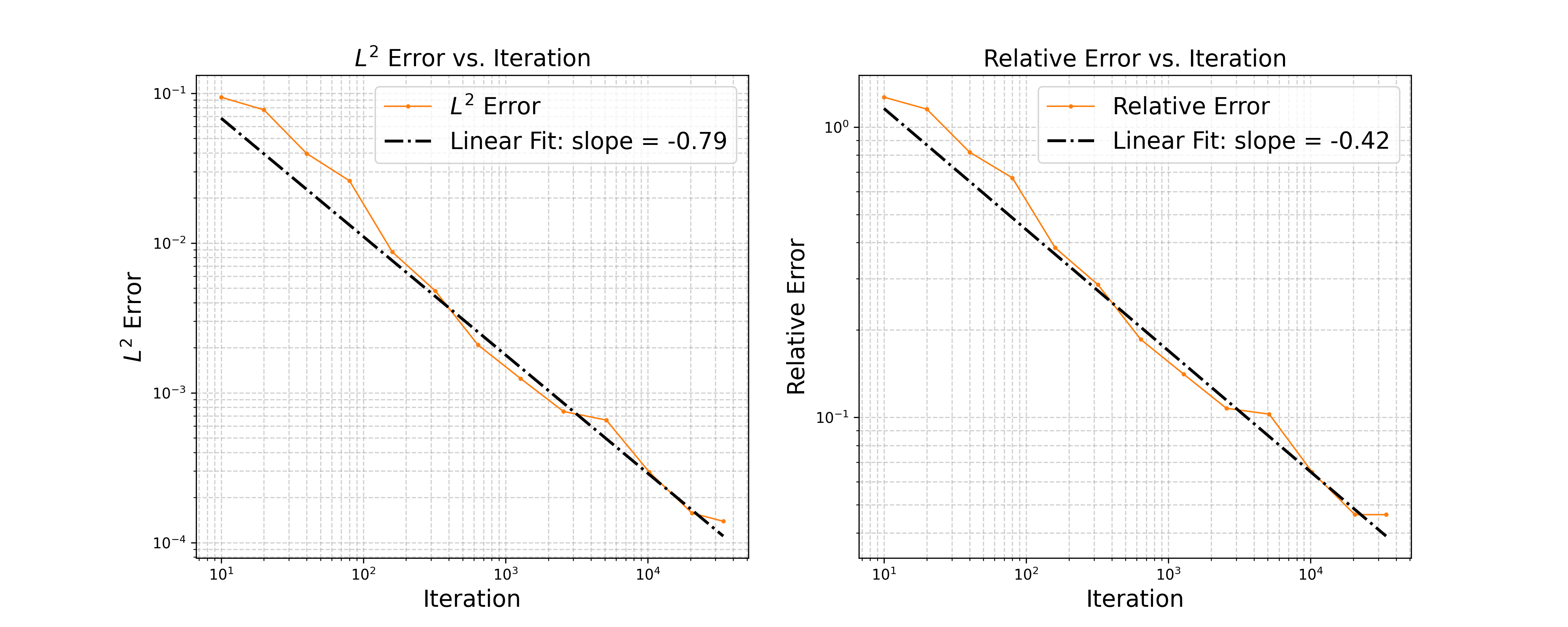}
\caption{Log--log plots of the test mean squared $L^2$
error and mean relative $L^2$ error in the online setting.
Dash-dotted lines show linear fits applied from iteration 160 to 20480,
with empirical slopes of approximately $-0.79$
and $-0.42$, respectively.}
\label{fig3}
\end{figure}

To evaluate prediction performance, we compute
the test mean squared $L^2$ error and mean relative
$L^2$ error:
\[
\begin{aligned}
\widehat{\E}(h)
:=
\frac{1}{N}\sum_{j=1}^N
\|h(g_j)-h^\dagger(g_j)\|_{L^2(D)}^2,\quad
\widehat{\E}^{\mathrm{rel}}(h)
:=
\frac{1}{N}\sum_{j=1}^N
\frac{\|h(g_j)-h^\dagger(g_j)\|_{L^2(D)}}
{\|h^\dagger(g_j)\|_{L^2(D)}}.
\end{aligned}
\]
Here $\{g_j\}_{j=1}^N$ denotes the test inputs,
and the spatial integrals are approximated on the
computational grid.
The mean relative $L^2$ errors are $4.67\%$ and $4.66\%$
in the online and finite-horizon settings, respectively.
Figure \ref{fig3} shows approximately polynomial decay
of the errors over the fitted iteration range.
These experiments illustrate the proposed stochastic
approximation scheme for PDE operator learning. 
The implementation code is available at \url{https://github.com/JiaqiYang-Fdu/Stochastic-Approximation-Operator-Learning}.

    \section{Proof of the Main Theorems} \label{section 4}
    This section presents the proofs of the main theorems
and Corollary \ref{cor:finite-interpolation}.
The auxiliary proofs are deferred to
Appendix \ref{Proofs in Section}.
We first state an error representation and
the corresponding decomposition.

\begin{lemma}\label{lemma:representation}
For any $T\geq1$, the following identity holds in
$L^2(\X,\rho_{\X};\Y)$:
\begin{equation}\label{temp1}
\begin{aligned}
\iota_K h_{T+1}-h^\dagger
=
-\prod_{t=1}^T(I-\eta_tL_K)h^\dagger
+\sum_{t=1}^T\eta_t\,
\iota_K\prod_{j=t+1}^T(I-\eta_jC_K)\W_t,
\end{aligned}
\end{equation}
where $\W_t
:=\iota_K^*(\iota_K h_t-h^\dagger)
  +ev_{x_t}^*(y_t-h_t(x_t))
\in\H_K$ and $\be_{z_t\sim\rho}\l[\W_t\r]=\mathbf{0}$. Empty products are understood as identity operators.
\end{lemma}

An equivalent formulation of the prediction error $\E(h)-\E(h^\dagger)$, valid for any estimator $h\in L^2(\X,\rho_{\X};\Y)$, is given by
\begin{equation} \label{temp2}
    \begin{aligned}
        \E(h)-\E(h^\dagger)&=\be_{(x,y)\sim\rho}\l[\l\|h(x)-y\r\|_{\Y}^2\r]-\be_{(x,y)\sim\rho}\l[\l\|\be_{y\sim\rho(y|x)}[y]-y\r\|_{\Y}^2\r] \\
        &=\be_{x\sim\rho_{\X}}\l[\l\|h(x)-\be_{y\sim\rho(y|x)}[y]\r\|_{\Y}^2\r] \\
        &=\l\|h-h^\dagger\r\|_{\rho_{\X}}^2.
    \end{aligned}
\end{equation}
Under Assumption~\ref{a3}, we have $h^\dagger\in\ker L_K^\perp.$ Moreover, every iterate satisfies
\[
h_t\in\iota_K(\H_K)
=
[\H_K]^1
\subseteq\ker L_K^\perp.
\]
Consequently, $h_t-h^\dagger\in\ker L_K^\perp,$ and the isometric property of $L_K^{1/2}:
\ker L_K^\perp\to\iota_K(\H_K)$ gives
\[
\mathcal E(h_t)-\mathcal E(h^\dagger)
=
\|h_t-h^\dagger\|_{\rho_{\X}}^2
=
\left\|
L_K^{1/2}(h_t-h^\dagger)
\right\|_K^2.
\]

We next present a proposition for error decomposition.
\begin{proposition}[Error Decomposition] \label{error decomp}
    Let $T\geq1$ and $0\leq\alpha\leq1/2$. Assume that
    $h^\dagger\in[\H_K]^{1-2\alpha}.$
Write $P_{t+1:T}:=\prod_{j=t+1}^T(I-\eta_jC_K).$ Define
    \begin{equation*}
        \begin{aligned}
            \T_1(\alpha)&:=\l\|L_K^{\alpha}\prod_{t=1}^T(I-\eta_tL_K)h^\dagger\r\|_K^2.
        \end{aligned}
    \end{equation*}  
    \begin{itemize}
        \item[(1)]
Under Assumption \ref{a0}, define
\[
\xi(x,y):=ev_x^*(y-h^\dagger(x)),
\qquad
Q:=\be_{(x,y)\sim\rho}
[\xi(x,y)\otimes\xi(x,y)].
\]
Then $Q$ is a nonnegative trace-class operator satisfying
$\operatorname{Tr}(Q)\leq\kappa^2\sigma^2$.
Write
\[
R_t
:=\be_{z^{t-1}}
\|\iota_Kh_t-h^\dagger\|_{\rho_{\X}}^2
=\be_{z^{t-1}}[\E(h_t)-\E(h^\dagger)],
\]
and define
\[
\begin{aligned}
\T_2(\alpha)
:={}&
\operatorname{Tr}\l(
Q\sum_{t=1}^T\eta_t^2
C_K^{2\alpha}P_{t+1:T}^2
\r)+\kappa^2\sum_{t=1}^T\eta_t^2
\|C_K^{2\alpha}P_{t+1:T}^2\|R_t.
\end{aligned}
\]
Then we have
\[
\be_{z^T}
\l\|L_K^\alpha
(\iota_Kh_{T+1}-h^\dagger)\r\|_K^2
\leq\T_1(\alpha)+\T_2(\alpha).
\]
Moreover,
\begin{equation} \label{eq:T2-coarse}
\T_2(\alpha)
\leq
\kappa^2\sum_{t=1}^T
(\sigma^2+R_t)\eta_t^2
\left\|
L_K^{2\alpha}
\prod_{j=t+1}^T(I-\eta_jL_K)^2
\right\|.
\end{equation}
        \item[(2)]
Under Assumptions \ref{a1} and \ref{a2}, suppose additionally
that $h^\dagger\in\H_K$ and
$0<\eta_t\leq\kappa^{-2}$ for $1\leq t\leq T$.
Define $B_2:=\kappa^2\|h^\dagger\|_K^2+2\sigma^2$ and
\[
\T_3(\alpha)
:=B_2\sum_{t=1}^T\eta_t^2
\operatorname{Tr}\l(
L_K^{1+2\alpha}
\prod_{j=t+1}^T(I-\eta_jL_K)^2
\r).
\]
Then
\[
\be_{z^T}
\l\|L_K^\alpha
\l(\iota_Kh_{T+1}-h^\dagger\r)\r\|_K^2
\leq\T_1(\alpha)+\T_3(\alpha).
\]
    \end{itemize}
\end{proposition}

Under Assumptions \ref{a0} and \ref{a3},
taking $\alpha=1/2$ in part (1) gives
\begin{equation} \label{eq:prediction-recursion}
\begin{aligned}
R_{T+1}
\leq
\T_1(1/2)
+\operatorname{Tr}\l(
Q\sum_{t=1}^T\eta_t^2
C_KP_{t+1:T}^2
\r)+\kappa^2\sum_{t=1}^T\eta_t^2
\|C_KP_{t+1:T}^2\|R_t.
\end{aligned}
\end{equation}

The following lemma is adapted from \cite[Lemma 4.3]{shi2024learning}. While the original proof in \cite{shi2024learning} assumes that the operator 
$A$ is compact, this condition can be relaxed without affecting the validity of the argument. Throughout, we use the convention $0^0=1$.
\begin{lemma}[Lemma 4.3, \cite{shi2024learning}] \label{lemma 1}
Suppose $A$ is a self-adjoint positive operator on
a Hilbert space $H$.
Let $l,T\in\bn$ with $1\leq l\leq T$,
and let $\beta\geq0$. If $\eta_t\|A\|<1$ for all $1\leq t\leq T$, then
    \begin{equation*}
        \left\|A^\beta\prod_{j=l}^T(I-\eta_jA)^2\right\|\leq\left(\frac{\beta}{2e}\right)^\beta\left(\sum_{j=l}^T\eta_j\right)^{-\beta},
    \end{equation*}
    and
    \begin{equation*}
        \left\|A^\beta\prod_{j=l}^T(I-\eta_jA)^2\right\|\leq2\frac{(\frac{\beta}{2e})^\beta+\|A\|^\beta}{1+(\sum_{j=l}^T\eta_j)^\beta}.
    \end{equation*}
\end{lemma}

The following corollary follows directly from Lemma \ref{lemma 1}.
\begin{corollary} \label{term1}
    Let $0\leq\alpha\leq\frac{1}{2}$ and $0< s\leq1$. Suppose $\eta_t\|L_K\|<1$ for all $1\leq t\leq T$. Then, the following bounds hold:
    \begin{equation*}
        \l\|L_K\prod_{j=t+1}^T(I-\eta_jL_K)^2\r\|\leq\frac{1/e+2\kappa^2}{1+\sum_{j=t+1}^T\eta_j},
    \end{equation*}
    and
    \begin{equation*}
        \l\|L_K^{1+2\alpha-s}\prod_{j=t+1}^T(I-\eta_jL_K)^2\r\|
        \leq2\frac{(\frac{1+2\alpha-s}{2e})^{1+2\alpha-s}+\kappa^{2(1+2\alpha-s)}}{1+(\sum_{j=t+1}^T\eta_j)^{1+2\alpha-s}}.
    \end{equation*}
\end{corollary}

We now proceed to bound the term $\T_1(\alpha)$ in Proposition  \ref{error decomp}.
\begin{proposition} \label{term1'}
    Let $0\leq\alpha\leq\frac{1}{2}$ and $T\geq1$. Suppose that Assumption \ref{a3} holds with $r\geq\frac{1}{2}-\alpha$ and $g^\dagger\in L^2(\X, \rho_{\X};\Y)$. Then the following estimates for $\T_1(\alpha)$ hold:
    \begin{itemize}
        \item[(1)] If the step sizes are chosen as $\left\{\eta_t=\eta_1t^{-\theta}\right\}_{t\geq1}$ with $0<\eta_1<\|L_K\|^{-1}$ and $0<\theta<1$, then 
        \begin{equation*}
            \T_1(\alpha)\leq\left(\frac{2\alpha+2r-1}{e}\right)^{2\alpha+2r-1}\l\|g^\dagger\r\|_{\rho_{\X}}^2\eta_1^{-(2\alpha+2r-1)}(T+1)^{-(2\alpha+2r-1)(1-\theta)}.
        \end{equation*}
        \item[(2)] If constant step sizes $\{\eta_t=\eta_1\}_{t\in\bn_T}$ are used with $0<\eta_1<\|L_K\|^{-1}$, then
        \begin{equation*}
            \T_1(\alpha)\leq\left(\frac{2\alpha+2r-1}{2e}\right)^{2\alpha+2r-1}\l\|g^\dagger\r\|_{\rho_{\X}}^2\eta_1^{-(2\alpha+2r-1)}T^{-(2\alpha+2r-1)}.
        \end{equation*}
    \end{itemize}
\end{proposition}

The following two propositions were previously established in \cite{shi2024learning}.
\begin{proposition}[Proposition 4.5, \cite{shi2024learning}] \label{cite1}
		Let $v>0$, $T\geq2$, and consider the step sizes $\left\{\eta_t=\eta_{1}t^{-\theta}\right\}_{t\in \mathbb{N}_T}$ with $\eta_{1}>0$ and $0<\theta<1$. Then:
		\begin{enumerate}
			\item[(1)] Case $0<v<1$:
			\[
			\sum_{t=1}^{T-1}\frac{\eta_{t}^2}{1+\left(\sum_{j=t+1}^{T}\eta_{j}\right)^v}
			\leq \delta\frac{\eta_{1}^2}{\min\{1,(\frac{\eta_{1}}{1-\theta})^v\}}\begin{cases}
				(T+1)^{1-v-\theta(2-v)}, & \text{ if } 0<\theta<\frac{1}{2}, \\
				(T+1)^{-v/2}\log (T+1), & \text{ if }\theta=\frac{1}{2}, \\
				(T+1)^{-v(1-\theta)}, &\text{ if } \frac{1}{2}<\theta<1.
			\end{cases}
			\]
			\item[(2)] Case $v=1$:
			\[
			\sum_{t=1}^{T-1}\frac{\eta_{t}^2}{1+\left(\sum_{j=t+1}^{T}\eta_{j}\right)^v}
			\leq \delta\frac{\eta_{1}^2}{\min\{1,(\frac{\eta_{1}}{1-\theta})^v\}}\begin{cases}
				(T+1)^{-\theta}\log (T+1) , & \text{ if } 0<\theta\leq\frac{1}{2}, \\
				(T+1)^{-(1-\theta)} , & \text{ if } \frac{1}{2}<\theta<1.
			\end{cases}
			\]
			\item[(3)] Case $v>1$:
			\[
			\sum_{t=1}^{T-1}\frac{\eta_{t}^2}{1+\left(\sum_{j=t+1}^{T}\eta_{j}\right)^v}
			\leq \delta\frac{\eta_{1}^2}{\min\{1,(\frac{\eta_{1}}{1-\theta})^v\}}(T+1)^{-\min\{\theta,v(1-\theta)\}}.
			\]
		\end{enumerate}
Here, the constant $\delta$ is independent of both $T$ and $\eta_1$.	
	\end{proposition}

\begin{proposition}[Proposition 4.7, \cite{shi2024learning}] \label{cite}
    Let $v>0$, $T\geq2$, and consider the step sizes $\{\eta_t=\eta_1\}_{t\in\bn_T}$ with $\eta_1>0$. Then,
    \begin{equation*}
        \sum_{t=1}^{T-1}\frac{\eta_t^2}{1+\left(\sum_{j=t+1}^T\eta_j\right)^v}\leq \delta'
        \begin{cases}
        \eta_1^{2-v}(T+1)^{1-v},&\text{ if }0<v<1,\\
        \eta_1\log\left(1+\eta_1(T+1)\right),
&\text{ if }v=1,\\
        \eta_1,&\text{ if }v>1,
        \end{cases}
    \end{equation*}
    where the constant $\delta'$ is given by
    \begin{equation*}
        \delta':=
        \begin{cases}1/(1-v),&\text{ if }0<v<1,\\
        1,&\text{ if }v=1,\\
        v/(v-1),&\text{ if }v>1.
        \end{cases}
    \end{equation*}
\end{proposition}
The same bound holds for $T=1$, since the sum is empty.

The next two propositions bound $\T_2(\alpha)$ and $\T_3(\alpha)$.
\begin{proposition} \label{term2}
Let $0\leq\alpha\leq1/2$ and $T\geq1$.
Suppose that Assumption \ref{a0} holds and
Assumption \ref{a3} holds with $r\geq1/2-\alpha$ and
$g^\dagger\in L^2(\X,\rho_{\X};\Y)$.
Then the following statements hold:
    \begin{itemize}
        \item[(1)] Choose step sizes $\left\{\eta_t=\eta_1t^{-\theta}\right\}_{t\geq1}$ with $0<\eta_1<\min\l\{\|L_K\|^{-1},1-\theta\r\}$ and $0<\theta<1$. Suppose that there exists some constant $M_1>0$ such that
        \begin{equation} \label{(1)}
        \be_{z^{t-1}}\l[\E(h_t)-\E(h^\dagger)\r]
\leq M_1,
\quad \forall t\in\bn_T.
        \end{equation} 
        Then 
        \begin{equation*}
            \T_2\l(\frac{1}{2}\r)\leq2\kappa^2\l(1+2\kappa^2\r)\l(\sigma^2+M_1\r)\l(\delta+3\r)\eta_{1}\begin{cases}
				(T+1)^{-\theta}\log (T+1) , & \text{ if } 0<\theta\leq\frac{1}{2}, \\
				(T+1)^{-(1-\theta)} , & \text{ if } \frac{1}{2}<\theta<1.
			\end{cases}
        \end{equation*}
        \item[(2)] Choose constant step sizes $\{\eta_t=\eta_1=\eta T^{-\theta'}\}_{t\in\bn_T}$ with $0<\eta<\min\l\{\|L_K\|^{-1},1\r\}$ and $0<\theta'<1$. Suppose that there exists some constant $M_1'>0$ such that 
        \begin{equation} \label{(2)}
        \be_{z^{t-1}}\l[\E(h_t)-\E(h^\dagger)\r]
\leq M_1',
\quad \forall t\in\bn_T.
        \end{equation}
        Then, when $\frac{1}{2}\leq\theta'<1$,
        \begin{equation*}
            \T_2(0)\leq4\kappa^2\l(\sigma^2+M_1'\r)\eta^2 (T+1)^{1-2\theta'};
        \end{equation*}
        when $0<\theta'<1$,
        \begin{equation*}
            \T_2\l(\frac{1}{2}\r)\leq4\kappa^2\l(1+2\kappa^2\r)\l(\sigma^2+M_1'\r)\eta\l(2\eta+3\r)(T+1)^{-\theta'}\log(T+1).
        \end{equation*}
    \end{itemize}
\end{proposition}

\begin{proposition} \label{term2'}
Let $0\leq\alpha\leq1/2$ and $T\geq1$.
Suppose that Assumption \ref{a3} holds with $r\geq1/2$ and
$g^\dagger\in L^2(\X,\rho_{\X};\Y)$, Assumption \ref{a1} holds,
and Assumption \ref{a2} holds with $0<s\leq1$.
Let $B_2$ be as defined in Proposition \ref{error decomp}.
Then the following statements hold:
    \begin{itemize}
        \item[(1)] Choose step sizes
$\{\eta_t=\eta_1t^{-\theta}\}_{t\geq1}$ with
$0<\eta_1<\min\{\kappa^{-2},1-\theta\}$ and
$0<\theta<1$.
        Then, when $0< s<1$,
        \begin{equation} \label{(5)}
            \begin{aligned}
                \T_3(0)\leq&4B_2\l(1+\kappa^{2(1-s)}\r)\mathrm{Tr}(L_K^s)(\delta+3) \\
                &\times\eta_1^{1+s}
               \begin{cases}
				(T+1)^{s-\theta(1+s)}, & \text{ if } 0<\theta<\frac{1}{2}, \\
				(T+1)^{-(1-s)/2}\log (T+1), & \text{ if }\theta=\frac{1}{2}, \\
				(T+1)^{-(1-s)(1-\theta)}, &\text{ if } \frac{1}{2}<\theta<1;
			\end{cases}
            \end{aligned}
        \end{equation}
        when $0< s<1$,
        \begin{equation}\label{(6)}
\begin{aligned}
\T_3\l(\frac12\r)
\leq{}&
4B_2\l(1+\kappa^{2(2-s)}\r)
\operatorname{Tr}(L_K^s)(\delta+3)\\
&\times\eta_1^s
(T+1)^{-\min\{\theta,(2-s)(1-\theta)\}}.
\end{aligned}
\end{equation}
When $s=1$, there exists a constant $C_{\theta,\kappa}>0$, independent of $T$ and $\eta_1$, such that
\begin{equation}\label{eq:T3-online-trace}
\T_3\l(\frac12\r)
\leq C_{\theta,\kappa}B_2\operatorname{Tr}(L_K)\eta_1V_\theta(T),
\end{equation}
where $V_\theta$ is defined in \eqref{eq:online-variance-scale}.
        \item[(2)] Choose constant step sizes
$\{\eta_t=\eta_1=\eta T^{-\theta'}\}_{t\in\bn_T}$
with $0<\eta<\min\{\kappa^{-2},1\}$ and
$0<\theta'<1$.
        Then, when $0< s\leq1$,
        \begin{equation} \label{(7)}
        \begin{aligned}
            \T_3(0)\leq16B_2\l(1+\kappa^{2(1-s)}\r)\mathrm{Tr}(L_K^s)\frac{1}{s}\eta^{1+s}(T+1)^{s-\theta'(1+s)},
        \end{aligned}
    \end{equation}
        and 
        \begin{equation}\label{(8)}
\T_3\l(\frac12\r)
\leq B_2\operatorname{Tr}(L_K)\eta T^{-\theta'}.
\end{equation}
    \end{itemize}
\end{proposition}

Proposition \ref{M_1} verifies the bounds
\eqref{(1)} and \eqref{(2)} required in
Proposition \ref{term2}.
Proposition \ref{M_2} establishes the operator
second-moment bounds used in part (2) of
Proposition \ref{error decomp}.

\begin{proposition} \label{M_1}
    Suppose Assumption \ref{a0} holds with $\sigma^2>0$ and Assumption \ref{a3} holds with $r>0$ and $g^\dagger\in L^2(\X, \rho_{\X};\Y)$.
    \begin{itemize}
        \item[(1)] Choose step sizes  $\left\{\eta_t=\eta_1t^{-\theta}\right\}_{t\geq1}$ with $0<\eta_1<\min\l\{\|L_K\|^{-1},1-\theta\r\}$ and $0<\theta<1$. When
        \[
        \eta_1<\frac{\theta}{4\kappa^2\l(1+2\kappa^2\r)(\delta+1)},
        \]
        define
        \[
        M_1:=2\|h^\dagger\|_{\rho_{\X}}^2+4\kappa^2\l(1+2\kappa^2\r)\sigma^2\frac{\delta+1}{\theta}\eta_1.
        \]
        Then, 
        \begin{equation} \label{(9)}
            \be_{z^t}\l[\E(h_{t+1})-\E(h^\dagger)\r]
        \leq M_1,\quad\forall t\geq0.
        \end{equation}
        \item[(2)] Let $T\geq 1$. Choose step sizes $\{\eta_t=\eta_1= \eta T^{-\theta'}\}_{t\in\bn_T}$ with $0<\eta<\min\l\{\|L_K\|^{-1},1\r\}$ and $0<\theta'<1$. When
        \[
        \eta<\frac{\theta'}{4\kappa^2\l(1+2\kappa^2\r)\l(1+2\theta'\r)},
        \]
        define
        \[
        M_1':=2\|h^\dagger\|_{\rho_{\X}}^2+4\kappa^2\l(1+2\kappa^2\r)\sigma^2\l(2+\frac{1}{\theta'}\r)\eta.
        \]
        Then, 
        \begin{equation} \label{(10)}
            \be_{z^t}\l[\E(h_{t+1})-\E(h^\dagger)\r]
        \leq M_1',\quad t=0,\ldots,T.
        \end{equation}
    \end{itemize}
\end{proposition}

\begin{proposition} \label{M_2}
Suppose that Assumption \ref{a1} holds and
$h^\dagger\in\H_K$. Let $T\geq1$ and assume that
$0<\eta_t\leq\kappa^{-2}$ for $1\leq t\leq T$.
Define
\[
M:=\|h^\dagger\|_K^2+\frac{\sigma^2}{\kappa^2},
\qquad
B_2:=\kappa^2M+\sigma^2,
\]
and
\[
\Sigma_t
:=\be_{z^{t-1}}
\l[(h_t-h^\dagger)\otimes(h_t-h^\dagger)\r],
\qquad
\Gamma_t:=\be_{z^t}\l[\W_t\otimes\W_t\r],
\]
where $\W_t$ is defined in Lemma \ref{lemma:representation}.
For self-adjoint operators $A$ and $B$, write
$A\preceq B$ if $B-A$ is nonnegative.
Then
\begin{equation} \label{(11)}
\Sigma_t\preceq MI,
\qquad 1\leq t\leq T+1,
\end{equation}
and
\begin{equation} \label{(12)}
\Gamma_t\preceq B_2C_K,
\qquad 1\leq t\leq T.
\end{equation}
\end{proposition}

The next two propositions sharpen the prediction error
bounds under Assumptions \ref{a0} and \ref{a3}.

\begin{proposition}\label{prop:finite-prediction}
Suppose Assumptions \ref{a0} and \ref{a3} hold with $r>0$.
Let $T\geq1$ and choose constant step sizes
$\eta_t=\eta T^{-\theta'}$, where
\[
0<\theta'<1,
\qquad
0<\eta<\min\{\|L_K\|^{-1},1,\gamma_1'\}.
\]
Then there exists a constant $c>0$, independent of $T$ and
$\eta$, such that
\begin{equation}\label{eq:finite-prediction-bound}
\be_{z^t}\l[\E(h_{t+1})-\E(h^\dagger)\r]
\leq c\l[(1+\eta T^{-\theta'}t)^{-2r}+\eta T^{-\theta'}\r],
\qquad t=0,\ldots,T.
\end{equation}
\end{proposition}

\begin{proposition}\label{prop:online-prediction}
Suppose Assumptions \ref{a0} and \ref{a3} hold with $r>0$.
Choose step sizes $\eta_t=\eta_1t^{-\theta}$, where
\[
0<\theta<1,
\qquad
0<\eta_1<
\min\{\|L_K\|^{-1},1-\theta,\gamma_1\}.
\]
For integers $t\geq1$, define
\begin{equation}\label{eq:online-variance-scale}
V_\theta(t):=
\begin{cases}
(t+1)^{-\theta},
&\text{if }0<\theta<\frac12,\\
(t+1)^{-1/2}\log(t+1),
&\text{if }\theta=\frac12,\\
(t+1)^{-(1-\theta)},
&\text{if }\frac12<\theta<1.
\end{cases}
\end{equation}
Then there exists a constant $c>0$, independent of $T$
and $\eta_1$, such that
\begin{equation}\label{eq:online-prediction-bound}
\be_{z^T}\l[\E(h_{T+1})-\E(h^\dagger)\r]
\leq
c\l[
\eta_1^{-2r}(T+1)^{-2r(1-\theta)}
+\eta_1V_\theta(T)
\r],
\qquad T\geq1.
\end{equation}
\end{proposition}

We now combine these estimates to prove the main theorems
and Corollary \ref{cor:finite-interpolation}.
\begin{proof}[Proof of Theorem \ref{Thm1}]
We first consider the step sizes
$\eta_t=\eta_1t^{-\theta}$.
By Proposition \ref{prop:online-prediction},
\[
\be_{z^T}\l[\E(h_{T+1})-\E(h^\dagger)\r]
\leq
c\l[
\eta_1^{-2r}(T+1)^{-2r(1-\theta)}
+\eta_1V_\theta(T)
\r].
\]
Since $\eta_1<1$, we have $\eta_1\leq\eta_1^{-2r}$.
For $\theta\neq1/2$, comparing the powers of $T+1$
gives the stated bounds.

When $\theta=1/2$ and $0<r<1/2$, the inequality
\[
(T+1)^{-1/2}\log(T+1)
\leq
\frac{1}{e(1/2-r)}(T+1)^{-r}
\]
shows that the variance term is also bounded by a constant
times $\eta_1^{-2r}(T+1)^{-r}$.
When $\theta=1/2$ and $r\geq1/2$, both terms are bounded
by a constant times
$\eta_1^{-2r}(T+1)^{-1/2}\log(T+1)$.
This proves the online prediction bound.

    Next, consider constant step sizes
$\eta_t=\eta T^{-\theta'}$.
Applying Proposition \ref{prop:finite-prediction} with
$t=T$ gives the stated prediction error bound. For the estimation error, take $\alpha=0$ in
Proposition \ref{error decomp} and use
Propositions \ref{term1'}, \ref{term2}, and \ref{M_1}.
When $r>\frac12$ and $\frac12<\theta'<1$, we obtain
    \begin{equation*}
        \begin{aligned}
            \be_{z^T}\l\|h_{T+1}-h^\dagger\r\|_K^2&\leq\left(\frac{2r-1}{2e}\right)^{2r-1}\l\|g^\dagger\r\|_{\rho_{\X}}^2\eta^{-(2r-1)}T^{-(2r-1)(1-\theta')}+4\kappa^2\l(\sigma^2+M_1'\r)\eta^2 (T+1)^{1-2\theta'} \\
            &\leq c_1'\eta^{-(2r-1)}
            \begin{cases}
                (T+1)^{1-2\theta'}, & \text{ if }\frac12<\theta'\leq\frac{2r}{2r+1}, \\
                (T+1)^{-(2r-1)(1-\theta')} & \text{ if }\frac{2r}{2r+1}<\theta'<1.
            \end{cases}
            \end{aligned}
    \end{equation*}
    The proof is finished.
\end{proof}

\begin{proof}[Proof of Theorem \ref{thm3}]
    We use the notation $a \lesssim b$ to indicate that $a \leq Cb$ for some constant $C$ independent of $T$, $\eta$, and $\eta_1$. Since $h_{T+1}\in[\H_K]^1$ and $0<\beta<1$, the continuous
embedding gives $h_{T+1}\in[\H_K]^\beta.$ Moreover, Assumption~\ref{a3} and $r>\beta/2$ imply $h^\dagger\in[\H_K]^{2r}
\hookrightarrow[\H_K]^\beta.$ Therefore,
\[
h_{T+1}-h^\dagger\in[\H_K]^\beta.
\]
By Theorem \ref{thm:interpolation}, we have
    \[
    \l\|L_K^{\frac{1-\beta}{2}}\l(h_{T+1}-h^\dagger\r)\r\|_{K}^2\lesssim\l\|h_{T+1}-h^\dagger\r\|_{\beta,2}^2\lesssim\l\|L_K^{\frac{1-\beta}{2}}\l(h_{T+1}-h^\dagger\r)\r\|_{K}^2.
    \]

    Consider the polynomially decaying step sizes $\left\{\eta_t=\eta_1t^{-\theta}\right\}_{t\geq1}$.
In the error decomposition of Proposition \ref{error decomp}, set $\alpha = \tfrac{1-\beta}{2}$.
We bound $\mathcal{T}_1\left(\tfrac{1-\beta}{2}\right)$
using Proposition \ref{term1'}, and
$\mathcal{T}_2\left(\tfrac{1-\beta}{2}\right)$
using \eqref{eq:T2-coarse}, Lemma \ref{lemma 1},
Proposition \ref{cite1} with $v=1-\beta$,
and Proposition \ref{M_1}.
Consequently, we obtain
    \begin{equation*}
    \begin{aligned}
        \be_{z^T}\l[\l\|h_{T+1}-h^\dagger\r\|_{\beta,2}^2\r]\lesssim&\eta_1^{-(2r-\beta)}(T+1)^{-(2r-\beta)(1-\theta)}+ \eta_{1}^{1+\beta}\begin{cases}
				(T+1)^{\beta-\theta(1+\beta)}, & \text{ if } 0<\theta<\frac{1}{2}, \\
                (T+1)^{-\frac{1-\beta}{2}}\log (T+1)& \text{ if } \theta=\frac{1}{2}, \\
				(T+1)^{-(1-\beta)(1-\theta)} , & \text{ if } \frac{1}{2}<\theta<1.
			\end{cases}
        \\\lesssim&\eta_1^{-(2r-\beta)}
            \begin{cases}
                (T+1)^{\beta-\theta(1+\beta)}f_1(T), &\text{ if } \frac{\beta}{1+\beta}<\theta\leq\min\l\{\frac{2r}{2r+1},\frac{1}{2}\r\}, \\
                (T+1)^{-\min\l\{2r-\beta,1-\beta\r\}(1-\theta)}, &\text{ if } \min\l\{\frac{2r}{2r+1},\frac{1}{2}\r\}<\theta<1,
            \end{cases}
    \end{aligned}
    \end{equation*}
    where 
    \begin{equation*}
        f_1(T):=
        \begin{cases}
            \log(T+1), &\text{ if } \theta=\frac{1}{2}, \\
            1, &\text{ otherwise}.
        \end{cases}
    \end{equation*}

    Now consider constant step sizes
$\{\eta_t=\eta_1\}_{t\in\bn_T}$ with
$\eta_1=\eta T^{-\theta'}$.
We apply Proposition \ref{cite} with $v=1-\beta$,
together with \eqref{eq:T2-coarse},
Proposition \ref{error decomp} with
$\alpha=\tfrac{1-\beta}{2}$,
Proposition \ref{term1'}, Lemma \ref{lemma 1},
and Proposition \ref{M_1}, to obtain
    \begin{equation*}
    \begin{aligned}
        \be_{z^T}\l[\l\|h_{T+1}-h^\dagger\r\|_{\beta,2}^2\r]\lesssim&\eta^{-(2r-\beta)}T^{-(2r-\beta)(1-\theta')}+ \eta^{1+\beta}T^{\beta-\theta'(1+\beta)}
        \\\lesssim&\eta^{-(2r-\beta)}
            \begin{cases}
                T^{\beta-\theta'(1+\beta)}, &\text{ if } \frac{\beta}{1+\beta}<\theta'\leq\frac{2r}{2r+1}, \\
                T^{-(2r-\beta)(1-\theta')}, &\text{ if } \frac{2r}{2r+1}<\theta'<1.
            \end{cases}
    \end{aligned}
    \end{equation*}

    Thus we complete the proof.
\end{proof}

\begin{proof}[Proof of Theorem \ref{Thm2}]
Since $r\geq1/2$, Assumption \ref{a3} gives $h^\dagger\in\H_K$.

Consider the step sizes $\eta_t=\eta_1t^{-\theta}$ with $0<\theta<1$. Proposition \ref{error decomp} with $\alpha=1/2$ and Proposition \ref{term1'} give
\[
\begin{aligned}
\be_{z^T}\l[\E(h_{T+1})-\E(h^\dagger)\r]
\leq{}&
\left(\frac{2r}{e}\right)^{2r}
\|g^\dagger\|_{\rho_{\X}}^2
\eta_1^{-2r}(T+1)^{-2r(1-\theta)}
+\T_3\l(\frac12\r).
\end{aligned}
\]
For $0<s<1$, applying \eqref{(6)} yields
\[
\begin{aligned}
\be_{z^T}\l[\E(h_{T+1})-\E(h^\dagger)\r]
\leq C\l[
\eta_1^{-2r}(T+1)^{-2r(1-\theta)}
+\eta_1^s(T+1)^{-\min\{\theta,(2-s)(1-\theta)\}}
\r],
\end{aligned}
\]
where $C$ is independent of $T$ and $\eta_1$. Since $\eta_1<1$, we have $\eta_1^s\leq\eta_1^{-2r}$. Comparing the powers of $T+1$ gives the stated prediction bound.

For $s=1$, applying \eqref{eq:T3-online-trace} gives
\[
\be_{z^T}\l[\E(h_{T+1})-\E(h^\dagger)\r]
\leq C\l[
\eta_1^{-2r}(T+1)^{-2r(1-\theta)}
+\eta_1V_\theta(T)
\r].
\]
Since $2r\geq1$, the term $\eta_1^{-2r}(T+1)^{-2r(1-\theta)}$ is bounded by a constant times $\eta_1^{-2r}V_\theta(T)$ in each of the three ranges of $\theta$. Using $\eta_1\leq\eta_1^{-2r}$ gives the stated prediction bounds for $s=1$.

    For the estimation error, assume $r>\frac12$, $0<s<1$, and $\frac{s}{1+s}<\theta<1$. Taking $\alpha=0$ in Proposition \ref{error decomp} and combining Proposition \ref{term1'} with \eqref{(5)}, we obtain
    \begin{equation*}
        \begin{aligned}
            \be_{z^T}\l\|h_{T+1}-h^\dagger\r\|_K^2\leq&\left(\frac{2r-1}{e}\right)^{2r-1}\l\|g^\dagger\r\|_{\rho_{\X}}^2\eta_1^{-(2r-1)}(T+1)^{-(2r-1)(1-\theta)}+4B_2 \\
                &\times\l(1+\kappa^{2(1-s)}\r)\mathrm{Tr}(L_K^s)(\delta+3)\eta_1^{1+s}
               \begin{cases}
				(T+1)^{s-\theta(1+s)}, & \text{ if } 0<\theta<\frac{1}{2}, \\
				(T+1)^{-(1-s)/2}\log (T+1), & \text{ if }\theta=\frac{1}{2}, \\
				(T+1)^{-(1-s)(1-\theta)}, &\text{ if } \frac{1}{2}<\theta<1,
			\end{cases}
            \\\leq&c_3\eta_1^{-(2r-1)}
            \begin{cases}
                (T+1)^{s-(1+s)\theta}f_2(T), &\text{ if } \frac{s}{1+s}<\theta\leq\min\l\{\frac{2r+s-1}{2r+s},\frac{1}{2}\r\}, \\
                (T+1)^{-\min\l\{2r-1,1-s\r\}(1-\theta)}, &\text{ if } \min\l\{\frac{2r+s-1}{2r+s},\frac{1}{2}\r\}<\theta<1,
            \end{cases}
        \end{aligned}
    \end{equation*}
    where $f_2(T):=\log(T+1)$ when $\theta=\tfrac{1}{2}$ and 1 otherwise.

    Now consider constant step sizes $\eta_t=\eta_1=\eta T^{-\theta'}$ with $0<\theta'<1$. Proposition \ref{error decomp} with $\alpha=1/2$, Proposition \ref{term1'}, and \eqref{(8)} give
\[
\begin{aligned}
\be_{z^T}\l[\E(h_{T+1})-\E(h^\dagger)\r]
&\leq
\left(\frac{r}{e}\right)^{2r}
\|g^\dagger\|_{\rho_{\X}}^2
\eta^{-2r}T^{-2r(1-\theta')}
+B_2\operatorname{Tr}(L_K)\eta T^{-\theta'}\\
&\leq c_3'\eta^{-2r}
\begin{cases}
(T+1)^{-\theta'},
&\text{if }0<\theta'\leq\frac{2r}{2r+1},\\
(T+1)^{-2r(1-\theta')},
&\text{if }\frac{2r}{2r+1}<\theta'<1.
\end{cases}
\end{aligned}
\]
For the estimation error, assume $r>\frac12$ and $\frac{s}{1+s}<\theta'<1$. Taking $\alpha=0$ in Proposition \ref{error decomp} and combining Proposition \ref{term1'} with \eqref{(7)}, we obtain
    \begin{equation*}
        \begin{aligned}
            \be_{z^T}\l\|h_{T+1}-h^\dagger\r\|_K^2\leq&\left(\frac{2r-1}{2e}\right)^{2r-1}\l\|g^\dagger\r\|_{\rho_{\X}}^2\eta^{-(2r-1)}T^{-(2r-1)(1-\theta')}\\&+16B_2\l(1+\kappa^{2(1-s)}\r)\mathrm{Tr}(L_K^s)\frac{1}{s}\eta^{1+s}(T+1)^{s-\theta'(1+s)} \\
            \leq&c_3'\eta^{-(2r-1)}
            \begin{cases}
                (T+1)^{s-(1+s)\theta'}, &\text{ if } \frac{s}{1+s}<\theta'\leq\frac{2r+s-1}{2r+s}, \\
                (T+1)^{-(2r-1)(1-\theta')} &\text{ if } \frac{2r+s-1}{2r+s}<\theta'<1.
            \end{cases}
        \end{aligned}
    \end{equation*}
    The proof is then finished.
\end{proof}

\begin{proof}[Proof of Corollary \ref{cor:finite-interpolation}]
Recall that $\eta_t=\eta_1=\eta T^{-\theta'}$ and
$\eta_1\kappa^2<1$.
Under the identification of $\H_K$ with its image under
$\iota_K$, both $h_{T+1}$ and $h^\dagger$ belong to
$[\H_K]^1\subset[\H_K]^\beta$.
By the definition of $[\H_K]^\beta$ and the isometric
property of $L_K^{1/2}$, we have
\[
\|u\|_{[\H_K]^\beta}^2
=
\left\|L_K^{(1-\beta)/2}u\right\|_K^2,
\qquad u\in[\H_K]^\beta.
\]
Combining this identity with
Theorem \ref{thm:interpolation} and part (2) of
Proposition \ref{error decomp}, we obtain
\[
\be_{z^T}\l[
\l\|h_{T+1}-h^\dagger\r\|_{\beta,2}^2
\r]
\leq
2d_\beta\left[
\T_1\left(\frac{1-\beta}{2}\right)
+
\T_3\left(\frac{1-\beta}{2}\right)
\right].
\]
Since $2r-\beta>0$, Proposition \ref{term1'} gives
\[
\T_1\left(\frac{1-\beta}{2}\right)
\leq
\left(\frac{2r-\beta}{2e}\right)^{2r-\beta}
\|g^\dagger\|_{\rho_{\X}}^2
(\eta_1T)^{-(2r-\beta)}.
\]

For $\T_3$, geometric summation yields
\[
\begin{aligned}
\T_3\left(\frac{1-\beta}{2}\right)
&=
B_2\eta_1^2\sum_{m=0}^{T-1}
\operatorname{Tr}\left(
L_K^{2-\beta}(I-\eta_1L_K)^{2m}
\right)\\
&=
B_2\eta_1\operatorname{Tr}\left(
L_K^{1-\beta}(2I-\eta_1L_K)^{-1}
\left[I-(I-\eta_1L_K)^{2T}\right]
\right).
\end{aligned}
\]
For $0\leq\lambda\leq\kappa^2$, we have
\[
0\leq
\frac{1-(1-\eta_1\lambda)^{2T}}{2-\eta_1\lambda}
=
\eta_1\lambda\sum_{m=0}^{T-1}
(1-\eta_1\lambda)^{2m}
\leq
\min\{1,\eta_1T\lambda\}.
\]
If $s+\beta\leq1$, use
$\lambda^{1-\beta}
\leq\kappa^{2(1-s-\beta)}\lambda^s$.
If $s+\beta>1$, then $0<s+\beta-1<1$, and
$\min\{1,x\}\leq x^{s+\beta-1}$ for $x\geq0$ gives
$\lambda^{1-\beta}\min\{1,\eta_1T\lambda\}
\leq(\eta_1T)^{s+\beta-1}\lambda^s$.
By functional
calculus, the preceding inequalities imply
\[
\begin{aligned}
0
&\preceq
L_K^{1-\beta}(2I-\eta_1L_K)^{-1}
\left[I-(I-\eta_1L_K)^{2T}\right]\\
&\preceq
\begin{cases}
\kappa^{2(1-s-\beta)}L_K^s,
& s+\beta\leq1,\\[1mm]
(\eta_1T)^{s+\beta-1}L_K^s,
& s+\beta>1.
\end{cases}
\end{aligned}
\]
Since $L_K^s$ is trace class by Assumption \ref{a2},
the nonnegative operator on the first line is also
trace class. Taking traces and using their monotonicity
under the operator order, we obtain
\[
\begin{aligned}
\T_3\left(\frac{1-\beta}{2}\right)
&=
B_2\eta_1\operatorname{Tr}\left(
L_K^{1-\beta}(2I-\eta_1L_K)^{-1}
\left[I-(I-\eta_1L_K)^{2T}\right]
\right)\\
&\leq
B_2\operatorname{Tr}(L_K^s)
\begin{cases}
\kappa^{2(1-s-\beta)}\eta_1,
& s+\beta\leq1,\\[1mm]
\eta_1^{s+\beta}T^{s+\beta-1},
& s+\beta>1.
\end{cases}
\end{aligned}
\]
In particular, the critical case $s+\beta=1$
does not introduce a logarithmic factor.

Combining the preceding estimates, absorbing their
constant factors into $c_\beta$, and substituting
$\eta_1=\eta T^{-\theta'}$ proves
\eqref{eq:finite-interpolation-bound}.
The constant $c_\beta$ is independent of $T$ and $\eta$,
since $B_2=\kappa^2\|h^\dagger\|_K^2+2\sigma^2$
and all other absorbed factors are independent of them.

Finally, the choices in
\eqref{eq:finite-interpolation-step} balance
$(2r-\beta)(1-\theta')$ with $\theta'$ when
$s+\beta\leq1$, and with
$(s+\beta)\theta'-(s+\beta-1)$ when $s+\beta>1$.
Both choices lie in $(0,1)$ because $2r-\beta>0$.
Using $0<\eta<1$, we have
$\eta\leq\eta^{-(2r-\beta)}$ and
$\eta^{s+\beta}\leq\eta^{-(2r-\beta)}$.
Thus \eqref{eq:finite-interpolation-rate} follows.
\end{proof}

\appendix
    
	\section*{Appendix}
 
    \setcounter{equation}{0}
    \setcounter{theorem}{0}
    \addtocounter{section}{0}

    In this Appendix, we complete the proofs omitted in Section \ref{section 2}, \ref{section 3}, and \ref{section 4}.
Appendix~\ref{Appendix1} contains the proof of Theorem~\ref{thm:interpolation}, while Appendix~\ref{Appendix 2} provides the proofs of Remarks~\ref{remark 3} and~\ref{remark 4}.
The proof of Proposition~\ref{Green} is given in Appendix~\ref{Appendix 3}, and Appendix~\ref{Proofs in Section} contains the proofs omitted from Section~\ref{section 4}.
    \section{Proof of Theorem \ref{thm:interpolation}} \label{Appendix1}

\begin{proof}
By the spectral theorem
\cite[Theorem 7.20]{hall2013quantum}, there exist a
$\sigma$-finite measure space $(\Z,\Sigma,\mu)$, an essentially
bounded measurable function $\lambda:\Z\longrightarrow[0,\infty),$ and a unitary operator
\[
U:L^2(\X,\rho_{\X};\Y)\longrightarrow L^2(\Z,\mu)
\]
such that
\[
UL_KU^*=M_\lambda.
\]
Hence, for every $\alpha>0$,
\[
L_K^\alpha=U^*M_{\lambda^\alpha}U.
\]

The spectral representation gives
\[
U\l(\ker L_K^\perp\r)
=
\left\{
\psi\in L^2(\Z,\mu):
\psi=0\ \text{$\mu$-a.e. on }\{\lambda=0\}
\right\}.
\]
Since $[\H_K]^1=\operatorname{ran}L_K^{1/2}$, it follows that
\[
U([\H_K]^1)
=
\left\{
\psi\in L^2(\Z,\mu):
\begin{array}{l}
\psi=0\ \text{$\mu$-a.e. on }\{\lambda=0\},\\[1mm]
\displaystyle
\int_{\{\lambda>0\}}
\lambda(z)^{-1}|\psi(z)|^2\,\mathrm{d}\mu(z)<\infty
\end{array}
\right\},
\]
and, for every $h\in[\H_K]^1$,
\[
\|h\|_{[\H_K]^1}^2
=
\int_{\{\lambda>0\}}
\lambda(z)^{-1}|(Uh)(z)|^2\,\mathrm{d}\mu(z).
\]

Let
\[
K_2(f,t,\mathcal G_1,\mathcal G_2)
:=
\left(
\inf_{\substack{f=f_1+f_2\\ f_i\in\mathcal G_i}}
\left\{
\|f_1\|_{\mathcal G_1}^2
+
t^2\|f_2\|_{\mathcal G_2}^2
\right\}
\right)^{1/2}.
\]
Then
\[
K_2(f,t,\mathcal G_1,\mathcal G_2)
\leq
K(f,t,\mathcal G_1,\mathcal G_2)
\leq
\sqrt{2}\,
K_2(f,t,\mathcal G_1,\mathcal G_2).
\]

For any $f\in L^2(\X,\rho_{\X};\Y)$, minimizing over the second
component of the decomposition gives
\[
\begin{aligned}
&K_2\left(
f,t,L^2(\X,\rho_{\X};\Y),[\H_K]^1
\right)^2\\
&\quad =
\inf_{\psi\in U([\H_K]^1)}
\left\{
\int_{\Z}|(Uf)(z)-\psi(z)|^2\,\mathrm{d}\mu(z)
+
t^2\int_{\{\lambda>0\}}
\lambda(z)^{-1}|\psi(z)|^2\,\mathrm{d}\mu(z)
\right\}.
\end{aligned}
\]
The pointwise minimizer is
\[
\psi_t(z)
=
\frac{\lambda(z)}{\lambda(z)+t^2}(Uf)(z).
\]
It is admissible since it vanishes on $\{\lambda=0\}$ and
\[
\begin{aligned}
\int_{\{\lambda>0\}}
\lambda(z)^{-1}|\psi_t(z)|^2\,\mathrm{d}\mu(z)
&=
\int_{\{\lambda>0\}}
\frac{\lambda(z)}
     {(\lambda(z)+t^2)^2}
|(Uf)(z)|^2\,\mathrm{d}\mu(z)\\
&\leq
\frac{1}{4t^2}
\|f\|_{\rho_{\X}}^2.
\end{aligned}
\]
Therefore,
\[
K_2\left(
f,t,L^2(\X,\rho_{\X};\Y),[\H_K]^1
\right)^2
=
\int_{\Z}
\frac{t^2}{\lambda(z)+t^2}
|(Uf)(z)|^2\,\mathrm{d}\mu(z).
\]

For $0<\beta<1$, Tonelli's theorem yields
\[
\begin{aligned}
&\int_0^\infty
t^{-2\beta}
K_2\left(
f,t,L^2(\X,\rho_{\X};\Y),[\H_K]^1
\right)^2
\frac{\mathrm{d}t}{t}\\
&\quad =
\begin{cases}
\displaystyle
d_\beta
\int_{\{\lambda>0\}}
\lambda(z)^{-\beta}|(Uf)(z)|^2\,\mathrm{d}\mu(z),
&
Uf=0\ \text{$\mu$-a.e. on }\{\lambda=0\},\\[3ex]
+\infty,
&
\text{otherwise},
\end{cases}
\end{aligned}
\]
where $d_\beta
:=
\int_0^\infty
\frac{s^{1-2\beta}}{1+s^2}\,\mathrm{d}s
\in(0,\infty).$

On the other hand, $f\in\operatorname{ran}L_K^{\beta/2}$ if and only if
\[
Uf=0\quad\text{$\mu$-a.e. on }\{\lambda=0\}
\]
and
\[
\int_{\{\lambda>0\}}
\lambda(z)^{-\beta}|(Uf)(z)|^2\,\mathrm{d}\mu(z)<\infty.
\]
Moreover,
\[
\|f\|_{[\H_K]^\beta}^2
=
\int_{\{\lambda>0\}}
\lambda(z)^{-\beta}|(Uf)(z)|^2\,\mathrm{d}\mu(z).
\]

Combining the preceding identities with $K_2\leq K\leq\sqrt{2}\,K_2$ gives
\[
\left[
L^2(\X,\rho_{\X};\Y),[\H_K]^1
\right]_{\beta,2}
=
[\H_K]^\beta
=
\operatorname{ran}L_K^{\beta/2},
\]
and, for every $h\in[\H_K]^\beta$,
\[
d_\beta^{1/2}\|h\|_{[\H_K]^\beta}
\leq
\|h\|_{\beta,2}
\leq
(2d_\beta)^{1/2}\|h\|_{[\H_K]^\beta}.
\]
Equivalently, for every $f\in\ker L_K^\perp$,
\[
d_\beta^{1/2}\|f\|_{\rho_{\X}}
\leq
\left\|L_K^{\beta/2}f\right\|_{\beta,2}
\leq
(2d_\beta)^{1/2}\|f\|_{\rho_{\X}}.
\]

Finally, for $f\in[\H_K]^0=\ker L_K^\perp$, every decomposition
\[
f=f_0+f_1,
\qquad
f_1\in[\H_K]^1,
\]
with $f_0\in L^2(\X,\rho_{\X};\Y)$ necessarily satisfies
$f_0\in[\H_K]^0$. Hence the corresponding $K$-functionals
coincide on $[\H_K]^0$, and therefore
\[
\left[
L^2(\X,\rho_{\X};\Y),[\H_K]^1
\right]_{\beta,2}
=
\left[
[\H_K]^0,[\H_K]^1
\right]_{\beta,2}.
\]
This completes the proof.
\end{proof}

    \section{Proofs of Remarks \ref{remark 3} and \ref{remark 4}}
\label{Appendix 2}

Throughout this appendix, assume $K(x,x')=k(x,x')I$.
Recall that $k(x,x)\leq\kappa^2$,
and hence
$\operatorname{Tr}(L_k)
=\int_{\X}k(x,x)\,\mathrm{d}\rho_{\X}(x)\leq\kappa^2$.
Thus the scalar integral operator $L_k$ is trace class
and admits the nonzero spectral decomposition
\begin{equation}\label{temp:L_k}
L_k
=
\sum_{n\in\mathcal I}
\sigma_n
\langle\cdot,[f_n]\rangle_{L^2(\X,\rho_{\X};\mathbb R)}
[f_n],
\end{equation}
where $\sigma_n>0$,
$\{[f_n]\}_{n\in\mathcal I}$ is an orthonormal basis
of $\ker L_k^\perp$, and $\mathcal I$ is finite
or countable.
Each $[f_n]$ admits an RKHS representative $f_n\in\H_k$.
The scalar RKHS embedding is injective.
Consequently, $\{\sigma_n^{1/2}f_n\}_{n\in\mathcal I}$
is an orthonormal basis of $\H_k$, and
\[
C(\sigma_n^{1/2}f_n)
=
\sigma_n(\sigma_n^{1/2}f_n).
\]
Moreover, for every $y\in\Y$,
$L_K([f_n](\cdot)y)=\sigma_n[f_n](\cdot)y$,
and the functions $[f_n](\cdot)y$, with
$n\in\mathcal I$ and $y\in\Y$, span a dense subspace
of $\ker L_K^\perp$.

\begin{proof}[Proof of Remark \ref{remark 3}]
By definition,
$\{\sigma_n^{\beta/2}[f_n]\}_{n\in\mathcal I}$
is an orthonormal basis of $[\mathcal H]_X^\beta$.
By the definition of the adjoint,
\[
\begin{aligned}
\left\langle
\iota_\beta^*[f_n],\sigma_m^{\beta/2}[f_m]
\right\rangle_{[\mathcal H]_X^\beta}
=
\left\langle
[f_n],\sigma_m^{\beta/2}[f_m]
\right\rangle_{L^2(\rho_{\X})}=\sigma_m^{\beta/2}\delta_{nm}.
\end{aligned}
\]
Since $\{\sigma_m^{\beta/2}[f_m]\}_{m\in\mathcal I}$
is an orthonormal basis of $[\mathcal H]_X^\beta$,
this gives $\iota_\beta^*[f_n]=\sigma_n^\beta[f_n]$.
Moreover, $\iota_\beta^*$ vanishes on $\ker L_k$
since $\operatorname{ran}\iota_\beta\subset\ker L_k^\perp$. 
Therefore,
\[
\Psi(C^*\iota_\beta^*)
=
\sum_{n\in\mathcal I}
\sigma_n^{\beta/2}[f_n](\cdot)
C^*(\sigma_n^{\beta/2}[f_n]).
\]
Define $g^\dagger
:=
\sum_{n\in\mathcal I}
[f_n](\cdot)C^*(\sigma_n^{\beta/2}[f_n]).$ This series converges in $L^2(\X,\rho_{\X};\Y)$,
with $g^\dagger\in\ker L_K^\perp$ and
\[
\|g^\dagger\|_{\rho_{\X}}^2
=
\sum_{n\in\mathcal I}
\|C^*(\sigma_n^{\beta/2}[f_n])\|_{\Y}^2
=
\|C^*\|_{\mathrm{HS}}^2.
\]
The spectral representation of $L_K$ gives
$\Psi(C^*\iota_\beta^*)=L_K^{\beta/2}g^\dagger$.

Conversely, suppose
$h^\dagger=L_K^{\beta/2}g^\dagger$ with
$g^\dagger\in\ker L_K^\perp$.
Write
$g^\dagger=\sum_{n\in\mathcal I}[f_n](\cdot)y_n$,
where $y_n\in\Y$ and
$\sum_{n\in\mathcal I}\|y_n\|_{\Y}^2
=\|g^\dagger\|_{\rho_{\X}}^2$.
Defining
$C^*(\sigma_n^{\beta/2}[f_n]):=y_n$
on the orthonormal basis of $[\mathcal H]_X^\beta$
gives a unique Hilbert--Schmidt operator satisfying
$\Psi(C^*\iota_\beta^*)=h^\dagger$ and
$\|C^*\|_{\mathrm{HS}}=\|g^\dagger\|_{\rho_{\X}}$.
Since $g^\dagger=L_K^{-\beta/2}h^\dagger$,
this proves the equivalence.
\end{proof}

\begin{proof}[Proof of Remark \ref{remark 4}]
For $h^\dagger\in\H_K$, let
$H^\dagger\in S_2(\H_k,\Y)$ be its unique representing
operator and write
$y_n:=H^\dagger(\sigma_n^{1/2}f_n)$.
Then $h^\dagger
=
\sum_{n\in\mathcal I}
\sigma_n^{1/2}[f_n](\cdot)y_n$ in $L^2(\X,\rho_{\X};\Y)$.
Thus we have
\[
h^\dagger\in\operatorname{ran}L_K^{q+1/2}
\quad\Longleftrightarrow\quad
\sum_{n\in\mathcal I}
\sigma_n^{-2q}\|y_n\|_{\Y}^2<\infty.
\]

If $H^\dagger=S^\dagger C^q$, then
$y_n=\sigma_n^qS^\dagger(\sigma_n^{1/2}f_n)$,
so the sum on the right equals
$\|S^\dagger\|_{\mathrm{HS}}^2$ and is finite.
Conversely, if this sum is finite, defining
\[
S^\dagger(\sigma_n^{1/2}f_n)
:=
\sigma_n^{-q}y_n,
\qquad n\in\mathcal I,
\]
gives a Hilbert--Schmidt operator satisfying
$H^\dagger=S^\dagger C^q$.
Since $C$ is injective, $C^q$ has dense range,
so $S^\dagger$ is uniquely determined.
Finally,
\[
\begin{aligned}
\|L_K^{-(q+1/2)}h^\dagger\|_{\rho_{\X}}^2
=
\sum_{n\in\mathcal I}
\sigma_n^{-2q}\|y_n\|_{\Y}^2=
\|S^\dagger\|_{\mathrm{HS}}^2.
\end{aligned}
\]
This proves the claimed equivalence.
\end{proof}

\section{Proof of Proposition \ref{Green}} \label{Appendix 3}

\begin{proof}
Set
\[
\mathsf H_X:=L^2(D_{\X}),\qquad
\mathsf H_Y:=L^2(D_{\Y}),\qquad
\mathsf H_Z:=L^2(D_{\Y}\times D_{\X}).
\]
The restriction of $L_k^{1/2}$ to $\ker L_k^\perp$ is an
isometric isomorphism onto $\mathcal G_k$ under the range norm, hence
$\mathcal G_k$ is complete. For $f\in\mathsf H_X$, recall that
$B_f:\mathsf H_Y\to\mathsf H_Z$ is defined by $B_fg=g\otimes f$. Fubini's theorem gives
\[
\|B_fg\|_{\mathsf H_Z}=\|g\|_{\mathsf H_Y}\|f\|_{\mathsf H_X},
\qquad
B_f^*G=h_G(f),
\]
and hence $\|B_f\|=\|f\|_{\mathsf H_X}$. If $(e_j)_{j\geq1}$ is an
orthonormal basis of $\mathsf H_X$, then Parseval's identity and
Fubini's theorem yield
\begin{equation}
\sum_{j\geq1}\|h_G(e_j)\|_{\mathsf H_Y}^2
=\|G\|_{\mathsf H_Z}^2,
\label{eq:appC-hs}
\end{equation}
so $G\mapsto h_G$ is an isometric embedding of $\mathsf H_Z$ into the
Hilbert--Schmidt operators from $\mathsf H_X$ to $\mathsf H_Y$.

Suppose that $h_G(f)=0$ for every $f\in\X$. Then $h_G$ vanishes on
$\operatorname{span}(\X)$ and, by continuity and the density assumption,
on all of $\mathsf H_X$. Equation~\eqref{eq:appC-hs} implies $G=0$ in
$\mathsf H_Z$. Thus $G\mapsto h_G|_{\X}$ is injective. Since
$\mathcal G_k$ is complete, the transferred inner product makes
$\mathcal H_K$ a Hilbert space isometrically isomorphic to
$\mathcal G_k$. Moreover, for $G\in\mathcal G_k$,
\begin{equation}
\|h_G(f)\|_{\mathsf H_Y}
\leq \|G\|_{\mathsf H_Z}\|f\|_{\mathsf H_X}
\leq \|L_k\|^{1/2}\|G\|_{\mathcal G_k}\|f\|_{\mathsf H_X}.
\label{eq:appC-eval}
\end{equation}
Consequently, every evaluation map on $\mathcal H_K$ is bounded and
every $h_G$ is continuous on $\X$.

We next identify the reproducing kernel. Let $P$ be the orthogonal
projection of $\mathsf H_Z$ onto $\ker L_k^\perp$. For
$q\in\mathsf H_Z$ and $G\in\mathcal G_k$, the inner product
satisfies
\begin{equation}
\begin{aligned}
\langle G,L_kq\rangle_{\mathcal G_k}
&=\big\langle L_k^{-1/2}G,L_k^{-1/2}L_kq\big\rangle_{\mathsf H_Z}\\
&=\big\langle L_k^{-1/2}G,L_k^{1/2}Pq\big\rangle_{\mathsf H_Z}
 =\langle G,q\rangle_{\mathsf H_Z}.
\end{aligned}
\label{eq:appC-range}
\end{equation}
Here we used $L_k^{1/2}L_k^{-1/2}G=G$ and
$G\in\ker L_k^\perp$. Since
$\operatorname{ran}L_k\subseteq\operatorname{ran}L_k^{1/2}
=\mathcal G_k$, for $f\in\X$, $g\in\mathsf H_Y$, and
$G\in\mathcal G_k$, equation~\eqref{eq:appC-range} gives
\[
\begin{aligned}
\langle h_G,h_{L_kB_fg}\rangle_K
&=\langle G,L_kB_fg\rangle_{\mathcal G_k}
 =\langle G,B_fg\rangle_{\mathsf H_Z}\\
&=\langle B_f^*G,g\rangle_{\mathsf H_Y}
 =\langle h_G(f),g\rangle_{\mathsf H_Y}.
\end{aligned}
\]
Therefore
\begin{equation}
K(\cdot,f)g=h_{L_kB_fg},
\qquad
K(f_1,f_2)=B_{f_1}^*L_kB_{f_2}.
\label{eq:appC-kernel}
\end{equation}
Expanding~\eqref{eq:appC-kernel} gives the integral formula in the
proposition. Self-adjointness of $L_k$ gives Hermitian symmetry.

For positive semidefiniteness, take $f_1,\ldots,f_n\in\X$ and
$g_1,\ldots,g_n\in\mathsf H_Y$. Then
\[
\sum_{i,j=1}^n
\langle K(f_i,f_j)g_j,g_i\rangle_{\mathsf H_Y}
=
\left\langle
L_k\sum_{j=1}^nB_{f_j}g_j,
\sum_{i=1}^nB_{f_i}g_i
\right\rangle_{\mathsf H_Z}\geq0,
\]
so $K$ is positive semidefinite. Thus $\mathcal H_K$ is the
$\mathsf H_Y$-valued RKHS induced by $K$.

Since $\|B_f-B_{\widetilde f}\|=\|f-\widetilde f\|_{\mathsf H_X}$,
\eqref{eq:appC-kernel} implies
\[
\begin{aligned}
&\|K(f_1,f_2)-K(\widetilde f_1,\widetilde f_2)\|\\
&\quad\leq \|L_k\|\big(
\|f_1-\widetilde f_1\|_{\mathsf H_X}\|f_2\|_{\mathsf H_X}
+\|\widetilde f_1\|_{\mathsf H_X}
 \|f_2-\widetilde f_2\|_{\mathsf H_X}\big).
\end{aligned}
\]
Hence $K$ is jointly continuous in operator norm and is Mercer. By the property of Hilbert--Schmidt
operators,
\[
\|K(f_1,f_2)\|_{\mathcal S_2(\mathsf H_Y)}
\leq
\|L_k\|_{\mathcal S_2(\mathsf H_Z)}
\|f_1\|_{\mathsf H_X}\|f_2\|_{\mathsf H_X},
\]
so every $K(f_1,f_2)$ is Hilbert--Schmidt, and hence compact. Similarly,
\[
\|K(f_1,f_2)\|
\leq \|L_k\|\|f_1\|_{\mathsf H_X}\|f_2\|_{\mathsf H_X},
\qquad
\sup_{f\in\X}\|K(f,f)\|\leq R^2\|L_k\|.
\]

It remains to prove the compactness of $L_K$. Let
$\mathsf H_\rho:=L^2(\X,\rho_{\X};\mathsf H_Y)$ and define
\[
\mathsf S\Phi
:=\int_{\X}B_f\Phi(f)\,\mathrm d\rho_{\X}(f)
=\int_{\X}\Phi(f)\otimes f\,\mathrm d\rho_{\X}(f).
\]
The bilinear map $(v,f)\mapsto v\otimes f$ is continuous,
and the spaces involved are separable, so the integrand
is strongly measurable. Moreover,
\[
\int_{\X}\|B_f\Phi(f)\|_{\mathsf H_Z}\,\mathrm d\rho_{\X}(f)
\leq
\|\Phi\|_{\mathsf H_\rho}
\left(\int_{\X}\|f\|_{\mathsf H_X}^2
\,\mathrm d\rho_{\X}(f)\right)^{1/2}.
\]
Thus $\mathsf S$ is well defined and $\|\mathsf S\|\leq R$. For
$q\in\mathsf H_Z$ and $\Phi\in\mathsf H_\rho$,
\[
\langle\mathsf S\Phi,q\rangle_{\mathsf H_Z}
=
\int_{\X}\langle\Phi(f),B_f^*q\rangle_{\mathsf H_Y}
\,\mathrm d\rho_{\X}(f),
\]
whence
\[
(\mathsf S^*q)(f)=B_f^*q
\qquad \rho_{\X}\text{-almost everywhere}.
\]
Using~\eqref{eq:appC-kernel}, for every $\Phi\in\mathsf H_\rho$,
\[
\begin{aligned}
(\mathsf S^*L_k\mathsf S\Phi)(f)
&=B_f^*L_k
  \int_{\X}B_{f'}\Phi(f')\,\mathrm d\rho_{\X}(f')\\
&=\int_{\X}K(f,f')\Phi(f')\,\mathrm d\rho_{\X}(f')
 =(L_K\Phi)(f)
\end{aligned}
\]
for $\rho_{\X}$-almost every $f$. Hence
$L_K=\mathsf S^*L_k\mathsf S$. The ideal property now gives
\[
\|L_K\|_{\mathcal S_2(\mathsf H_\rho)}
\leq \|\mathsf S\|^2\|L_k\|_{\mathcal S_2(\mathsf H_Z)}
\leq R^2\|L_k\|_{\mathcal S_2(\mathsf H_Z)}.
\]
Therefore $L_K$ is Hilbert--Schmidt, and in particular compact. The same
factorization also gives
\[
\operatorname{rank}(L_K)\leq\operatorname{rank}(L_k),
\qquad
\|L_K\|_{\mathcal S_1(\mathsf H_\rho)}
\leq R^2\|L_k\|_{\mathcal S_1(\mathsf H_Z)}
\]
whenever the corresponding right-hand side is finite.
\end{proof}

\section{Proofs in Section \ref{section 4}} \label{Proofs in Section}

\begin{proof}[Proof of Lemma \ref{lemma:representation}]
Applying $\iota_K$ to the update rule~\eqref{algorithm} and using
$L_K=\iota_K\iota_K^*$, we obtain
\begin{align*}
    \iota_K h_{t+1}-h^\dagger
&=\iota_K h_t-h^\dagger-\eta_t \iota_Kev_{x_t}^*(h_t(x_t)-y_t)
\\&=(I-\eta_tL_K)(\iota_K h_t-h^\dagger)
  +\eta_t\iota_K\W_t.
\end{align*}
Iterating this identity from $t=1$ to $t=T$, using $h_1=0$ and $(I-\eta_jL_K)\iota_K
=\iota_K(I-\eta_jC_K),$ yields~\eqref{temp1}. Since $h_t$ depends only on $z^{t-1}$, the definition
of $h^\dagger$ gives
\[
\begin{aligned}
\be_{z_t\sim\rho}
\l[ev_{x_t}^*(y_t-h_t(x_t))\r]
&=
\be_{x_t\sim\rho_{\X}}
\l[ev_{x_t}^*(h^\dagger(x_t)-h_t(x_t))\r]\\
&=\iota_K^*(h^\dagger-\iota_K h_t).
\end{aligned}
\]
Therefore, $\be_{z_t\sim\rho}\l[\W_t\r]=0.$

This concludes the proof.
\end{proof}

\begin{proof}[Proof of Proposition \ref{error decomp}]
    Starting from the decomposition \eqref{temp1}, we use
$L_K^\alpha\iota_K=\iota_K C_K^\alpha$ and the zero-mean
property $\be_{z_t\sim\rho}\l[\W_t\r]=\mathbf{0}$.
More precisely, for $t<t'$ and deterministic bounded operators
$A,B$ on $\H_K$, integration over $z_{t'}$ with
$z^{t'-1}$ fixed gives $\be_{z_{t'}\sim\rho}
\l[\langle A\W_t,B\W_{t'}\rangle_K\r]=0.$ Hence,
\begin{equation} \label{temp3}
    \begin{aligned}
        \be_{z^T}
\left\|L_K^\alpha(\iota_K h_{T+1}-h^\dagger)\right\|_K^2&=\be_{z^T}\left\|-L_K^\alpha\prod_{t=1}^T(I-\eta_tL_K)h^\dagger
+\sum_{t=1}^T\eta_t\,
L_K^\alpha\iota_K\prod_{j=t+1}^T(I-\eta_jC_K)\W_t\right\|_K^2
\\&=\left\|L_K^\alpha
\prod_{t=1}^T(I-\eta_tL_K)h^\dagger\right\|_K^2+
\sum_{t=1}^T\eta_t^2\be_{z^t}
\left\|C_K^\alpha
\prod_{j=t+1}^T(I-\eta_jC_K)\W_t\right\|_K^2.
    \end{aligned}
\end{equation}
Set $D_t:=C_K^\alpha P_{t+1:T}.$

For part (1), Assumption \ref{a0} gives
\[
\operatorname{Tr}(Q)
=\be_{(x,y)\sim\rho}\|\xi(x,y)\|_K^2
\leq\kappa^2\sigma^2.
\]
Write $d_t:=\iota_Kh_t-h^\dagger$, with representative
$d_t(x)=h_t(x)-h^\dagger(x)$.
By the definition of $\W_t$ in Lemma \ref{lemma:representation},
\[
\W_t
=\iota_K^*d_t-ev_{x_t}^*d_t(x_t)+\xi(x_t,y_t).
\]
For fixed $z^{t-1}$, we have
\begin{equation} \label{eq:noise-split-sampling-mean}
\be_{x_t\sim\rho_{\X}}
[ev_{x_t}^*d_t(x_t)]
=\iota_K^*d_t,
\end{equation}
and
\begin{equation} \label{eq:noise-split-noise-mean}
\be_{y_t\sim\rho(y_t|x_t)}
[\xi(x_t,y_t)]=0.
\end{equation}
Since $Q$ is nonnegative and trace class and $D_t$
is bounded and self-adjoint,
\[
\be_{z_t\sim\rho}
\|D_t\xi(x_t,y_t)\|_K^2
=\operatorname{Tr}(D_tQD_t)
=\operatorname{Tr}(QD_t^2).
\]
Expanding the square and using
\eqref{eq:noise-split-sampling-mean} and
\eqref{eq:noise-split-noise-mean}, we obtain
\[
\begin{aligned}
\be_{z_t\sim\rho}\|D_t\W_t\|_K^2
={}&
\operatorname{Tr}(QD_t^2)
+\be_{x\sim\rho_{\X}}
\|D_t(\iota_K^*d_t-ev_x^*d_t(x))\|_K^2\\
={}&
\operatorname{Tr}(QD_t^2)
+\be_{x\sim\rho_{\X}}
\|D_tev_x^*d_t(x)\|_K^2
-\|D_t\iota_K^*d_t\|_K^2\\
\leq{}&
\operatorname{Tr}(QD_t^2)
+\be_{x\sim\rho_{\X}}
\|D_tev_x^*d_t(x)\|_K^2\\
\leq{}&
\operatorname{Tr}(QD_t^2)
+\kappa^2\|D_t\|^2\be_{x\sim\rho_{\X}}
\|d_t(x)\|_{\Y}^2\\
={}&
\operatorname{Tr}(QD_t^2)
+\kappa^2\|D_t\|^2\|d_t\|_{\rho_{\X}}^2.
\end{aligned}
\]
Taking expectation over $z^{t-1}$ yields
\[
\be_{z^t}\|D_t\W_t\|_K^2
\leq
\operatorname{Tr}(QD_t^2)
+\kappa^2\|D_t\|^2R_t.
\]
Since $D_t^2=C_K^{2\alpha}P_{t+1:T}^2$ and $\|D_t\|^2=\|D_t^2\|$, 
substitution into \eqref{temp3} gives the error bound
in part (1).

Moreover,
\[
\operatorname{Tr}(QD_t^2)
\leq\operatorname{Tr}(Q)\|D_t^2\|
\leq\kappa^2\sigma^2\|D_t^2\|.
\]
Combining this estimate with
\[
\|C_K^{2\alpha}P_{t+1:T}^2\|
\leq
\left\|
L_K^{2\alpha}
\prod_{j=t+1}^T(I-\eta_jL_K)^2
\right\|
\]
proves \eqref{eq:T2-coarse}.

For part (2), Proposition \ref{M_2} gives
\[
\Gamma_t:=\be_{z^t}[\W_t\otimes\W_t]
\preceq B_2C_K.
\]
Since $D_t$ is self-adjoint and commutes with $C_K$,
\[
\begin{aligned}
\be_{z^t}\|D_t\W_t\|_K^2
&=\operatorname{Tr}(D_t\Gamma_tD_t)\\
&\leq B_2\operatorname{Tr}(D_tC_KD_t)\\
&=B_2\operatorname{Tr}\l(
C_K^{1+2\alpha}
\prod_{j=t+1}^T(I-\eta_jC_K)^2
\r)\\
&=B_2\operatorname{Tr}\l(
L_K^{1+2\alpha}
\prod_{j=t+1}^T(I-\eta_jL_K)^2
\r).
\end{aligned}
\]
The last equality follows from the unitary equivalence of
$C_K$ and the restriction of $L_K$ to $\ker L_K^\perp$,
together with the fact that $L_K^{1+2\alpha}$ vanishes
on $\ker L_K$.
Substituting this bound into \eqref{temp3} proves part (2).
\end{proof}

\begin{proof}[Proof of Proposition \ref{term1'}]
    By Assumption \ref{a3}, the target function satisfies
    $h^\dagger=L_K^rg^\dagger$ for some $g^\dagger\in L^2(\X, \rho_{\X};\Y)$. Then
    \begin{equation*}
        \begin{aligned}
            \l\|L_K^{\alpha}\prod_{t=1}^T(I-\eta_tL_K)h^\dagger\r\|_K^2&=\l\|L_K^{\alpha+r}\prod_{t=1}^T(I-\eta_tL_K)g^\dagger\r\|_K^2 \\
            &=\l\|L_K^{\alpha+r-\frac{1}{2}}\prod_{t=1}^T(I-\eta_tL_K)g^\dagger\r\|_{\rho_{\X}}^2
            \\&\leq\left\|L_K^{2\alpha+2r-1}\prod_{t=1}^T(I-\eta_tL_K)^2\right\|\l\|g^\dagger\r\|_{\rho_{\X}}^2.
        \end{aligned}
    \end{equation*}
    Applying Lemma \ref{lemma 1} with $A=L_K$, $\beta=2\alpha+2r-1$, and $l=1$, we obtain
    \begin{equation*}
        \begin{aligned}
            \l\|L_K^{\alpha}\prod_{t=1}^T(I-\eta_tL_K)h^\dagger\r\|_K^2\leq\left(\frac{2\alpha+2r-1}{2e}\right)^{2\alpha+2r-1}\left(\sum_{j=1}^T\eta_j\right)^{-(2\alpha+2r-1)}\l\|g^\dagger\r\|_{\rho_{\X}}^2,
        \end{aligned}
    \end{equation*}
    For the case $\eta_t=\eta_1t^{-\theta}$, we estimate
    \begin{equation*}
        \sum_{j=1}^T\eta_j\geq\eta_1\int_{1}^{T+1}t^{-\theta}\mathrm{d}t
        \geq\frac{1-2^{\theta-1}}{1-\theta}\eta_1(T+1)^{1-\theta}\geq\frac{\eta_1}{2}(T+1)^{1-\theta},
    \end{equation*}
    which implies    
    \begin{equation*}
        \begin{aligned}
            \T_1(\alpha)&\leq\left(\frac{2\alpha+2r-1}{2e}\right)^{2\alpha+2r-1}\left(\frac{\eta_1}{2}(T+1)^{1-\theta}\right)^{-(2\alpha+2r-1)}\l\|g^\dagger\r\|_{\rho_{\X}}^2 \\ 
            &=\left(\frac{2\alpha+2r-1}{e}\right)^{2\alpha+2r-1}\l\|g^\dagger\r\|_{\rho_{\X}}^2\eta_1^{-(2\alpha+2r-1)}(T+1)^{-(2\alpha+2r-1)(1-\theta)}.
        \end{aligned}
    \end{equation*}
    For the case of constant step size $\eta_t=\eta_1$,  we directly have $\sum_{j=1}^T\eta_j=T\eta_1$, and thus
    \begin{equation*}
        \T_1(\alpha)\leq\left(\frac{2\alpha+2r-1}{2e}\right)^{2\alpha+2r-1}\l\|g^\dagger\r\|_{\rho_{\X}}^2\eta_1^{-(2\alpha+2r-1)}T^{-(2\alpha+2r-1)}.
    \end{equation*}
    The proof is then finished.
\end{proof}

\begin{proof}[Proof of Proposition \ref{term2}]
We first consider the polynomially decaying step sizes $\eta_t=\eta_1t^{-\theta}$, and assume that \eqref{(1)} holds for all $t\in\bn_T$. Using \eqref{eq:T2-coarse}, Corollary \ref{term1},
and Proposition \ref{cite1} with $v=1$, we estimate
    \begin{equation*}
        \begin{aligned}
            \T_2\l(\frac12\r)&\leq\sum_{t=1}^T2\kappa^2\l(\sigma^2+M_1\r)\eta_t^2\l\|L_K\prod_{j=t+1}^T(I-\eta_jL_K)^2\r\| \\
            &\leq2\kappa^2\l(1/e+2\kappa^2\r)\l(\sigma^2+M_1\r)\sum_{t=1}^T\frac{\eta_t^2}{1+\sum_{j=t+1}^T\eta_j} \\
            &\leq2\kappa^2\l(1+2\kappa^2\r)\l(\sigma^2+M_1\r)\l(\delta+3\r)\eta_{1}\begin{cases}
				(T+1)^{-\theta}\log (T+1) , & \text{ if } 0<\theta\leq\frac{1}{2}, \\
				(T+1)^{-(1-\theta)} , & \text{ if } \frac{1}{2}<\theta<1.
			\end{cases}
        \end{aligned}
    \end{equation*} 
    Next, we turn to the constant step size $\eta_t=\eta_1=\eta T^{-\theta'}$, and assume that \eqref{(2)} holds for all $t\in\bn_T$. By \eqref{eq:T2-coarse}, we have
    \begin{equation*}
        \T_2\l(\alpha\r)\leq\sum_{t=1}^T2\kappa^2\l(\sigma^2+M_1'\r)\eta_t^2\l\|L_K^{2\alpha}\prod_{j=t+1}^T(I-\eta_jL_K)^2\r\|.
    \end{equation*}
    If $\alpha=0$, we obtain 
    \begin{equation*}
        \begin{aligned}
            \T_2(0)\leq2\kappa^2\l(\sigma^2+M_1'\r)T\eta_1^2\leq 4\kappa^2\l(\sigma^2+M_1'\r)\eta^2 (T+1)^{1-2\theta'}.
        \end{aligned}
    \end{equation*}
    If $\alpha=1/2$, applying Corollary \ref{term1} and Proposition \ref{cite} with $v=1$, we derive
    \begin{equation*}
\begin{aligned}
\T_2\l(\frac12\r)
&\leq
2\kappa^2\l(1/e+2\kappa^2\r)
\l(\sigma^2+M_1'\r)
\sum_{t=1}^T
\frac{\eta_t^2}{1+\sum_{j=t+1}^T\eta_j}\\
&\leq
2\kappa^2\l(1/e+2\kappa^2\r)
\l(\sigma^2+M_1'\r)
\eta_1\left[
\eta_1+\log\left(1+\eta_1(T+1)\right)
\right]\\
&\leq
4\kappa^2\l(1+2\kappa^2\r)
\l(\sigma^2+M_1'\r)
\eta(2\eta+3)(T+1)^{-\theta'}\log(T+1).
\end{aligned}
\end{equation*}
    We thus complete the proof.
\end{proof}

\begin{proof}[Proof of Proposition \ref{term2'}]
The case $T=1$ follows directly from the definition of $\T_3(\alpha)$. We therefore assume $T\geq2$.

Set $p:=1+2\alpha-s\in[0,2)$. By the definition of $\T_3(\alpha)$ and Assumption \ref{a2},
\[
\begin{aligned}
\T_3(\alpha)
&\leq B_2\operatorname{Tr}(L_K^s)
\sum_{t=1}^T\eta_t^2
\left\|L_K^p\prod_{j=t+1}^T(I-\eta_jL_K)^2\right\|\\
&\leq
2B_2(1+\kappa^{2p})\operatorname{Tr}(L_K^s)
\sum_{t=1}^T
\frac{\eta_t^2}
{1+\left(\sum_{j=t+1}^T\eta_j\right)^p},
\end{aligned}
\]
where the second inequality follows from Corollary \ref{term1}.

For $\eta_t=\eta_1t^{-\theta}$, apply Proposition \ref{cite1} to the sum over $1\leq t<T$ and bound the $t=T$ summand by $\eta_T^2=\eta_1^2T^{-2\theta}$. Taking $p=1-s$ gives \eqref{(5)} for $0<s<1$, and taking $p=2-s$ gives \eqref{(6)} for $0<s<1$. When $s=1$ and $1/2\leq\theta<1$, the same argument with $p=1$ gives \eqref{eq:T3-online-trace}.

It remains to prove \eqref{eq:T3-online-trace} for $s=1$ and $0<\theta<1/2$. Write
\[
G_{t:T}:=\prod_{j=t}^T(I-\eta_jL_K),
\qquad G_{T+1:T}:=I,
\]
and put $\ell:=\lfloor T/2\rfloor$. For $t\leq\ell$,
\[
\sum_{j=t+1}^T\eta_j
\geq\eta_1(T-\ell)T^{-\theta}
\geq\frac{\eta_1}{2}T^{1-\theta}.
\]
Lemma \ref{lemma 1} therefore gives
\[
\|L_KG_{t+1:T}^2\|
\leq\frac{1}{e\eta_1T^{1-\theta}},
\qquad t\leq\ell.
\]
Consequently, we have
\begin{equation}\label{eq:T3-trace-early}
\begin{aligned}
\sum_{t=1}^{\ell}\eta_t^2
\operatorname{Tr}(L_K^2G_{t+1:T}^2)
&\leq
\operatorname{Tr}(L_K)
\sum_{t=1}^{\ell}\eta_t^2\|L_KG_{t+1:T}^2\|\\
&\leq
\frac{\eta_1\operatorname{Tr}(L_K)}{eT^{1-\theta}}
\sum_{t=1}^{\ell}t^{-2\theta}\\
&\leq
\frac{\eta_1\operatorname{Tr}(L_K)}{e(1-2\theta)}T^{-\theta}\\
&\leq
\frac{2^\theta}{e(1-2\theta)}
\operatorname{Tr}(L_K)\eta_1(T+1)^{-\theta}.
\end{aligned}
\end{equation}

Since $\eta_t\kappa^2<1$, the identity
\[
G_{t+1:T}^2-G_{t:T}^2
=\eta_tL_K(2I-\eta_tL_K)G_{t+1:T}^2
\]
implies $\eta_tL_K^2G_{t+1:T}^2
\preceq L_K(G_{t+1:T}^2-G_{t:T}^2).$ Using $\eta_t\leq\eta_{\ell+1}$ for $t>\ell$, we obtain
\begin{equation}\label{eq:T3-trace-late}
\begin{aligned}
\sum_{t=\ell+1}^T\eta_t^2L_K^2G_{t+1:T}^2
&\preceq
\eta_{\ell+1}L_K
\sum_{t=\ell+1}^T(G_{t+1:T}^2-G_{t:T}^2)\\
&=\eta_{\ell+1}L_K(I-G_{\ell+1:T}^2)\\
&\preceq\eta_{\ell+1}L_K\\
&\preceq2^\theta\eta_1(T+1)^{-\theta}L_K,
\end{aligned}
\end{equation}
where the last inequality uses $\ell+1\geq(T+1)/2$.

Taking traces in \eqref{eq:T3-trace-late} and using \eqref{eq:T3-trace-early}, we obtain
\[
\begin{aligned}
\T_3\l(\frac12\r)
&=B_2\sum_{t=1}^T\eta_t^2
\operatorname{Tr}(L_K^2G_{t+1:T}^2)\\
&\leq
B_2\left(\frac{2^\theta}{e(1-2\theta)}+2^\theta\right)
\operatorname{Tr}(L_K)\eta_1(T+1)^{-\theta}.
\end{aligned}
\]
This proves \eqref{eq:T3-online-trace} for $0<\theta<1/2$.

Next, let $\eta_t=\eta_1=\eta T^{-\theta'}$. For $0<s<1$, Proposition \ref{cite}, together with the last term of the sum, gives \eqref{(7)} with $p=1-s$. If $s=1$, the bound $\|\prod_{j=t+1}^T(I-\eta_jL_K)^2\|\leq1$ yields
\[
\begin{aligned}
\T_3(0)
&\leq B_2\operatorname{Tr}(L_K)\eta_1^2T\\
&=B_2\operatorname{Tr}(L_K)\eta^2T^{1-2\theta'}\\
&\leq2B_2\operatorname{Tr}(L_K)\eta^2(T+1)^{1-2\theta'},
\end{aligned}
\]
which also implies \eqref{(7)}.

Finally, Assumption \ref{a2} gives
\[
\operatorname{Tr}(L_K)
\leq\kappa^{2(1-s)}\operatorname{Tr}(L_K^s)<\infty.
\]
Since $\eta_1\kappa^2<1$, the geometric sum identity gives
\[
\begin{aligned}
\eta_1^2\sum_{m=0}^{T-1}L_K^2(I-\eta_1L_K)^{2m}
&=\eta_1L_K(2I-\eta_1L_K)^{-1}
\l[I-(I-\eta_1L_K)^{2T}\r]\\
&\preceq\eta_1L_K.
\end{aligned}
\]
Therefore,
\[
\begin{aligned}
\T_3\l(\frac12\r)
&=B_2\operatorname{Tr}\left(
\eta_1^2\sum_{m=0}^{T-1}L_K^2(I-\eta_1L_K)^{2m}
\right)\\
&\leq B_2\eta_1\operatorname{Tr}(L_K)\\
&=B_2\operatorname{Tr}(L_K)\eta T^{-\theta'}.
\end{aligned}
\]
This proves \eqref{(8)}.
\end{proof}

\begin{proof}[Proof of Proposition \ref{M_1}]

    We prove inequality \eqref{(9)} by induction. For the base case $t=0$, we have
    \[
    \be_{z^0}\l[\E(h_{1})-\E(h^\dagger)\r]=\|L_K^{1/2}h^\dagger\|_K^2=\|h^\dagger\|_{\rho_{\X}}^2\leq M_1.
    \]
    Assume that inequality \eqref{(9)} holds for all $0 \leq t \leq T-1$. We prove that it also holds for $t = T$.

    Applying Proposition \ref{error decomp} with $\alpha=1/2$, we obtain
    \begin{equation}
        \be_{z^T}\l[\E(h_{T+1})-\E(h^\dagger)\r]\leq\T_1\l(\frac12\r)+\T_2\l(\frac12\r).
    \end{equation}
    It follows from the isometric property of $L_K^{1/2}$ between $\ker L_K^\perp$ and $\H_K$ that
    \begin{equation*}
            \T_1\l(\frac{1}{2}\r)=\l\|L_K^{\frac{1}{2}}\prod_{t=1}^T(I-\eta_tL_K)h^\dagger\r\|_K^2\leq\l\|h^\dagger\r\|_{\rho_{\X}}^2.
    \end{equation*}
    Using the induction hypothesis, \eqref{eq:T2-coarse},
Corollary \ref{term1}, and Proposition \ref{cite1}
with $v=1$, we estimate
    \begin{equation*}
        \begin{aligned}
            \T_2\l(\frac12\r)&\leq\sum_{t=1}^T2\kappa^2\l(\sigma^2+M_1\r)\eta_t^2\l\|L_K\prod_{j=t+1}^T(I-\eta_jL_K)^2\r\| \\
            &\leq2\kappa^2\l(1/e+2\kappa^2\r)\l(\sigma^2+M_1\r)\sum_{t=1}^T\frac{\eta_t^2}{1+\sum_{j=t+1}^T\eta_j} \\
            &\leq2\kappa^2\l(1+2\kappa^2\r)\l(\sigma^2+M_1\r)\frac{\delta+1}{\theta}\eta_1.
        \end{aligned}
    \end{equation*}
    Here we used the inequality $x^{-\theta}\log x\leq 1/(e\theta)$ for any $x>0$ and the fact that $\eta_T^2\leq \eta_1/\theta$. By the choice of $\eta_1$ and $M_1$, inequality \eqref{(9)} holds for $t = T$. This completes the induction.
    
    For constant step sizes, inequality \eqref{(10)} clearly holds at $t=0$. Suppose it holds for $0\leq t<k$, and consider $t=k$. Then, we have
    \begin{equation*}
            \T_1\l(\frac{1}{2}\r)=\l\|L_K^{\frac{1}{2}}\prod_{t=1}^k(I-\eta_tL_K)h^\dagger\r\|_K^2\leq\l\|h^\dagger\r\|_{\rho_{\X}}^2.
    \end{equation*}
    Using the induction hypothesis, \eqref{eq:T2-coarse},
Corollary \ref{term1}, and Proposition \ref{cite}
with $v=1$, we estimate
    \begin{equation*}
        \begin{aligned}
            \T_2\l(\frac12\r)&\leq\sum_{t=1}^k2\kappa^2\l(\sigma^2+M_1'\r)\eta_t^2\l\|L_K\prod_{j=t+1}^k(I-\eta_jL_K)^2\r\| \\
            &\leq2\kappa^2\l(1/e+2\kappa^2\r)\l(\sigma^2+M_1'\r)\sum_{t=1}^k\frac{\eta_t^2}{1+\sum_{j=t+1}^k\eta_j} \\
            &\leq
2\kappa^2\l(1/e+2\kappa^2\r)
\l(\sigma^2+M_1'\r)
\eta_1\left[
\eta_1+\log\left(1+\eta_1(k+1)\right)
\right]\\
&\leq
2\kappa^2\l(1+2\kappa^2\r)
\l(\sigma^2+M_1'\r)
\left(2+\frac1{\theta'}\right)\eta.
        \end{aligned}
    \end{equation*}
    By the choice of $\eta$ and $M_1'$, inequality \eqref{(10)} holds for $t = k$. This completes the induction.
    
    The proof is finished.
\end{proof}

\begin{proof}[Proof of Proposition \ref{M_2}]
Write
\[
e_t:=h_t-h^\dagger,
\qquad
A_x:=ev_x^*ev_x,
\qquad
\xi_t:=ev_{x_t}^*(y_t-h^\dagger(x_t)).
\]
Then we have
\begin{align*}
    e_{t+1}&=e_t-\eta_tev_{x_t}^*\left[h_t(x_t)-h^\dagger(x_t)+h^\dagger(x_t)-y_t\right] \\
    &=(I-\eta_tA_{x_t})e_t+\eta_t\xi_t,\\
    \W_t&=(C_K-A_{x_t})e_t+\xi_t,
\end{align*}
and
\[
0\preceq A_x\preceq\kappa^2I,
\qquad
\be_{x\sim\rho_{\X}}[A_x]=C_K.
\]
Since $\|I-\eta_tA_x\|\leq1$ and
$\be_{z_t\sim\rho}\|\xi_t\|_K^2\leq\kappa^2\sigma^2$,
the recursion gives
\[
\be_{z^t}\|e_{t+1}\|_K^2
\leq
2\be_{z^{t-1}}\|e_t\|_K^2
+2\eta_t^2\kappa^2\sigma^2.
\]
Starting from $e_1=-h^\dagger\in\H_K$, induction shows
that each $e_t$ has a finite second moment.
Since $\|C_K-A_x\|\leq2\kappa^2$, the identity for
$\W_t$ implies $\be_{z^t}\|\W_t\|_K^2<\infty$.
Hence $\Sigma_t$ and $\Gamma_t$ are trace-class operators.

Set $Q:=\be_{z_t\sim\rho}[\xi_t\otimes\xi_t].$ For every $v\in\H_K$, Assumption \ref{a1} gives
\[
\begin{aligned}
\langle Qv,v\rangle_K
&=\be_{(x,y)\sim\rho}
  \big|\langle y-h^\dagger(x),v(x)\rangle_{\Y}\big|^2\\
&\leq\sigma^2\|v\|_{\rho_{\X}}^2
=\sigma^2\langle C_Kv,v\rangle_K.
\end{aligned}
\]
Thus $Q\preceq\sigma^2C_K$.

For fixed $z^{t-1}$ and $x_t$, we have
\[
\be_{y_t\sim\rho(y_t|x_t)}[\xi_t]
=
ev_{x_t}^*
\be_{y_t\sim\rho(y_t|x_t)}
[y_t-h^\dagger(x_t)]
=0.
\]
Thus the two cross terms in the expansion of
$e_{t+1}\otimes e_{t+1}$ have zero expectation.
Since $e_t$ is independent of $z_t$, we obtain
\[
\begin{aligned}
\Sigma_{t+1}
&=\be_{z^t}\l[
(I-\eta_tA_{x_t})(e_t\otimes e_t)
(I-\eta_tA_{x_t})
\r]+\eta_t^2Q\\
&=\be_{x\sim\rho_{\X}}
\l[(I-\eta_tA_x)\Sigma_t(I-\eta_tA_x)\r]
+\eta_t^2Q.
\end{aligned}
\]

We prove \eqref{(11)} by induction. Since $h_1=0$,
\[
\Sigma_1
=h^\dagger\otimes h^\dagger
\preceq\|h^\dagger\|_K^2I
\preceq MI.
\]
Suppose that $\Sigma_t\preceq MI$.
Using $Q\preceq\sigma^2C_K$ and
$\be_{x\sim\rho_{\X}}[A_x]=C_K$, we obtain
\[
\begin{aligned}
\Sigma_{t+1}
&\preceq
M\be_{x\sim\rho_{\X}}[(I-\eta_tA_x)^2]
+\eta_t^2\sigma^2C_K\\
&=
MI-2M\eta_tC_K
+M\eta_t^2\be_{x\sim\rho_{\X}}[A_x^2]
+\eta_t^2\sigma^2C_K\\
&\preceq
MI-\eta_t\l[2M-\eta_t(\kappa^2M+\sigma^2)\r]C_K,
\end{aligned}
\]
where the last inequality follows from
$A_x^2\preceq\kappa^2A_x$.
Moreover, $\sigma^2\leq\kappa^2M$ and
$\eta_t\kappa^2\leq1$ imply
\[
\eta_t(\kappa^2M+\sigma^2)
\leq2\eta_t\kappa^2M
\leq2M.
\]
Since $C_K$ is nonnegative, it follows that
$\Sigma_{t+1}\preceq MI$.
This completes the induction.

Using \eqref{(11)} and the same cancellation of cross terms,
we obtain
\[
\begin{aligned}
\Gamma_t
&=\be_{x\sim\rho_{\X}}
  [(C_K-A_x)\Sigma_t(C_K-A_x)]+Q\\
&\preceq
M\be_{x\sim\rho_{\X}}[(C_K-A_x)^2]+\sigma^2C_K\\
&=
M\l(\be_{x\sim\rho_{\X}}[A_x^2]-C_K^2\r)
+\sigma^2C_K\\
&\preceq
(\kappa^2M+\sigma^2)C_K-MC_K^2\\
&\preceq B_2C_K.
\end{aligned}
\]
Here the second equality follows by expanding the square
and using $\be_{x\sim\rho_{\X}}[A_x]=C_K$.
The last two inequalities follow from
$A_x^2\preceq\kappa^2A_x$ and $C_K^2\succeq0$.
This proves \eqref{(12)}.
\end{proof}

\begin{proof}[Proof of Proposition \ref{prop:finite-prediction}]
Write $a:=\eta T^{-\theta'}$. Recall that $R_j:=\be_{z^{j-1}}
\l[\E(h_j)-\E(h^\dagger)\r].$ Since $\eta<\min\{1,\gamma_1'\}$, we have $a<1$ and
$a\kappa^2<1$. Proposition \ref{M_1} gives
\[
R_j\leq M_0,
\qquad 1\leq j\leq T+1,
\]
where
\[
M_0:=
2\|h^\dagger\|_{\rho_{\X}}^2
+4\kappa^2(1+2\kappa^2)\sigma^2
\left(2+\frac{1}{\theta'}\right)
\]
is independent of $T$ and $\eta$.

Let $Q$ be the noise covariance operator defined in
Proposition \ref{error decomp}.
For $1\leq t\leq T$, the geometric sum identity gives
\[
a^2\sum_{m=0}^{t-1}C_K(I-aC_K)^{2m}
=
a(2I-aC_K)^{-1}
\l[I-(I-aC_K)^{2t}\r]
\preceq aI.
\]
Consequently,
\begin{equation}\label{eq:finite-prediction-noise}
\operatorname{Tr}\left(
Q a^2\sum_{m=0}^{t-1}C_K(I-aC_K)^{2m}
\right)
\leq a\operatorname{Tr}(Q)
\leq a\kappa^2\sigma^2.
\end{equation}

Write $b_t:=\|(I-aL_K)^t h^\dagger\|_{\rho_{\X}}^2.$ Assumption \ref{a3} and Lemma \ref{lemma 1} imply that,
for $t\geq1$,
\begin{equation}\label{eq:finite-prediction-bias}
\begin{aligned}
b_t
&\leq
\|g^\dagger\|_{\rho_{\X}}^2
\min\left\{
\kappa^{4r},
\left(\frac{r}{eat}\right)^{2r}
\right\}\\
&\leq
\|g^\dagger\|_{\rho_{\X}}^2
\max\left\{\kappa^{4r},(r/e)^{2r}\right\}
\bigl(\max\{1,at\}\bigr)^{-2r}\\
&\leq B_r(1+at)^{-2r},
\end{aligned}
\end{equation}
where $B_r:=
2^{2r}
\max\left\{\kappa^{4r},(r/e)^{2r}\right\}
\|g^\dagger\|_{\rho_{\X}}^2.$ The last inequality follows from
$1+at\leq2\max\{1,at\}$. 
The bound $b_t\leq B_r(1+at)^{-2r}$ also holds for $t=0$.

Define
\[
k_m:=\|C_K(I-aC_K)^{2m}\|,
\qquad m\geq0.
\]
Since $h^\dagger\in\ker L_K^\perp$, the RKHS norm identity gives
\[
\|L_K^{1/2}(I-aL_K)^t h^\dagger\|_K^2
=
\|(I-aL_K)^t h^\dagger\|_{\rho_{\X}}^2
=b_t.
\]
Applying \eqref{eq:prediction-recursion} with $t$ in place
of $T$ and using \eqref{eq:finite-prediction-noise} and
\eqref{eq:finite-prediction-bias}, we obtain
\begin{equation}\label{eq:finite-prediction-recursion}
R_{t+1}
\leq
B_r(1+at)^{-2r}
+\kappa^2\sigma^2a
+\kappa^2a^2\sum_{j=1}^{t}k_{t-j}R_j,
\qquad 1\leq t\leq T.
\end{equation}
Moreover, $k_0\leq\kappa^2$, and Lemma \ref{lemma 1} gives
\[
k_m\leq\frac{1}{2eam},
\qquad m\geq1.
\]
Hence we have
\begin{equation}\label{eq:finite-prediction-kernel-sum}
a^2\sum_{m=0}^{t-1}k_m
\leq
\kappa^2a^2+\frac{a}{2e}(1+\log T),
\qquad 1\leq t\leq T.
\end{equation}

Set $w_t:=(1+at)^{-2r}+a$ and $F_T:=\max_{0\leq t\leq T}\frac{R_{t+1}}{w_t}$. The uniform bound on $R_j$ ensures that $F_T$ is finite.
For each $1\leq t\leq T$, put $\ell:=\lfloor t/2\rfloor$.
The contribution from $j\leq\ell$ satisfies
\begin{equation}\label{eq:finite-prediction-early-sum}
\begin{aligned}
\kappa^2a^2\sum_{j=1}^{\ell}k_{t-j}R_j
\leq\frac{\kappa^2M_0a}{2e}\sum_{j=1}^\ell\frac{1}{t-j}\leq
\frac{\kappa^2M_0a}{2e}
\frac{\ell}{t-\ell}\leq\frac{\kappa^2M_0a}{2e}.
\end{aligned}
\end{equation}
The sum is empty when $\ell=0$.

For $j>\ell$, the inequalities $a<1$ and $j-1\geq\ell$
give
\[
1+at\leq2(1+a\ell)\leq2(1+a(j-1)).
\]
Therefore, we have
\[
R_j\leq F_Tw_{j-1}\leq2^{2r}F_Tw_t,
\]
and by \eqref{eq:finite-prediction-kernel-sum},
\begin{equation}\label{eq:finite-prediction-late-sum}
\begin{aligned}
\kappa^2a^2\sum_{j=\ell+1}^{t}k_{t-j}R_j
&\leq
2^{2r}\kappa^2F_Tw_t\,
a^2\sum_{j=\ell+1}^{t}k_{t-j}\\
&=
2^{2r}\kappa^2F_Tw_t\,
a^2\sum_{m=0}^{t-\ell-1}k_m\\
&\leq
2^{2r}\kappa^2F_Tw_t\,
a^2\sum_{m=0}^{t-1}k_m\\
&\leq
2^{2r}\kappa^2F_Tw_t
\left(\kappa^2a^2+\frac{a}{2e}(1+\log T)\right)\\
&=q_TF_Tw_t,
\end{aligned}
\end{equation}
where 
\[
q_T:=
2^{2r}\kappa^2
\left(\kappa^2a^2+\frac{a}{2e}(1+\log T)\right).
\]
Set $H:=\max\left\{
B_r,\,
\kappa^2\sigma^2+\frac{\kappa^2M_0}{2e}
\right\}.$ Substituting \eqref{eq:finite-prediction-early-sum} and
\eqref{eq:finite-prediction-late-sum} into
\eqref{eq:finite-prediction-recursion}, we obtain
\[
\begin{aligned}
R_{t+1}
&\leq
B_r(1+at)^{-2r}
+\left(\kappa^2\sigma^2+\frac{\kappa^2M_0}{2e}\right)a
+q_TF_Tw_t\\
&\leq
H\l[(1+at)^{-2r}+a\r]+q_TF_Tw_t\\
&=(H+q_TF_T)w_t.
\end{aligned}
\]
Since $w_t>0$, this gives
\[
\frac{R_{t+1}}{w_t}\leq H+q_TF_T,
\qquad 1\leq t\leq T.
\]
Since $R_1/w_0\leq M_0$, it follows that
\[
F_T\leq\max\{M_0,H+q_TF_T\}.
\]

Using $a\leq T^{-\theta'}$, we have
\[
q_T\leq
2^{2r}\kappa^2
\left(
\kappa^2T^{-2\theta'}
+\frac{T^{-\theta'}}{2e}(1+\log T)
\right)
\longrightarrow0.
\]
Thus there exists an integer $T_0\geq2$, depending only on
$r$, $\theta'$, and $\kappa$, such that $q_T\leq1/2$
whenever $T\geq T_0$. For such $T$,
\[
F_T\leq\max\{M_0,2H\}.
\]
If $T<T_0$, then
\[
w_t\geq(1+T_0)^{-2r},
\qquad 0\leq t\leq T,
\]
and hence
\[
F_T\leq M_0(1+T_0)^{2r}.
\]
Therefore,
\[
R_{t+1}
\leq F_Tw_t
\leq
c\l[
(1+\eta tT^{-\theta'})^{-2r}
+\eta T^{-\theta'}
\r],
\qquad 0\leq t\leq T,
\]
with
\[
c:=\max\{M_0(1+T_0)^{2r},2H\},
\]
which is independent of $T$ and $\eta$.
This proves \eqref{eq:finite-prediction-bound}.
\end{proof}

\begin{proof}[Proof of Proposition \ref{prop:online-prediction}]
For $1/2\leq\theta<1$, the conclusion follows directly from Propositions \ref{error decomp}, \ref{term1'}, \ref{term2}, and \ref{M_1}. It remains to consider $0<\theta<1/2$.

Recall that $R_j:=\be_{z^{j-1}}\l[\E(h_j)-\E(h^\dagger)\r]$. The step size condition implies $\eta_1<1$ and $\eta_j\kappa^2<1$ for every $j\geq1$. By Proposition \ref{M_1},
\[
R_j\leq M_0,\qquad j\geq1,
\]
where
\[
M_0:=2\|h^\dagger\|_{\rho_{\X}}^2
+4\kappa^2(1+2\kappa^2)\sigma^2\frac{\delta+1}{\theta}
\]
is independent of $\eta_1$.

Propositions \ref{error decomp}, \ref{term1'}, \ref{term2}, and \ref{M_1} also give a constant $C_0>0$, independent of $j$ and $\eta_1$, such that
\begin{equation}\label{eq:online-prediction-coarse-decay}
R_j\leq C_0\l[
\eta_1^{-2r}j^{-2r(1-\theta)}
+\eta_1j^{-\theta}\log j
\r],\qquad j\geq1.
\end{equation}
These propositions give the bound for $j\geq2$. The case $j=1$ follows by increasing $C_0$, since $R_1=\|h^\dagger\|_{\rho_{\X}}^2$ and $\eta_1<1$.

For $1\leq j\leq t$, write
\[
P_{j:t}:=\prod_{i=j}^{t}(I-\eta_iC_K),\qquad P_{t+1:t}:=I,
\]
and define
\[
k_{j,t}:=\|C_KP_{j+1:t}^2\|,\qquad
S_t:=\sum_{j=1}^{t}\eta_j^2C_KP_{j+1:t}^2.
\]
Let $Q$ be as in Proposition \ref{error decomp}. Since $h^\dagger\in\ker L_K^\perp$, the RKHS norm identity gives
\[
\left\|L_K^{1/2}\prod_{j=1}^{t}(I-\eta_jL_K)h^\dagger\right\|_K^2
=
\left\|\prod_{j=1}^{t}(I-\eta_jL_K)h^\dagger\right\|_{\rho_{\X}}^2
=:b_t.
\]
Applying \eqref{eq:prediction-recursion} with $t$ in place of $T$, we obtain
\begin{equation}\label{eq:online-prediction-recursion}
R_{t+1}\leq b_t+\operatorname{Tr}(QS_t)
+\kappa^2\sum_{j=1}^{t}\eta_j^2k_{j,t}R_j,\qquad t\geq1.
\end{equation}
Proposition \ref{term1'} with $\alpha=1/2$ gives
\begin{equation}\label{eq:online-prediction-bias}
b_t\leq B_r\eta_1^{-2r}(t+1)^{-2r(1-\theta)},\qquad
B_r:=\left(\frac{2r}{e}\right)^{2r}\|g^\dagger\|_{\rho_{\X}}^2.
\end{equation}

Fix $t\geq1$ and put $\ell:=\lfloor t/2\rfloor$. For $j\leq\ell$, we have
\[
\sum_{i=j+1}^{t}\eta_i
\geq\eta_1(t-\ell)t^{-\theta}
\geq\frac{\eta_1}{2}t^{1-\theta}.
\]
Lemma \ref{lemma 1} therefore gives $k_{j,t}\leq\frac{1}{e\eta_1t^{1-\theta}}$ for any $j\leq\ell$. Since $0<\theta<1/2$, we have $\sum_{j=1}^{\ell}j^{-2\theta}\leq\frac{t^{1-2\theta}}{1-2\theta}$. Consequently,
\begin{equation}\label{eq:online-prediction-early-kernel}
\begin{aligned}
\sum_{j=1}^{\ell}\eta_j^2k_{j,t}
&\leq\frac{\eta_1t^{-\theta}}{e(1-2\theta)}\\
&\leq\frac{2^\theta}{e(1-2\theta)}\eta_1(t+1)^{-\theta}.
\end{aligned}
\end{equation}
The sum is empty when $\ell=0$.

For the remaining indices, the identity
\[
P_{j+1:t}^2-P_{j:t}^2
=\eta_jC_K(2I-\eta_jC_K)P_{j+1:t}^2
\]
and $\eta_j\kappa^2<1$ imply $\eta_jC_KP_{j+1:t}^2\preceq P_{j+1:t}^2-P_{j:t}^2$. Since $\eta_j\leq\eta_{\ell+1}$ for $j>\ell$,
\begin{equation}\label{eq:online-prediction-late-operator}
\begin{aligned}
\sum_{j=\ell+1}^{t}\eta_j^2C_KP_{j+1:t}^2
&\preceq\eta_{\ell+1}\sum_{j=\ell+1}^{t}(P_{j+1:t}^2-P_{j:t}^2)\\
&=\eta_{\ell+1}(I-P_{\ell+1:t}^2)\\
&\preceq2^\theta\eta_1t^{-\theta}I\\
&\preceq2^{2\theta}\eta_1(t+1)^{-\theta}I.
\end{aligned}
\end{equation}

Set $A_\theta:=\frac{2^\theta}{e(1-2\theta)}+2^{2\theta}$. By \eqref{eq:online-prediction-early-kernel} and \eqref{eq:online-prediction-late-operator}, $0\preceq S_t\preceq A_\theta\eta_1(t+1)^{-\theta}I$. It follows that
\begin{equation}\label{eq:online-prediction-noise}
\operatorname{Tr}(QS_t)
\leq A_\theta\eta_1(t+1)^{-\theta}\operatorname{Tr}(Q)
\leq\kappa^2\sigma^2A_\theta\eta_1(t+1)^{-\theta}.
\end{equation}
Also, using $R_j\leq M_0$ and \eqref{eq:online-prediction-early-kernel}, we obtain
\begin{equation}\label{eq:online-prediction-early-sum}
\kappa^2\sum_{j=1}^{\ell}\eta_j^2k_{j,t}R_j
\leq\kappa^2M_0A_\theta\eta_1(t+1)^{-\theta}.
\end{equation}

For $j>\ell$, we have $j\geq(t+1)/2$ and $\log j\leq\log(t+1)$. Thus \eqref{eq:online-prediction-coarse-decay} gives
\begin{equation}\label{bound of R_j}
\begin{aligned}
R_j
&\leq C_0\l[
2^{2r(1-\theta)}\eta_1^{-2r}(t+1)^{-2r(1-\theta)}
+2^\theta\eta_1(t+1)^{-\theta}\log(t+1)
\r]\\
&\leq C_1\l[
\eta_1^{-2r}(t+1)^{-2r(1-\theta)}
+\eta_1(t+1)^{-\theta}\log(t+1)
\r],
\end{aligned}
\end{equation}
where $C_1$ is independent of $t$ and $\eta_1$.

To bound $\sum_{j=\ell+1}^{t}\eta_j^2k_{j,t}$, observe that
\[
\eta_j\leq2^\theta\eta_1t^{-\theta},\qquad j>\ell.
\]
For $j<t$, Lemma \ref{lemma 1} and
\[
\sum_{i=j+1}^{t}\eta_i\geq\eta_1(t-j)t^{-\theta}
\]
give
\[
k_{j,t}\leq\frac{t^\theta}{2e\eta_1(t-j)}.
\]
Since $k_{t,t}\leq\kappa^2$, we obtain
\begin{equation}\label{eq:online-prediction-kernel-sum}
\begin{aligned}
\sum_{j=\ell+1}^{t}\eta_j^2k_{j,t}
&\leq2^{2\theta}\eta_1^2t^{-2\theta}
\left(\kappa^2+\frac{t^\theta}{2e\eta_1}
\sum_{m=1}^{t-\ell-1}\frac1m\right)\\
&\leq2^{2\theta}
\left(\kappa^2\eta_1^2t^{-2\theta}
+\frac{\eta_1}{2e}t^{-\theta}(1+\log t)\right)\\
&\leq C_2\eta_1(t+1)^{-\theta}\log(t+1).
\end{aligned}
\end{equation}
Here $C_2$ is independent of $t$ and $\eta_1$.

Moreover,
\[
\eta_1(t+1)^{-\theta}\log(t+1)\leq\frac{1}{e\theta},
\qquad
\eta_1(t+1)^{-\theta}\bigl(\log(t+1)\bigr)^2
\leq\frac{4}{e^2\theta^2}.
\]
Using \eqref{bound of R_j} and \eqref{eq:online-prediction-kernel-sum}, we obtain
\begin{equation}\label{eq:online-prediction-late-sum}
\begin{aligned}
\kappa^2\sum_{j=\ell+1}^{t}\eta_j^2k_{j,t}R_j
&\leq\kappa^2C_1
\l[\eta_1^{-2r}(t+1)^{-2r(1-\theta)}
+\eta_1(t+1)^{-\theta}\log(t+1)\r]
\sum_{j=\ell+1}^{t}\eta_j^2k_{j,t}\\
&\leq\kappa^2C_1C_2
\l[\eta_1^{1-2r}(t+1)^{-2r(1-\theta)-\theta}\log(t+1)
+\eta_1^2(t+1)^{-2\theta}\bigl(\log(t+1)\bigr)^2\r]\\
&\leq C_3\l[
\eta_1^{-2r}(t+1)^{-2r(1-\theta)}
+\eta_1(t+1)^{-\theta}\r],
\end{aligned}
\end{equation}
where $C_3$ is independent of $t$ and $\eta_1$.

Substituting \eqref{eq:online-prediction-bias}, \eqref{eq:online-prediction-noise}, \eqref{eq:online-prediction-early-sum}, and \eqref{eq:online-prediction-late-sum} into \eqref{eq:online-prediction-recursion}, we obtain
\[
\begin{aligned}
R_{t+1}
&\leq B_r\eta_1^{-2r}(t+1)^{-2r(1-\theta)}
+\kappa^2(\sigma^2+M_0)A_\theta\eta_1(t+1)^{-\theta}\\
&\quad+C_3\l[
\eta_1^{-2r}(t+1)^{-2r(1-\theta)}
+\eta_1(t+1)^{-\theta}\r]\\
&\leq c\l[
\eta_1^{-2r}(t+1)^{-2r(1-\theta)}
+\eta_1(t+1)^{-\theta}\r],
\end{aligned}
\]
where
\[
c:=\max\left\{
B_r+C_3,\,
\kappa^2(\sigma^2+M_0)A_\theta+C_3
\right\}
\]
is independent of $t$ and $\eta_1$. Since $V_\theta(t)=(t+1)^{-\theta}$ for $0<\theta<1/2$, taking $t=T$ proves \eqref{eq:online-prediction-bound} and completes the proof.
\end{proof}

\section*{Declarations}

{\bf Conflict of Interest} The authors declared that they have no conflict of interest.


\bibliographystyle{plain}
\bibliography{reference}

\end{document}